\pdfoutput=1
\documentclass{article}
\PassOptionsToPackage{table}{xcolor}
\usepackage[main,final]{neurips_2026}
\usepackage[utf8]{inputenc}
\usepackage[T1]{fontenc}
\usepackage{hyperref}
\usepackage{url}
\usepackage{booktabs}
\usepackage{amsfonts}
\usepackage{nicefrac}
\usepackage{microtype}
\usepackage{xcolor}

\usepackage{graphicx}
\usepackage{subcaption}
\usepackage{listings}
\usepackage{tcolorbox}
\usepackage{xurl}
\usepackage{amsmath}
\usepackage{amssymb}
\usepackage{mathtools}
\usepackage{amsthm}
\usepackage[capitalize,noabbrev]{cleveref}
\usepackage[textsize=tiny]{todonotes}
\usepackage{algorithm}
\usepackage{algorithmic}
\usepackage{colortbl}

\theoremstyle{plain}
\newtheorem{theorem}{Theorem}[section]

\newtheorem{lemma}[theorem]{Lemma}

\theoremstyle{definition}
\newtheorem{definition}[theorem]{Definition}

\theoremstyle{remark}

\newcommand{\E}{\mathbb{E}}
\newcommand{\KL}{\mathrm{KL}}

\newcommand{\1}{\mathbf{1}}

\DeclareMathOperator{\Var}{Var}
\DeclareMathOperator{\conf}{conf}

\newcommand{\gbar}{\ensuremath{\bar{g}(x;\theta)}}
\newcommand{\ghat}{\ensuremath{\hat{g}_t(x,y;\theta)}}

\title{VI-CuRL: Stabilizing Verifier-Independent Reinforcement Learning Reasoning via Confidence-Guided Variance Reduction}

\author{%
  Xin-Qiang Cai$^{1}$ \quad Masashi Sugiyama$^{1,2}$ \\
  $^{1}$RIKEN AIP, Tokyo, Japan \\
  $^{2}$The University of Tokyo, Tokyo, Japan
}

\date{}

\newcommand{\rev}[1]{{\color{black}#1}}
\newenvironment{revision}{\begingroup\color{black}}{\endgroup}
\hypersetup{colorlinks=false}

\begin{document}

\maketitle

\begin{abstract}
Reinforcement Learning with Verifiable Rewards (RLVR) has emerged as a dominant paradigm for enhancing Large Language Models (LLMs) reasoning, yet its reliance on external verifiers limits its scalability. Recent findings suggest that RLVR primarily functions by eliciting latent capabilities, motivating the development of verifier-free algorithms. However, in such settings, standard methods like Group Relative Policy Optimization face a critical challenge: destructive gradient variance that often leads to training collapse. To address this issue, we introduce \textbf{Verifier-Independent Curriculum Reinforcement Learning (VI-CuRL)}, a framework that leverages the model's intrinsic confidence to construct a curriculum independent from external verifiers. By prioritizing high-confidence samples, VI-CuRL effectively manages the bias-variance trade-off, specifically targeting the reduction of \emph{action} and \emph{problem variance}. We provide a rigorous theoretical analysis, proving that our estimator guarantees asymptotic unbiasedness. Empirically, VI-CuRL promotes stability and consistently outperforms verifier-dependent/independent baselines across math and general reasoning benchmarks with/without verifiers.
\end{abstract}

\section{Introduction}
Reinforcement Learning with Verifiable Rewards (RLVR) has recently emerged as a dominant paradigm for enhancing the reasoning capabilities of large language models (LLMs) in post-training \citep{deepseek2025r1}. Typically, RLVR relies on an external verifier to provide reward signals that guide the model toward correct reasoning paths \citep{lightman2024letsverify}. However, obtaining reliable verifiers is often prohibitively expensive or infeasible for many open-ended tasks \citep{kou2026positiveunlabeled}. Meanwhile, recent studies suggest that the success of RLVR may be largely attributed to its ability to elicit the model's latent potential rather than injecting new knowledge \citep{chen2024hidden, deepseek2025r1}. Consequently, there is a growing interest in verifier-free algorithms that can unlock these capabilities without reliance on ground-truth supervision \citep{liu2025nover, zhou2025verifree}.

Despite its promise, RLVR, like many RL approaches, faces a significant challenge: training instability due to high gradient variance \citep{yao2025optimizing, ye2025robust}. Traditional RL methods employ various techniques to reduce variance, often introducing a controlled amount of bias to stabilize training, known as the bias-variance trade-off \citep{geman1992neural, schulman2016high}. In the context of RLVR, particularly in verifier-free settings, this issue is exacerbated. Without effective external supervision to prune incorrect trajectories, the variance in LLM training becomes destructively large, often leading to performance collapse \citep{zhang2025nofree}.

\begin{figure*}[ht]
    \centering
    \includegraphics[width=0.8\linewidth]{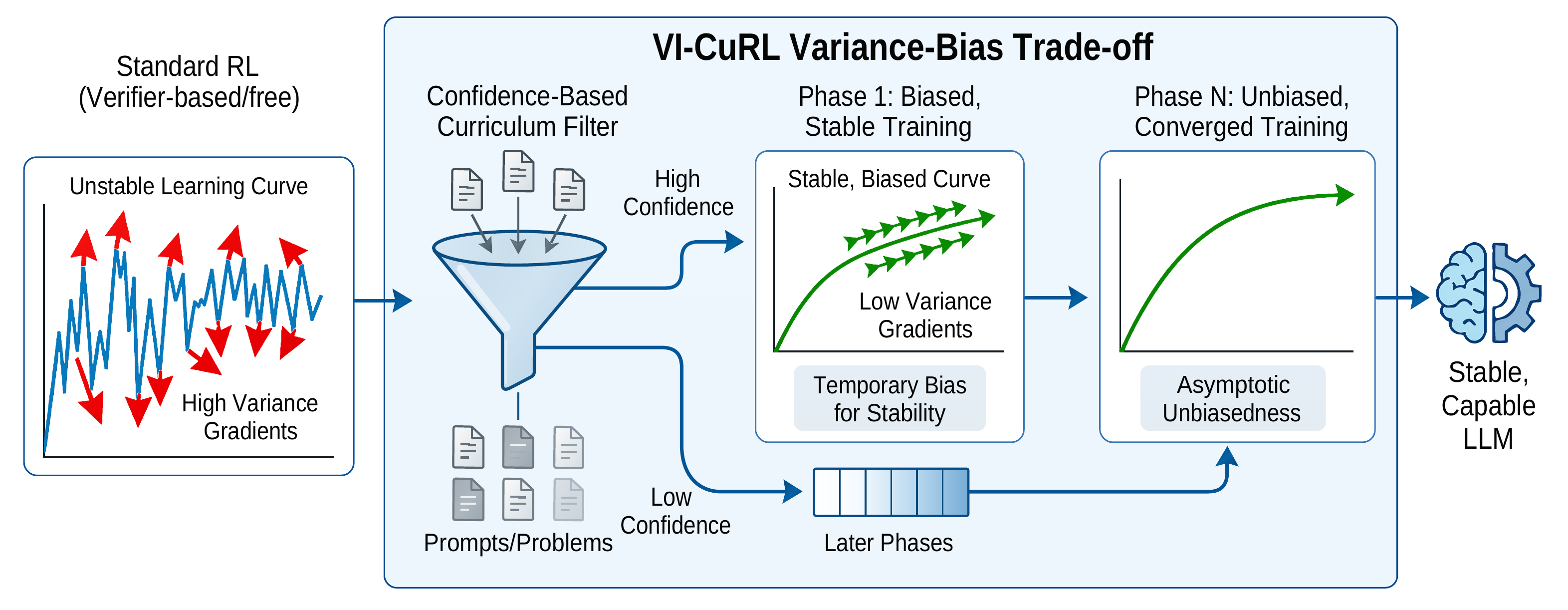}
    \caption{Conceptual overview of VI-CuRL. Unlike standard RL that treats all samples equally, VI-CuRL dynamically selects high-confidence samples to stabilize training via a principled bias-variance trade-off, without accessing external verifiers.}
    \label{fig:intro_concept}
\end{figure*}

To address this instability, curriculum learning presents a promising direction, proposing to train models on a sequence of progressively harder examples to improve convergence \citep{bengio2009curriculum, shi2025continual}. While curriculum-based approaches for RLVR, such as Adaptive Curriculum Reinforcement Finetuning (AdaRFT) \citep{chang2024adarft} and Variance-based Curriculum Reinforcement Learning (VCRL) \citep{jiang2025vcrl}, have been proposed, they are verifier-dependent and rely on verifier-provided information to construct the curriculum. While in verifier-free scenarios, however, the destructive impact of variance is far more severe, and such external guidance is absent \citep{zhang2025nofree}. Therefore, we propose \textbf{Verifier-Independent Curriculum Reinforcement Learning (VI-CuRL)}, as illustrated in Figure~\ref{fig:intro_concept}. VI-CuRL uses an annealed curriculum to stabilize RL process with confidence as the curriculum metric. Instead of relying on the variance of external training data \citep{zhang2025nofree} or signals from verifiers \citep{lightman2024letsverify, zhou2025reinforcing}, VI-CuRL leverages an intrinsic prompt-level curriculum score, which instantiated as entropy-derived confidence, as a verifier-free curriculum statistic for data selection. By prioritizing high-confidence samples that form a lower-uncertainty subset in the early stages, our method achieves a favorable bias-variance balance. We guarantee that, as the curriculum progresses, the surrogate objective asymptotically converges to the true objective.

\rev{We formalize the curriculum bias--noise trade-off and evaluate VI-CuRL on mathematical reasoning and general-knowledge benchmarks. Oracle rewards test composability; intrinsic rewards test stability when reliable external reward supervision is absent. Our contribution concerns prompt selection, independently of the chosen reward source.}

\section{Preliminaries and Problem Formulation}
In this section, we introduce the preliminaries and problem formulation.

\paragraph{Notation.}
We consider a prompt $x$ drawn from a data distribution $p(x)$. At each iteration, we maintain a trainable policy $\pi_\theta$ parameterized by $\theta \in \mathbb{R}^d$, and a behavior policy $\pi_{\theta_\mathrm{old}}$ (the policy from the current outer iteration). We also assume a fixed reference policy $\pi_{\text{ref}}$ (typically the initial Supervised Fine-Tuning (SFT) policy) for Kullback-Leibler (KL) regularization.
Given a prompt $x$, we sample $G$ trajectories $\{y_1, \dots, y_G\}$ from $\pi_{\theta_\mathrm{old}}(\cdot|x)$, where each $y_i = a_{i, 1:T_i}$ consists of $T_i$ tokens. We denote the state at step $t$ as $s_{i,t} = (x, a_{i, <t})$. The group-relative advantage $A_i$ for trajectory $i$ is:
\begin{equation}
    A_i = \frac{R(x, y_i) - \bar{R}}{\sigma_R + \epsilon},
\end{equation}
where $\bar{R}$ and $\sigma_R$ are the mean and standard deviation of rewards within the group, and $\epsilon$ is a small constant to prevent division by zero.

\paragraph{Base Surrogate Objective.}
We adopt the Group Relative Policy Optimization (GRPO) framework \citep{shao2024deepseekmath}, which we view as optimizing a Proximal Policy Optimization (PPO) style surrogate objective with group-relative advantages and explicit KL regularization.
Specifically, we define the log-ratio $r_{i,t}(\theta)$ and probability ratio $\rho_{i,t}(\theta)$ as:
\begin{align*}
    r_{i,t}(\theta) &:= \log \pi_\theta(a_{i,t}|s_{i,t}) - \log \pi_{\theta_\mathrm{old}}(a_{i,t}|s_{i,t}), \\
    \rho_{i,t}(\theta) &:= \exp(r_{i,t}(\theta)).
\end{align*}
The per-prompt, per-group surrogate loss averages over rollouts and tokens to ensure boundedness:
\begin{equation}
\ell(\theta; x, y_{1:G}) := 
\frac{1}{G} \sum_{i=1}^G \frac{1}{T_i} \sum_{t=1}^{T_i} \Big[ L_{\text{clip}}(i,t;\theta) - \beta_{\mathrm{KL}} \KL(\pi_\theta(\cdot|s_{i,t}) || \pi_{\text{ref}}(\cdot|s_{i,t})) \Big],
\end{equation}
where the standard PPO-Clip term is:
\[ L_{\text{clip}}(i,t;\theta) := \min\Big( \rho_{i,t}(\theta)A_i, \;\text{clip}(\rho_{i,t}(\theta), 1-\epsilon, 1+\epsilon)A_i \Big). \]
This normalization makes the loss scale approximately length-invariant (independent of $T_i$) in practice.
The global training objective is to maximize the expected value of this surrogate:
\begin{equation}\label{eq:L_surrogate}
\mathcal{L}(\theta) := \mathbb{E}_{x\sim p(x)}\;\mathbb{E}_{y_{1:G}\sim \pi_{\theta_\mathrm{old}}(\cdot|x)}\big[\,\ell(\theta; x, y_{1:G})\,\big].
\end{equation}
In standard GRPO, we estimate the gradient of $\mathcal{L}(\theta)$ by sampling prompts and groups from $\pi_{\theta_\mathrm{old}}$ and performing multiple epochs of stochastic gradient ascent on $\ell$.

\paragraph{Curriculum RL.}
Standard RL treats all samples from $p(x)$ equally. Curriculum methods effectively optimize a sequence of objectives $\mathcal{L}_t(\theta)$, where the data distribution is re-weighted. The goal is to start with a low-uncertainty distribution where gradient variance is low, and gradually anneal towards the target distribution $p(x)$. To this end, a prompt-level curriculum score is introduced and denoted as $c(x) \in [0, 1]$. Then the curriculum mask is defined as $w_t(x) = \1[c(x) \ge \tau_t]$ \citep{kadavath2022language, huang2023large, ren2023selfevaluation}. We formalize this signal as follows:

\begin{definition}[Curriculum Score]
Quantifying confidence in LLMs is achieved by aggregating token-level probabilities. To obtain a stable signal $c(x) \in [0, 1]$ that depends only on the prompt, we first define the \emph{normalized uncertainty} $u(x)$ as the expectation over trajectories sampled from the behavior policy $\pi_{\theta_\mathrm{old}}$:
\begin{equation}
    u(x) := \mathbb{E}_{y \sim \pi_{\theta_\mathrm{old}}(\cdot|x)} \left[ \frac{1}{T} \sum_{t=1}^T \frac{H(\pi_{\theta_\mathrm{old}}(\cdot | s_t))}{\log |\mathcal{V}|} \right],
\end{equation}
where $H$ is the Shannon entropy and $|\mathcal{V}|$ is the vocabulary size. The curriculum score is then instantiated as:
\begin{equation}
\label{eq:confidence}
    c(x) = 1 - u(x) \in [0, 1].
\end{equation}
\end{definition}
The confidence of a LLM on a prompt is a verifier-free curriculum statistic for data selection; it does not verify trajectories but gates prompt-level exposure to control gradient variance. Therefore, we leverage it here as an effective proxy for \emph{variance control} to stabilize training, rather than as a ground-truth verification signal.

\section{Verifier-Independent Curriculum Reinforcement Learning}
\label{sec:method}
We propose \textbf{Verifier-Independent Curriculum Reinforcement Learning (VI-CuRL)}, a novel framework for solving the high-variance challenge independent from verifiers. 

\subsection{The VI-CuRL Objective}
To formally operationalize our curriculum, we define a sequence of curriculum phases indexed by $t=1, \dots, T$. Each phase is associated with a confidence threshold $\tau_t$ that determines which samples are included in the training process.

\paragraph{Curriculum Mask and Retention Rate.}

\begin{revision}
For a fixed behavior policy and threshold within curriculum phase $t$, define the population retention probability of the curriculum mask as
\begin{equation}\label{eq:population_retention}
\beta_t=\E_{x\sim p(x)}[w_t(x)]\in(0,1].
\end{equation}
This probability is the mass of the retained region under $p(x)$ and normalizes the population objective and gradient estimator in Eqs.~\eqref{eq:Jt} and~\eqref{eq:masked}.

Let $B\ge1$ be the number of prompts sampled at one outer step. The schedule specifies $\beta_t^{\mathrm{target}}\in(0,1]$, the desired fraction of these prompts to retain. The integer retention count is $k_t=\lceil\beta_t^{\mathrm{target}}B\rceil$, so $1\le k_t\le B$. Exact top-$k_t$ selection produces batch masks $w_j\in\{0,1\}$, with $j\in\{1,\ldots,B\}$ indexing prompt occurrences in the sampled batch. The realized batch retention rate is
\begin{equation}\label{eq:batch_retention}
\hat{\beta}_t=\frac1B\sum_{j=1}^{B}w_j=\frac{k_t}{B}.
\end{equation}
We normalize the empirical batch objective by $\hat{\beta}_t$. The batch masks are constructed jointly by ranking the confidence scores, as detailed below.
\end{revision}
\paragraph{The Weighted Surrogate Objective.}
For each phase $t$, we optimize a re-weighted version of the standard surrogate $\mathcal{L}(\theta)$ that restricts learning to the selected subspace. To maintain the scale of the original objective, we introduce the importance sampling weight $1/\beta_t$:
\begin{equation}\label{eq:Jt}
    \mathcal{L}_t(\theta) := \mathbb{E}_{x\sim p(x)}\mathbb{E}_{y_{1:G}\sim\pi_{\theta_\mathrm{old}}(\cdot|x)}
\Big[\frac{w_t(x)}{\beta_t}\;\ell(\theta; x, y_{1:G})\Big].
\end{equation}
This weighted objective trades bias for variance. By focusing on the high-confidence region, it reduces the effective variance of the gradient estimator, albeit by optimizing a biased data distribution.

\paragraph{The Curriculum Gradient Estimator.}
To optimize $\mathcal{L}_t(\theta)$, we use the standard auto-differentiation subgradient of the clipped PPO surrogate. Critically, the advantages $A_i$, masks $w_t(x)$, and weights $\beta_t$ are treated as constants (``stop-gradients'') during the inner loop updates:
\begin{equation}\label{eq:masked}
    \hat{g}_t := \frac{w_t(x)}{\beta_t}\;\nabla_\theta \ell(\theta; x, y_{1:G}).
\end{equation}

\subsection{Curriculum Schedule and Algorithm}
We now detail the practical implementation of the VI-CuRL framework. The optimization process is divided into $T$ steps. At each step $t$, we seek to effectively filter the training batch to retain only the samples where the model exhibits high confidence.

\paragraph{Dynamic Quantile Thresholding.}

\begin{revision}
We implement the target retention schedule through exact top-$k_t$ selection, adapting the confidence threshold to the current batch. For $\mathcal B=\{x_j\}_{j=1}^{B}$, write $c_j=c(x_j)$ using Eq.~\eqref{eq:confidence}, and let $c_{(1)}\le\cdots\le c_{(B)}$ be the ordered scores, with $c_{(j)}$ denoting the $j$-th smallest score. The retention count and threshold are
\begin{equation}\label{eq:topk_threshold}
k_t=\left\lceil\beta_t^{\mathrm{target}}B\right\rceil,
\qquad
\tau_t=c_{(B-k_t+1)}
      =c_{(\lfloor(1-\beta_t^{\mathrm{target}})B\rfloor+1)}.
\end{equation}
The index $B-k_t+1$ lies in $\{1,\ldots,B\}$ and identifies the smallest score among the $k_t$ highest-scoring prompts. The two threshold expressions agree because $B-\lceil\beta_t^{\mathrm{target}}B\rceil=\lfloor(1-\beta_t^{\mathrm{target}})B\rfloor$.

To obtain exactly $k_t$ retained prompts even when scores coincide, let $n_t^{>}=|\{j:c_j>\tau_t\}|$ be the number of prompts strictly above the threshold, where $|\cdot|$ denotes set cardinality. We retain all $n_t^{>}$ such prompts and select the remaining $k_t-n_t^{>}$ from those with $c_j=\tau_t$, taking smaller batch indices first. Denote the resulting set of retained indices by $\mathcal I_t\subseteq\{1,\ldots,B\}$ and define $w_j=\1[j\in\mathcal I_t]$. With distinct scores, this rule reduces to $w_j=\1[c_j\ge\tau_t]$. In all cases, it guarantees
\begin{equation}\label{eq:retention_rounding}
\sum_{j=1}^{B}w_j=k_t,
\qquad
\hat{\beta}_t=\frac{k_t}{B},
\qquad
0\le\hat{\beta}_t-\beta_t^{\mathrm{target}}<\frac1B.
\end{equation}
\end{revision}

\paragraph{Algorithm.}

\begin{revision}
VI-CuRL follows the GRPO outer/inner-loop structure. At outer step $t$, we set $\pi_{\theta_\mathrm{old}}\leftarrow\pi_\theta$, sample $B$ prompts, and collect $G$ rollouts per prompt to compute rewards, advantages, and confidence scores. The linear schedule determines $\beta_t^{\mathrm{target}}$, after which we form $\mathcal I_t$ using the ranking and tie-breaking rule above and set $\hat{\beta}_t=k_t/B$. Using the base surrogate $\ell$ and the collected rollout groups $y_{j,1:G}$, define the empirical curriculum objective as
\begin{equation}\label{eq:batch_curriculum_objective}
\begin{aligned}
\widehat{\mathcal L}_t(\theta)
=\frac1B\sum_{j=1}^{B}\frac{w_j}{\hat{\beta}_t}\ell(\theta;x_j,y_{j,1:G})=\frac1{k_t}\sum_{j\in\mathcal I_t}\ell(\theta;x_j,y_{j,1:G}).
\end{aligned}
\end{equation}

During the inner loop, we fix the outer-batch data and apply the stop-gradient convention to the advantages, masks, and normalization. At each update, let $m\in\{1,\ldots,B\}$ denote the minibatch size and draw a fresh set of indices $\mathcal M\subseteq\{1,\ldots,B\}$ uniformly from all subsets of size $m$ of the full outer batch. Thus $|\mathcal M|=m$, and every prompt index has inclusion probability $m/B$. The minibatch gradient estimator is
\begin{equation}\label{eq:batch_gradient}
\hat g_t^{\mathcal M}
=\frac1{|\mathcal M|}\sum_{j\in\mathcal M}
\frac{w_j}{\hat{\beta}_t}\nabla_\theta\ell(\theta;x_j,y_{j,1:G}).
\end{equation}
Conditional on the outer-batch data, its fixed curriculum mask, and the current parameters, the expectation over the fresh minibatch draw satisfies
\begin{equation}\label{eq:batch_gradient_unbiasedness}
\begin{aligned}
\E_{\mathcal M}[\hat g_t^{\mathcal M}]
=\frac1m\sum_{j=1}^{B}\frac{m}{B}
  \frac{w_j}{\hat{\beta}_t}\nabla_\theta\ell(\theta;x_j,y_{j,1:G})=\nabla_\theta\widehat{\mathcal L}_t(\theta).
\end{aligned}
\end{equation}
Early in training, the schedule concentrates the empirical objective on high-confidence prompts, allowing the policy to start from a low-uncertainty training subset; At full retention, it recovers $B^{-1}\sum_{j=1}^{B}\ell(\theta;x_j,y_{j,1:G})$ exactly, ensuring unbiased optimization of the full distribution. Algorithm~\ref{alg:vicurl} summarizes the pseudocode for VI-CuRL.
\end{revision}

\begin{algorithm}[tb]
   \caption{Verifier-Independent Curriculum RL}
   \label{alg:vicurl}
\begin{algorithmic}
   \STATE {\bfseries Input:} Dataset $\mathcal{D}$, Initial Policy $\pi_{\text{ref}}$, Total steps $T$, Curriculum steps $T_{\text{curr}}$, Batch size $B$, Number of epochs $K$, Learning rate $\eta$.
   \STATE Initialize $\pi_\theta \leftarrow \pi_{\text{ref}}$.
   \FOR{$t = 1$ {\bfseries to} $T$}
       \STATE $\pi_{\theta_\mathrm{old}} \leftarrow \pi_\theta$.
       \STATE Sample batch $\mathcal{B} = \{x_j\}_{j=1}^B \sim \mathcal{D}$.
       \FOR{each prompt $x_j \in \mathcal{B}$}
           \STATE Sample group $\{y_{j,k}\}_{k=1}^G \sim \pi_{\theta_\mathrm{old}}(\cdot|x_j)$.
            \STATE Compute rewards $R_{j,k}$ and advantages $A_{j,k}$ using the chosen reward source.
            \STATE Compute curriculum score $c_j$ from $\pi_{\theta_{\rm old}}$. \hfill $\triangleright$ \textit{Used only to form $w_j$; receives no gradient.}
       \ENDFOR
       \STATE \textcolor{black}{Determine $\beta_t^{\mathrm{target}}$ via the schedule; set $k_t=\lceil\beta_t^{\mathrm{target}}B\rceil$.}
       \STATE \textcolor{black}{Sort the scores in ascending order; set $\tau_t=c_{(B-k_t+1)}$ as in Eq.~\eqref{eq:topk_threshold}.}
       \STATE \textcolor{black}{Form $\mathcal I_t$ from the $k_t$ highest-scoring prompt indices, breaking ties by ascending batch index.}
       \STATE \textcolor{black}{Set $w_j=\1[j\in\mathcal I_t]$ and $\hat{\beta}_t=k_t/B$ as in Eq.~\eqref{eq:batch_retention}.}
       \FOR{$k = 1$ {\bfseries to} $K$}
           \STATE \textcolor{black}{Sample $\mathcal M\subseteq\{1,\ldots,B\}$ uniformly with fixed size $m\in\{1,\ldots,B\}$.}
           \STATE \textcolor{black}{Compute $\hat g=\hat g_t^{\mathcal M}$ using Eq.~\eqref{eq:batch_gradient}, with normalization $w_j/\hat{\beta}_t$.}
           \STATE Update $\theta \leftarrow \theta + \eta \hat{g}$.
       \ENDFOR
   \ENDFOR
\end{algorithmic}
\end{algorithm}

\section{Theoretical Analysis of the Bias-Variance Trade-off}\label{sec:theory}
In this section, we formalize the properties of VI-CuRL. Our analysis first addresses the ``bias'' component of the trade-off. We then dissect the ``variance'' component to reveal the source of the training stability.

\subsection{Consistency of the Objective}
A foundational property of our framework is that the normalized gradient estimator $\ghat$ is, by construction, an unbiased estimator of the gradient of the surrogate objective $\mathcal{L}_t(\theta)$. That is, $\E[\ghat] = \nabla_\theta \mathcal{L}_t(\theta)$. This ensures our algorithm correctly optimizes the objective for the current curriculum phase. The critical theoretical question, however, is whether maximizing these time-dependent surrogate objectives consistently leads to maximizing the base surrogate. Our first theorem confirms that it does:

\begin{theorem}[Consistency of the Objective]\label{thm:asymptotic_unbiasedness}
Assuming the per-prompt surrogate loss is bounded, $|\ell(\theta; x, y_{1:G})| \le L_{\max}$, the absolute difference between the true surrogate objective $\mathcal{L}(\theta)$ and the weighted objective $\mathcal{L}_t(\theta)$ is bounded by $(1 - \beta_t)$:
\begin{equation}\label{eq:gap}
|\mathcal{L}(\theta) - \mathcal{L}_t(\theta)| \le 2L_{\max}(1 - \beta_t).
\end{equation}
\end{theorem}
The proof can be found in Appendix~\ref{app:proof_asymptotic_unbiasedness}.
The bounded-surrogate assumption $|\ell(\theta; x, y_{1:G})| \le L_{\max}$ is mild in practice with bounded rewards, group-normalized advantage, ratio clipping, and per-token KL penalty~\citep{ouyang2022training,shao2024deepseekmath}. Meanwhile, as the curriculum is annealed ($\tau_t \to 0$), the retention rate $\beta_t \to 1$, and the bound converges to zero, ensuring that $\mathcal{L}_t(\theta) \to \mathcal{L}(\theta)$. This theorem guarantees that the approximation error diminishes over time. Meanwhile, this result does not depend on the correctness or calibration of the confidence score. It only requires the retention rate to anneal to one.

\subsection{A Deeper Investigation of Variance}
Having established that our objective is asymptotically correct, we now analyze how the curriculum tames the variance that destabilizes standard RL. To this end, here we decompose the variance into its constituent parts.

\begin{definition}[Vector Variance]
For a random vector $\mathbf{Z} \in \mathbb{R}^d$, we define its variance as the trace of its covariance matrix:
\[ \Var(\mathbf{Z}) := \E[\|\mathbf{Z} - \E[\mathbf{Z}]\|^2]. \]
\end{definition}

\begin{theorem}[Variance Decomposition]\label{thm:variance_decomposition}
Let $g(x, y_{1:G}; \theta) := \nabla_\theta \ell(\theta; x, y_{1:G})$ be the gradient of the surrogate loss, and let $\gbar(x;\theta) = \E_{y_{1:G}|\pi_{\theta_\mathrm{old}}}[g]$ be the expected gradient for a fixed prompt $x$. The variance of the VI-CuRL estimator $\ghat$ can be decomposed into three sources:
\begin{equation}\label{eq:full_variance_explained}
\Var(\ghat) = \underbrace{\frac{\sigma_{g,t}^2}{\beta_t}}_{\substack{\text{Action Variance}}} + \underbrace{\frac{V_{\mathrm{prob}, t}}{\beta_t}}_{\substack{\text{Problem Variance}}} + \underbrace{\frac{1-\beta_t}{\beta_t} \|\nabla_\theta \mathcal{L}_t(\theta)\|^2}_{\substack{\text{Masking Variance}}},
\end{equation}
where $\sigma_{g,t}^2 := \E_{x|w_t=1}[\Var_{y_{1:G}|\pi_{\theta_\mathrm{old}}}(g)]$ is the average action variance within the selected set of problems, and $V_{\mathrm{prob}, t} := \Var_{x|w_t=1}(\gbar)$ is the variance of mean gradients within that same set.
\end{theorem}

The proof can be found in Appendix~\ref{app:proof_variance_decomposition}. 

This decomposition motivates our focus on action and problem
variance, which capture within-prompt rollout variability and
across-prompt gradient heterogeneity, respectively
\citep{tucker2018mirage,yao2026optimizing}.
These intrinsic sources of gradient noise can persist even at
full retention and can perturb the direction of policy updates.
For the population estimator analyzed above, the masking
covariance $\frac{1-\beta_t}{\beta_t}\mu_t\mu_t^\top$,
with $\mu_t=\nabla_\theta\mathcal L_t(\theta)$,
contributes only along the selected mean-gradient direction.
We thus prioritize reducing action and problem variability
within the selected subset, while treating masking as an
additional cost of partial retention which becomes negligible as $\beta_t \to 1$.

Beyond this exact identity, Appendix~\ref{app:proof_conf_aware_bound} and Appendix~\ref{app:proof_curriculum_sensitive} contain two conditional bounds. They formalize regimes in which gradient norms or selected-subset variances shrink sufficiently fast as confidence decreases. 

\section{Experiments}
\label{sec:experiments}
To validate the effectiveness of VI-CuRL, we conducted experiments on several reasoning tasks. In addition to the main results, we provide additional experiments in Appendix~\ref{app:additional_results}, including learning curve analysis across different reward settings, curriculum control experiments, confidence signal analysis, long-output robustness, and mechanism analysis. Code is provided in~\url{https://github.com/caixq1996/VI-CURL}.
\subsection{Experimental Setup}
We trained on four backbones, \textit{Qwen2.5-Math-1.5B}, \textit{DeepSeek-R1-Distill-Qwen-1.5B}, \textit{Llama-3.2-3B-Instruct}, and \textit{Qwen2.5-Math-7B}. Rewards for the oracle baselines came from a rule-based checker that extracted the final \verb|\boxed{\cdot}| answer and tested numeric/rational equivalence. For the verifier-free reward settings, we relied solely on intrinsic reward signals (e.g., majority vote or entropy); the VI-CuRL curriculum remains verifier-independent in all settings. We implemented VI-CuRL as a curriculum hook within \texttt{VERL}~\citep{sheng2024hybridflow}, following a PPO-style GRPO recipe. In the math domain, we used DeepScaleR~\citep{DeepScaleR} as the training set, and the evaluation used six verifiable math suites \textsc{AIME-2024}, \textsc{AIME-2025}, \textsc{AMC-2023}, \textsc{MATH500}, \textsc{Minerva MATH}, and \textsc{OlympiadBench}. To evaluate the non-mathematical general knowledge capabilities, we selected the Llama-3.2-3B-Instruct model for these experiments. In the general knowledge domain, we used 
\textsc{WebInstruct-verified}~\citep{ma2025generalreasoner}, and filtered out mathematics prompts from the data, and the evaluation used \textsc{MMLU-Pro}~\citep{wang2024mmlu}, \textsc{GPQA-Diamond}~\citep{rein2024gpqa}, \textsc{TheoremQA}~\citep{chen-etal-2023-theoremqa}, and \textsc{WebInstruct-Validation}~\citep{ma2025generalreasoner}. We reported \emph{Pass@1} and \emph{Pass@8} with 16 samples and averaged over 5 random seeds. Compute was 8$\times$A100 (40GB) GPUs servers. More details can be found in Appendix~\ref{app:implementation}.
\begin{table*}[t]
  \centering
  \small
  \caption{Main Results (Part I): \textbf{Pass@1} performance for \textbf{1.5B parameter models}. The training dataset is DeepScaleR. \colorbox{blue!10}{Blue} rows highlight our method.}
  \label{tab:main_results_part1}
  \resizebox{\linewidth}{!}{
    \begin{tabular}{lrrrrrr|r}
      \toprule
      \rowcolor{gray!15} \textbf{Dataset} & \textbf{AIME} & \textbf{AIME} & \textbf{AMC} & \textbf{Math500} & \textbf{Minerva} & \textbf{Olympiad} & \textbf{Average} \\
      \rowcolor{gray!15} & \textbf{2024} & \textbf{2025} & \textbf{2023} & & \textbf{MATH} & \textbf{Bench} & \\
      \midrule
      \multicolumn{8}{c}{\cellcolor{gray!5}\textbf{Qwen2.5-Math-1.5B}} \\
      \midrule
      Base (No RL) & 6.0 $\pm$ 1.9 & 4.0 $\pm$ 0.6 & 34.2 $\pm$ 0.2 & 47.5 $\pm$ 0.3 & 5.1 $\pm$ 0.4 & 25.1 $\pm$ 0.4 & 20.3 $\pm$ 0.6 \\
      \midrule
      \multicolumn{8}{l}{\textit{Oracle Reward (w. Verifier)}} \\
      No Curriculum & 15.0 $\pm$ 0.4 & 5.6 $\pm$ 0.6 & 50.3 $\pm$ 0.6 & 69.4 $\pm$ 0.4 & 17.8 $\pm$ 0.6 & 31.6 $\pm$ 0.0 & 31.6 $\pm$ 0.4 \\
      VCRL & 12.9 $\pm$ 2.0 & 5.5 $\pm$ 1.0 & 48.8 $\pm$ 2.0 & 69.2 $\pm$ 0.2 & 18.7 $\pm$ 0.0 & 31.3 $\pm$ 0.2 & 31.1 $\pm$ 0.9 \\
      AdaRFT & 13.2 $\pm$ 1.5 & 6.2 $\pm$ 1.4 & 48.6 $\pm$ 2.0 & 68.7 $\pm$ 0.4 & 20.1 $\pm$ 0.5 & 29.0 $\pm$ 0.3 & 31.0 $\pm$ 1.0 \\
      SENT & 15.6 $\pm$ 1.3 & 8.1 $\pm$ 1.9 & 52.1 $\pm$ 1.1 & 70.0 $\pm$ 1.0 & 14.3 $\pm$ 1.9 & 32.0 $\pm$ 0.7 & 32.0 $\pm$ 1.3 \\
      \rowcolor{blue!10} VI-CuRL & \textbf{15.9 $\pm$ 1.2} & \textbf{9.6 $\pm$ 0.7} & \textbf{52.9 $\pm$ 1.3} & \textbf{70.1 $\pm$ 1.5} & \textbf{24.3 $\pm$ 1.4} & \textbf{34.8 $\pm$ 3.9} & \textbf{34.6 $\pm$ 1.7} \\
      \midrule
      \multicolumn{8}{l}{\textit{Verifier-Free: Majority Vote}} \\
      TTRL & 0.0 $\pm$ 0.0 & 0.0 $\pm$ 0.0 & 0.0 $\pm$ 0.0 & 0.2 $\pm$ 0.0 & 0.4 $\pm$ 0.0 & 0.0 $\pm$ 0.0 & 0.1 $\pm$ 0.0 \\
      \rowcolor{blue!10} VI-CuRL & \textbf{15.9 $\pm$ 1.1} & \textbf{8.7 $\pm$ 1.1} & \textbf{53.1 $\pm$ 1.3} & \textbf{72.3 $\pm$ 0.6} & \textbf{21.6 $\pm$ 4.0} & \textbf{33.5 $\pm$ 4.5} & \textbf{34.2 $\pm$ 2.1} \\
      \midrule
      \multicolumn{8}{l}{\textit{Verifier-Free: Entropy}} \\
      RENT & 0.0 $\pm$ 0.0 & 0.0 $\pm$ 0.0 & 10.8 $\pm$ 0.5 & 9.4 $\pm$ 0.1 & 4.6 $\pm$ 0.1 & 3.7 $\pm$ 0.2 & 4.8 $\pm$ 0.2 \\
      EMPO & 5.2 $\pm$ 1.0 & 3.1 $\pm$ 0.8 & 15.4 $\pm$ 2.1 & 18.2 $\pm$ 1.5 & 6.5 $\pm$ 1.2 & 8.1 $\pm$ 0.9 & 9.4 $\pm$ 1.3 \\
      \rowcolor{blue!10} VI-CuRL & \textbf{13.8 $\pm$ 0.8} & \textbf{6.0 $\pm$ 0.7} & \textbf{48.6 $\pm$ 0.9} & \textbf{69.9 $\pm$ 1.9} & \textbf{19.7 $\pm$ 4.0} & \textbf{32.0 $\pm$ 0.3} & \textbf{31.7 $\pm$ 1.4} \\
      \midrule
      \multicolumn{8}{c}{\cellcolor{gray!5}\textbf{DeepSeek-R1-Distill-Qwen-1.5B}} \\
      \midrule
      Base (No RL) & 29.5 $\pm$ 0.4 & 20.6 $\pm$ 0.9 & 69.3 $\pm$ 0.6 & 83.1 $\pm$ 0.3 & 24.9 $\pm$ 0.6 & 43.6 $\pm$ 0.5 & 45.2 $\pm$ 0.6 \\
      \midrule
      \multicolumn{8}{l}{\textit{Oracle Reward (w. Verifier)}} \\
      No Curriculum & 31.0 $\pm$ 0.7 & 26.3 $\pm$ 0.5 & 79.1 $\pm$ 0.9 & 88.7 $\pm$ 0.4 & 29.5 $\pm$ 0.1 & 50.1 $\pm$ 0.2 & 50.8 $\pm$ 0.5 \\
      VCRL & 33.2 $\pm$ 1.4 & 26.8 $\pm$ 1.6 & 80.2 $\pm$ 0.6 & 90.8 $\pm$ 0.8 & 30.1 $\pm$ 2.0 & 55.5 $\pm$ 4.6 & 52.8 $\pm$ 1.8 \\
      AdaRFT & 35.0 $\pm$ 0.9 & \textbf{29.9 $\pm$ 1.7} & 82.0 $\pm$ 1.6 & 91.1 $\pm$ 1.0 & 32.7 $\pm$ 1.0 & 56.3 $\pm$ 1.0 & 54.5 $\pm$ 1.2 \\
      SENT & 34.7 $\pm$ 2.2 & 28.1 $\pm$ 1.3 & 81.6 $\pm$ 1.2 & 91.2 $\pm$ 1.1 & 33.0 $\pm$ 1.5 & 55.5 $\pm$ 1.2 & 54.0 $\pm$ 1.4 \\
      \rowcolor{blue!10} VI-CuRL & \textbf{37.1 $\pm$ 1.2} & 29.6 $\pm$ 1.9 & \textbf{82.7 $\pm$ 1.3} & \textbf{92.3 $\pm$ 1.1} & \textbf{33.8 $\pm$ 3.4} & \textbf{58.6 $\pm$ 2.5} & \textbf{55.7 $\pm$ 1.9} \\
      \midrule
      \multicolumn{8}{l}{\textit{Verifier-Free: Majority Vote}} \\
      TTRL & 24.1 $\pm$ 1.2 & 21.1 $\pm$ 1.6 & 64.0 $\pm$ 1.7 & 75.4 $\pm$ 0.8 & 23.1 $\pm$ 0.9 & 39.5 $\pm$ 0.6 & 41.2 $\pm$ 1.1 \\
      \rowcolor{blue!10} VI-CuRL & \textbf{40.6 $\pm$ 1.7} & \textbf{34.0 $\pm$ 2.5} & \textbf{87.4 $\pm$ 0.9} & \textbf{93.5 $\pm$ 1.0} & \textbf{35.0 $\pm$ 0.5} & \textbf{60.9 $\pm$ 1.3} & \textbf{58.6 $\pm$ 1.3} \\
      \midrule
      \multicolumn{8}{l}{\textit{Verifier-Free: Entropy}} \\
      RENT & 25.3 $\pm$ 0.5 & 23.2 $\pm$ 0.5 & 71.1 $\pm$ 1.0 & 85.3 $\pm$ 0.8 & 31.5 $\pm$ 2.8 & 48.0 $\pm$ 2.0 & 47.4 $\pm$ 1.3 \\
      EMPO & 34.0 $\pm$ 2.4 & 28.0 $\pm$ 1.3 & 73.6 $\pm$ 1.5 & 83.7 $\pm$ 1.2 & 31.9 $\pm$ 2.4 & 41.6 $\pm$ 1.3 & 48.8 $\pm$ 1.7 \\
      \rowcolor{blue!10} VI-CuRL & \textbf{38.0 $\pm$ 1.4} & \textbf{31.0 $\pm$ 1.0} & \textbf{83.8 $\pm$ 1.0} & \textbf{92.8 $\pm$ 0.3} & \textbf{32.8 $\pm$ 1.1} & \textbf{57.8 $\pm$ 0.5} & \textbf{56.0 $\pm$ 0.9} \\
      \bottomrule
    \end{tabular}
  }
\end{table*}

\begin{table*}[t]
  \centering
  \small
  \caption{Main Results (Part II): \textbf{Pass@1} performance for \textbf{Llama3.2-3B-Instruct} and \textbf{Qwen2.5-Math-7B}. The training dataset is DeepScaleR. \colorbox{blue!10}{Blue} rows highlight our method.}
  \label{tab:main_results_part2}
  \resizebox{\linewidth}{!}{
    \begin{tabular}{lrrrrrr|r}
      \toprule
      \rowcolor{gray!15} \textbf{Dataset} & \textbf{AIME} & \textbf{AIME} & \textbf{AMC} & \textbf{Math500} & \textbf{Minerva} & \textbf{Olympiad} & \textbf{Average} \\
      \rowcolor{gray!15} & \textbf{2024} & \textbf{2025} & \textbf{2023} & & \textbf{MATH} & \textbf{Bench} & \\
      \midrule
      \multicolumn{8}{c}{\cellcolor{gray!5}\textbf{Llama3.2-3B-Instruct}} \\
      \midrule
      Base (No RL) & 5.7 $\pm$ 1.2 & 0.6 $\pm$ 0.4 & 17.2 $\pm$ 1.5 & 34.8 $\pm$ 0.7 & 4.8 $\pm$ 0.3 & 12.7 $\pm$ 0.2 & 12.6 $\pm$ 0.7 \\
      \midrule
      \multicolumn{8}{l}{\textit{Oracle Reward (w. Verifier)}} \\
      No Curriculum & 6.0 $\pm$ 1.0 & 1.0 $\pm$ 0.4 & 24.3 $\pm$ 1.7 & 40.4 $\pm$ 0.5 & 8.9 $\pm$ 0.7 & 15.0 $\pm$ 0.5 & 15.9 $\pm$ 0.8 \\
      VCRL & 4.0 $\pm$ 1.1 & 1.3 $\pm$ 0.9 & 24.1 $\pm$ 1.4 & 37.0 $\pm$ 0.3 & 7.8 $\pm$ 0.2 & 17.1 $\pm$ 0.8 & 15.2 $\pm$ 0.8 \\
      AdaRFT & 4.5 $\pm$ 1.6 & 1.6 $\pm$ 0.7 & 25.7 $\pm$ 0.4 & \textbf{44.0 $\pm$ 0.3} & 8.5 $\pm$ 0.1 & 17.8 $\pm$ 0.7 & 17.0 $\pm$ 0.6 \\
      SENT & 6.5 $\pm$ 0.8 & 1.4 $\pm$ 0.4 & 26.7 $\pm$ 1.5 & 39.9 $\pm$ 0.8 & 8.7 $\pm$ 1.5 & 18.6 $\pm$ 1.4 & 17.0 $\pm$ 1.1 \\
      \rowcolor{blue!10} VI-CuRL & \textbf{6.8 $\pm$ 0.9} & \textbf{1.9 $\pm$ 0.3} & \textbf{29.7 $\pm$ 1.1} & 43.0 $\pm$ 0.7 & \textbf{9.4 $\pm$ 1.1} & \textbf{19.5 $\pm$ 0.7} & \textbf{18.4 $\pm$ 0.8} \\
      \midrule
      \multicolumn{8}{l}{\textit{Verifier-Free: Majority Vote}} \\
      TTRL & \textbf{5.7 $\pm$ 1.3} & 0.6 $\pm$ 0.6 & 18.5 $\pm$ 2.3 & 41.9 $\pm$ 1.0 & 8.8 $\pm$ 0.3 & \textbf{15.4 $\pm$ 0.6} & 15.2 $\pm$ 1.0 \\
      \rowcolor{blue!10} VI-CuRL & 5.2 $\pm$ 1.1 & \textbf{0.8 $\pm$ 0.5} & \textbf{19.6 $\pm$ 2.3} & \textbf{46.2 $\pm$ 1.5} & \textbf{10.8 $\pm$ 0.2} & 15.1 $\pm$ 0.6 & \textbf{16.3 $\pm$ 1.0} \\
      \midrule
      \multicolumn{8}{l}{\textit{Verifier-Free: Entropy}} \\
      RENT & 0.6 $\pm$ 0.6 & 0.0 $\pm$ 0.0 & 3.0 $\pm$ 1.1 & 3.2 $\pm$ 0.0 & 2.1 $\pm$ 0.2 & 1.2 $\pm$ 0.1 & 1.7 $\pm$ 0.3 \\
      EMPO & 2.1 $\pm$ 0.5 & 0.2 $\pm$ 0.1 & 8.5 $\pm$ 1.2 & 10.1 $\pm$ 1.1 & 4.5 $\pm$ 0.6 & 5.2 $\pm$ 0.8 & 5.1 $\pm$ 0.7 \\
      \rowcolor{blue!10} VI-CuRL & \textbf{6.2 $\pm$ 1.4} & \textbf{0.8 $\pm$ 0.5} & \textbf{21.6 $\pm$ 1.4} & \textbf{42.6 $\pm$ 0.9} & \textbf{9.7 $\pm$ 0.4} & \textbf{14.7 $\pm$ 0.5} & \textbf{15.9 $\pm$ 0.9} \\
      \midrule
      \multicolumn{8}{c}{\cellcolor{gray!5}\textbf{Qwen2.5-Math-7B}} \\
      \midrule
      Base (No RL) & 12.7 $\pm$ 0.9 & 5.8 $\pm$ 0.7 & 44.4 $\pm$ 2.2 & 52.0 $\pm$ 0.4 & 9.8 $\pm$ 0.7 & 26.4 $\pm$ 0.4 & 25.2 $\pm$ 0.9 \\
      \midrule
      \multicolumn{8}{l}{\textit{Oracle Reward (w. Verifier)}} \\
      No Curriculum & 29.2 $\pm$ 2.5 & 13.5 $\pm$ 1.2 & 62.8 $\pm$ 1.9 & \textbf{78.9 $\pm$ 0.5} & 24.2 $\pm$ 0.7 & 37.8 $\pm$ 0.5 & 41.1 $\pm$ 1.2 \\
      VCRL & 22.0 $\pm$ 1.5 & 11.1 $\pm$ 0.7 & 57.0 $\pm$ 1.6 & 75.1 $\pm$ 0.5 & 22.1 $\pm$ 0.7 & 35.6 $\pm$ 0.5 & 37.2 $\pm$ 0.9 \\
      AdaRFT & 26.9 $\pm$ 0.6 & 12.9 $\pm$ 0.7 & 63.0 $\pm$ 0.8 & 76.6 $\pm$ 1.9 & 22.5 $\pm$ 6.9 & 39.1 $\pm$ 2.6 & 40.2 $\pm$ 2.3 \\
      SENT & 26.2 $\pm$ 0.7 & 10.8 $\pm$ 0.9 & 53.2 $\pm$ 1.3 & 69.2 $\pm$ 1.2 & 22.4 $\pm$ 2.0 & 31.2 $\pm$ 1.0 & 35.5 $\pm$ 1.2 \\
      \rowcolor{blue!10} VI-CuRL & \textbf{29.9 $\pm$ 0.9} & \textbf{13.9 $\pm$ 1.4} & \textbf{63.8 $\pm$ 0.5} & 78.8 $\pm$ 0.6 & \textbf{26.3 $\pm$ 5.5} & \textbf{45.0 $\pm$ 3.0} & \textbf{43.0 $\pm$ 2.0} \\
      \midrule
      \multicolumn{8}{l}{\textit{Verifier-Free: Majority Vote}} \\
      TTRL & 29.0 $\pm$ 1.5 & 14.0 $\pm$ 1.0 & 64.8 $\pm$ 1.4 & \textbf{80.5 $\pm$ 0.4} & 22.8 $\pm$ 0.3 & 39.5 $\pm$ 0.4 & 41.8 $\pm$ 0.8 \\
      \rowcolor{blue!10} VI-CuRL & \textbf{33.9 $\pm$ 0.7} & \textbf{14.4 $\pm$ 0.9} & \textbf{65.4 $\pm$ 0.4} & 79.8 $\pm$ 0.5 & \textbf{25.9 $\pm$ 0.6} & \textbf{41.2 $\pm$ 0.4} & \textbf{43.4 $\pm$ 0.6} \\
      \midrule
      \multicolumn{8}{l}{\textit{Verifier-Free: Entropy}} \\
      RENT & 13.7 $\pm$ 0.6 & 7.2 $\pm$ 0.4 & 55.7 $\pm$ 1.1 & 74.4 $\pm$ 2.1 & \textbf{24.8 $\pm$ 4.5} & 32.9 $\pm$ 4.4 & 34.8 $\pm$ 2.2 \\
      EMPO & 18.5 $\pm$ 1.5 & 9.4 $\pm$ 0.8 & 58.2 $\pm$ 2.0 & 75.1 $\pm$ 1.2 & 21.0 $\pm$ 2.5 & 34.5 $\pm$ 1.8 & 36.1 $\pm$ 1.6 \\
      \rowcolor{blue!10} VI-CuRL & \textbf{25.0 $\pm$ 1.8} & \textbf{13.1 $\pm$ 1.1} & \textbf{63.0 $\pm$ 2.8} & \textbf{76.8 $\pm$ 0.6} & 19.1 $\pm$ 5.0 & \textbf{37.0 $\pm$ 0.7} & \textbf{39.0 $\pm$ 2.0} \\
      \bottomrule
    \end{tabular}
  }
\end{table*}

\begin{table*}[t]
  \centering
  \small
  \caption{General Knowledge Results: \textbf{Pass@1} performance for \textbf{Llama3.2-3B-Instruct}. The training dataset is WebInstruct-Verified. \colorbox{blue!10}{Blue} rows highlight our method.}
  \label{tab:general_results}
  \resizebox{0.9\linewidth}{!}{
    \begin{tabular}{lrrrr|r}
      \toprule
      \rowcolor{gray!15} \textbf{Dataset} & \textbf{MMLU-Pro} & \textbf{GPQA-Diamond} & \textbf{TheoremQA} & \textbf{WebInstruct-Val} & \textbf{Average} \\
      \midrule
      Base (No RL) & 34.8 $\pm$ 0.4 & 25.3 $\pm$ 0.3 & 5.1 $\pm$ 0.2 & 20.6 $\pm$ 0.4 & 21.5 $\pm$ 0.3 \\
      \midrule
      \multicolumn{6}{l}{\textit{Oracle Reward (w. Verifier)}} \\
      No Curriculum & 36.2 $\pm$ 0.5 & 26.3 $\pm$ 0.2 & 10.4 $\pm$ 0.4 & 26.0 $\pm$ 0.5 & 24.7 $\pm$ 0.4 \\
      VCRL & 36.5 $\pm$ 0.6 & 26.6 $\pm$ 0.2 & 11.4 $\pm$ 0.4 & 22.6 $\pm$ 0.5 & 24.3 $\pm$ 0.4 \\
      AdaRFT & 36.9 $\pm$ 0.3 & 25.3 $\pm$ 0.5 & 11.5 $\pm$ 0.6 & 22.4 $\pm$ 0.3 & 24.0 $\pm$ 0.4 \\
      SENT & 37.5 $\pm$ 0.7 & 25.8 $\pm$ 0.4 & 9.0 $\pm$ 0.3 & \textbf{28.6 $\pm$ 0.6} & 25.2 $\pm$ 0.5 \\
      \rowcolor{blue!10} VI-CuRL & \textbf{37.8 $\pm$ 0.5} & \textbf{27.8 $\pm$ 0.6} & \textbf{13.2 $\pm$ 0.4} & 28.4 $\pm$ 0.5 & \textbf{26.8 $\pm$ 0.5} \\
      \midrule
      \multicolumn{6}{l}{\textit{Verifier-Free: Majority Vote}} \\
      TTRL & 24.5 $\pm$ 0.3 & 24.3 $\pm$ 0.2 & 1.5 $\pm$ 0.2 & 0.6 $\pm$ 0.2 & 12.7 $\pm$ 0.2 \\
      \rowcolor{blue!10} VI-CuRL & \textbf{35.6 $\pm$ 0.5} & \textbf{26.7 $\pm$ 0.4} & \textbf{9.5 $\pm$ 0.3} & \textbf{21.4 $\pm$ 0.5} & \textbf{23.3 $\pm$ 0.4} \\
      \midrule
      \multicolumn{6}{l}{\textit{Verifier-Free: Entropy}} \\
      RENT & 0.0 $\pm$ 0.0 & 8.3 $\pm$ 0.4 & 0.0 $\pm$ 0.0 & 0.0 $\pm$ 0.0 & 2.1 $\pm$ 0.1 \\
      EMPO & 10.3 $\pm$ 0.5 & 17.5 $\pm$ 0.6 & 0.0 $\pm$ 0.0 & 11.0 $\pm$ 0.4 & 9.7 $\pm$ 0.4 \\
      \rowcolor{blue!10} VI-CuRL & \textbf{35.6 $\pm$ 0.4} & \textbf{25.7 $\pm$ 0.5} & \textbf{9.5 $\pm$ 0.3} & \textbf{21.4 $\pm$ 0.4} & \textbf{23.1 $\pm$ 0.4} \\
      \bottomrule
    \end{tabular}
  }
\end{table*}

\paragraph{Baselines} We evaluated VI-CuRL against two categories of baselines:
\begin{itemize}
    \item \textbf{Curriculum RL with Ground Truth}: We compared against \textbf{VCRL} \citep{jiang2025vcrl}, \textbf{AdaRFT} \citep{chang2024adarft}, and \textbf{SENT} \citep{cao2025efficient}. They were curriculum-based RLVR algorithms. While VCRL and AdaRFT fundamentally relied on access to ground-truth labels (verifiers) to guide the curriculum, unlike VI-CuRL which is verifier-independent. SENT constructs an offline semantic-entropy curriculum. In contrast, VI-CuRL builds the curriculum online at each policy update, so the curriculum explicitly trades transient selection bias for lower gradient variance and anneals back to the full distribution as $\beta_t \to 1$.
    \item \textbf{Standard RL (Oracle)}: Standard RLVR training using the true verifier reward. This served as a topline for performance.
    \item \textbf{TTRL}: A verifier-free baseline where the reward was determined by self-consistency (majority voting) among rollout samples~\citep{ttrl2025test}.
    \item \textbf{RENT}: A verifier-free baseline that used the prediction entropy of the model's outputs as the reward signal (penalizing high entropy)~\citep{rent2025maximizing}. VI-CuRL used RENT's reward signal in the verifier-free entropy setting.
    \item \textbf{EMPO}: Entropy Minimization Policy Optimization \citep{empo2025right}, an entropy-based verifier-free baseline that combines entropy minimization with policy optimization.
\end{itemize}
\subsection{Main Results}

\textbf{Mathematical Reasoning Performance.} Table~\ref{tab:main_results_part1} and Table~\ref{tab:main_results_part2} summarize the mathematical reasoning performance of our method compared to standard baselines. We evaluated the models across two distinct settings: Oracle Reward (with verifier) and Verifier-Free (intrinsic reward). In the Oracle setting, VI-CuRL consistently matches or outperforms standard curriculum RL methods (such as VCRL, AdaRFT, and SENT) across all model scales, from 1.5B to 7B parameters. Notably, on harder benchmarks like AIME and OlympiadBench, VI-CuRL achieves better prformance over the baselines, demonstrating that confidence-based curriculum selection can stabilize updates and improve the optimization trajectory even when external rewards are available. The benefits of VI-CuRL become critically pronounced in the Verifier-Free settings (Majority Vote and Entropy). While standard verifier-free methods frequently succumb to high variance and suffer from severe model collapse (often degrading to near $0.0\%$ accuracy on harder datasets), VI-CuRL effectively stabilizes the policy updates and prevents this degradation. By dynamically filtering out high-variance samples in the early stages, VI-CuRL recovers performance comparable to the Oracle topline across all model scales. This phenomenon highlights that managing gradient variance via our annealed confidence curriculum is absolutely essential for stable learning from noisy, intrinsic signals. Additional results for Pass@8 are detailed in Appendix~\ref{app:additional_results}.

\textbf{General Knowledge Capabilities.} Table~\ref{tab:general_results} extends our evaluation to general knowledge benchmarks using the Llama3.2-3B-Instruct model, assessing whether VI-CuRL generalizes beyond formal mathematical reasoning domains. The results demonstrate that VI-CuRL maintains consistent performance advantages across a diverse array of tasks. Similar to the math domain, VI-CuRL establishes strong performance in the Oracle setting (achieving an average of 26.8\% compared to 25.2\% for the strongest baseline, SENT). More importantly, in the challenging Verifier-Free settings, VI-CuRL prevents the catastrophic policy collapse universally observed in verifier-free baselines. These findings provide compelling evidence that the variance-bias trade-off curriculum generalizes seamlessly to broad, open-ended applications, effectively stabilizing training and maintaining policy integrity without relying on domain-specific structures or ground-truth verifiers.

\subsection{Does VI-CuRL Really Reduce Variance?}
A key theoretical advantage of VI-CuRL is variance reduction through curriculum selection. According to Theorem~\ref{thm:variance_decomposition}, the variance of the curriculum-selected estimator depends on:

\textbf{Action Variance:} $\sigma_{g,t}^2 := \E_{x|w_t=1}[\Var_{y|x}(g)]$: the expected gradient variance due to sampling different rollouts for the \emph{same} prompt within the selected set.

\textbf{Problem Variance:} $V_{\mathrm{prob},t} := \Var_{x|w_t=1}(\bar{g})$: the variance of mean gradients \emph{across} different prompts within the selected set.

To empirically validate this, we estimated these variance components at each training step $t$ using Monte-Carlo approximations. For $K$ independent rollouts per-prompt,
\begin{equation*}
    \hat{\sigma}^2_{g,t} = \frac{1}{|\mathcal{S}_t|} \sum_{x \in \mathcal{S}_t} \widehat{\Var}_{y|x}(g(x,y)), \quad \hat{V}_{\mathrm{prob},t} = \widehat{\Var}_{x \in \mathcal{S}_t}(\bar{g}(x)),
\end{equation*}
where $\mathcal{S}_t$ denotes either the curriculum-selected subset (``kept'') or the full dataset (``full'').

\textbf{Validating Theorem~\ref{thm:variance_decomposition}.} The theorem predicted that the curriculum-selected subset $\mathcal{S}_t^{\text{kept}}$ should have lower action/problem variance than the full dataset $\mathcal{S}_t^{\text{full}}$ when $\beta_t < 1$. To test this, we computed
\begin{equation*}
    \mathrm{Ratio}_{\sigma,t} = \frac{\hat{\sigma}^2_{g,t}(\text{kept})}{\hat{\sigma}^2_{g,t}(\text{full})}, \quad \mathrm{Ratio}_{V_t} = \frac{\hat{V}_{\mathrm{prob},t}(\text{kept})}{\hat{V}_{\mathrm{prob},t}(\text{full})}.
\end{equation*}
If the curriculum effectively reduces variance, we expected $\mathrm{Ratio}_{\sigma,t}, \mathrm{Ratio}_{V_t} < 1$ when $\beta_t < 1$, and $\mathrm{Ratio}_{\sigma,t}, \mathrm{Ratio}_{V_t} \approx 1$ when $\beta_t \to 1$.

\begin{figure*}[t!]
    \centering
    \setlength{\tabcolsep}{2pt}
    \begin{tabular}{cc}
        \includegraphics[width=0.48\textwidth]{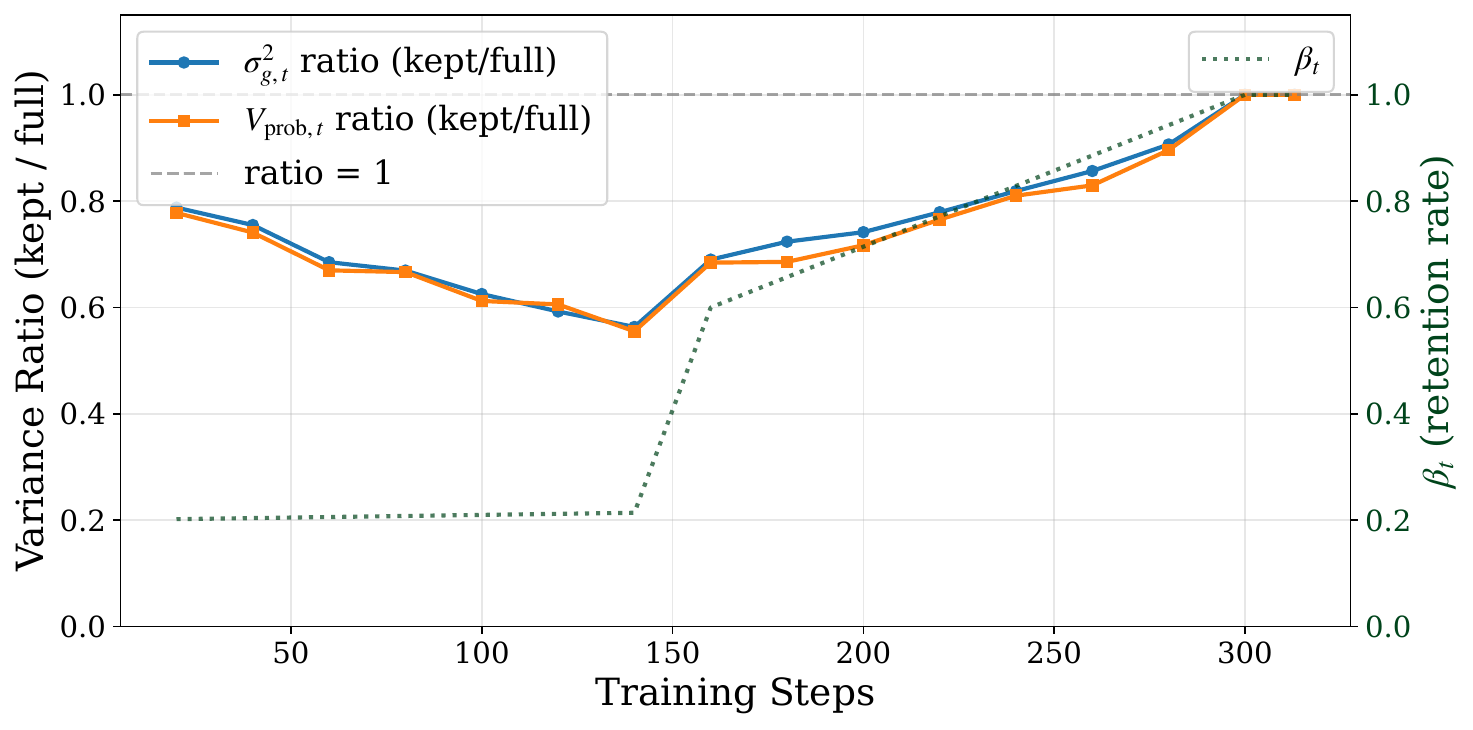} &
        \includegraphics[width=0.48\textwidth]{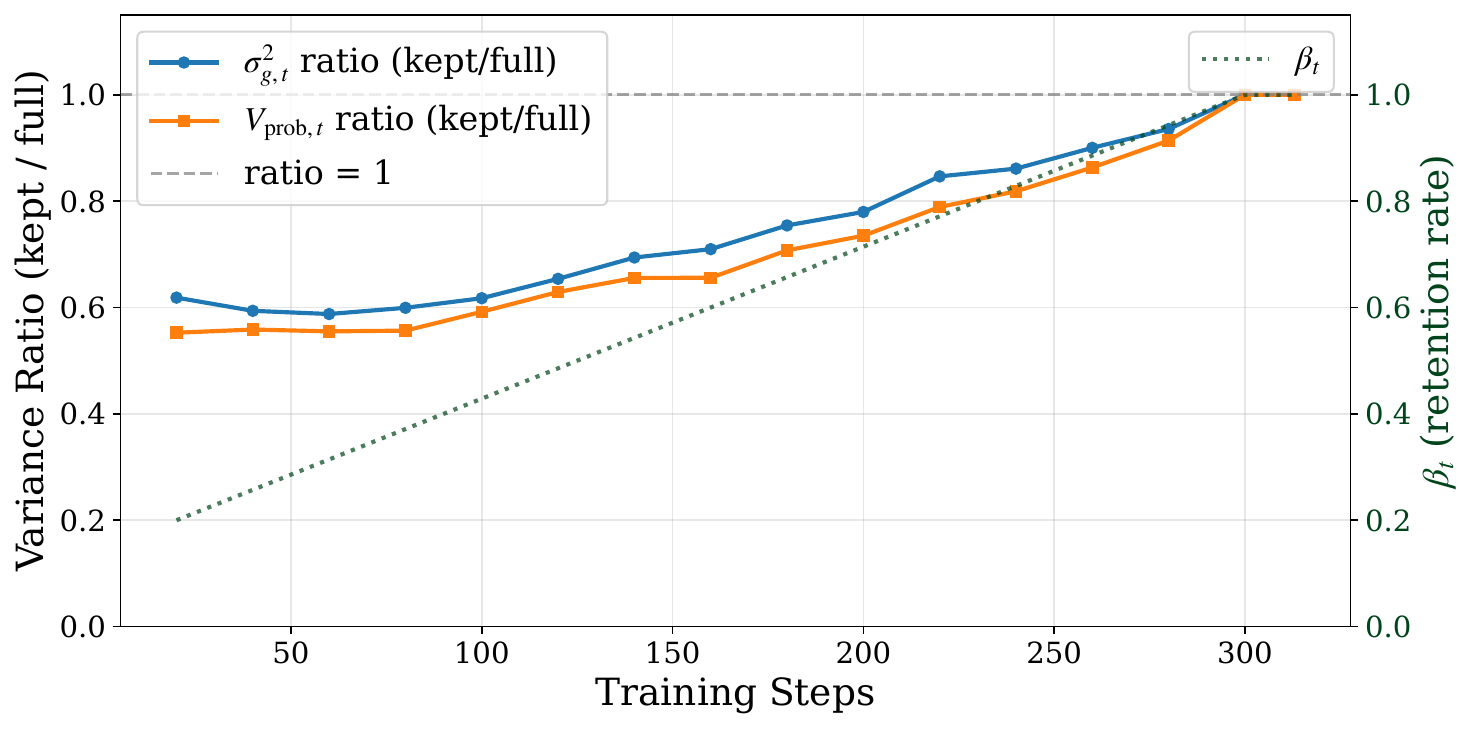} \\
        (a) Qwen2.5-Math-1.5B & (b) DeepSeek-R1-Distill-Qwen-1.5B
    \end{tabular}
    \caption{\textbf{Variance Ratio Analysis.} Variance ratio (kept/full) for Action Variance (\textcolor[HTML]{1f77b4}{$\sigma_{g,t}^2$}) and Problem Variance (\textcolor[HTML]{ff7f0e}{$V_{\mathrm{prob},t}$}) alongside the retention rate  (\textcolor[HTML]{00441b}{$\beta_t$}, right axis). When $\beta_t < 1$, both ratios are consistently below 1, confirming that curriculum selection reduces variance as predicted by Theorem~\ref{thm:variance_decomposition}. As $\beta_t \to 1$, the ratios approach 1 since the selected set converges to the full dataset.}
    \label{fig:variance_ratio}
\end{figure*}

Figure~\ref{fig:variance_ratio} presented the variance ratios for both Qwen2.5-Math-1.5B and DeepSeek-R1-Distill-Qwen-1.5B under the Oracle (ground-truth verifier) setting. We can oberve that when $\beta_t \approx 0.2$ (early training), the variance ratios ranged from $0.55$ to $0.80$, indicating a $20$--$45\%$ reduction in variance compared to using the full dataset. Meanwhile, as $\beta_t$ increased towards $1$, the variance ratios monotonically increased toward $1$, as predicted by the theorem. Besides, action Variance and Problem Variance exhibited similar reduction patterns, confirming that the curriculum simultaneously reduced sampling variance (within prompts) and heterogeneity (across prompts). Therefore, the curriculum selection strategy successfully reduced variance as predicted by Theorem~\ref{thm:variance_decomposition}. Additional analysis of the absolute variance values is provided in Appendix~\ref{app:additional_results}.

\section{Conclusion}\label{sec:conclusion}
In this work, we presented \textbf{Verifier-Independent Curriculum Reinforcement Learning (VI-CuRL)}, a theoretically grounded framework designed to stabilize Reinforcement Learning with Verifiable Rewards (RLVR) especially when reliable verifiers are absent. Recognizing that the primary bottleneck in RL is the extreme problem-level variance, VI-CuRL introduces a dynamic, confidence-based curriculum that allows the policy to ``start low-uncertainty'' by prioritizing high-confidence generations. We formalized this strategy as a principled bias-variance trade-off, proving that our curriculum estimator significantly reduces variance bounds during early training while guaranteeing asymptotic convergence to the true objective. Empirically, VI-CuRL acts as a robust stabilizer for Group Relative Policy Optimization (GRPO), outperforming verifier-based/free baselines across math and general knowledge benchmarks. By effectively filtering out high-variance gradients through intrinsic confidence, VI-CuRL paves the way for scalable, self-improving reasoning models in open-ended domains where ground-truth supervision is scarce. So far VI-CuRL has been verified successfully on reasoning benchmarks, and we plan to extend it to other RL domains in the future.

\section*{Acknowledgements}
We thank the anonymous reviewers for their constructive feedback. We also thank Zikang Shan for his valuable discussions. X-Q C was supported by JSPS KAKENHI Grant Number JP26K21323. MS was supported by JST ASPIRE Grant Number JPMJAP2405.

\bibliographystyle{plainnat}

\bibliography{references}

\newpage

\appendix
\onecolumn

\section{Related Work}
\paragraph{LLM Reasoning.}
Recent advances in LLM reasoning have been significantly propelled by RL~\citep{schulman2017proximal,NEURIPS2023_32285dd1}, especially for the development of LLM-based agents~\citep{tu2026dynamic,liu2026toolanchor,yang2026batched}. Early iterative approaches like Self-Taught Reasoner (STaR) \citep{zelikman2022star} demonstrated that models could self-improve by bootstrapping from their own verified correct solutions. \rev{\citet{shao2024deepseekmath}} introduced GRPO to achieve strong mathematical reasoning with improved sample efficiency, while DeepSeek-R1 \citep{deepseek2025r1} demonstrated that sophisticated behaviors can emerge purely from RL without supervised fine-tuning. The RLVR paradigm \citep{zhou2025reinforcing,zhao2026stride} formalized this approach. Parallel work on Process Reward Models (PRMs) \citep{lightman2024letsverify} has explored dense, step-level supervision to guide reasoning, though this often requires expensive human annotation. Recent verifier-free approaches have also gained traction: Reinforcement Learning via Entropy Minimization (RENT) \citep{rent2025maximizing} and Entropy Minimization Policy Optimization (EMPO) \citep{empo2025right} leverage entropy minimization as a fully unsupervised reward signal, while the Entropy Mechanism \citep{primerl2025entropy} stabilizes training by managing policy collapse. Test-Time Reinforcement Learning (TTRL) \citep{ttrl2025test} further explores converting test-time compute (e.g., majority voting) into training signals. However, most of these methods fundamentally rely on access to ground-truth verifiers or correctness labels, while verifier-free methods tend to be unstable and fragile during training. \rev{VI-CuRL addresses the complementary data-selection problem: it uses intrinsic confidence to rank prompts, while the chosen oracle or intrinsic reward still supplies the optimization signal.}

\paragraph{Bias-Variance Trade-off in Policy Gradients.}
The high variance of policy gradient estimators has long been recognized as a central challenge in RL \citep{Williams1992}. \citet{sutton1999policy} established the theoretical foundations for variance reduction through baselines. Modern LLM reasoning pipelines typically employ Group Relative Policy Optimization (GRPO) \rev{\citep{shao2024deepseekmath}}, which estimates a baseline from the mean reward of a group of sampled outputs. \rev{A group-relative baseline and prompt selection act on different aspects of the update; changing the optimizer does not eliminate the need to evaluate the prompt-distribution intervention separately.} In verifier-independent settings, the problem-level variance can be extreme due to the lack of accurate reward signals, necessitating a curriculum approach like VI-CuRL to curate the training distribution itself.

\paragraph{Curriculum Learning.}
Curriculum learning, first formalized by~\citet{bengio2009curriculum}, proposes training models on progressively harder examples to improve convergence and generalization. Self-paced learning \citep{kumar2010self} extends this by allowing the model itself to determine example difficulty based on training loss, with self-paced curriculum learning \citep{jiang2015self} unifying both paradigms. In the context of RL for LLMs, AdaRFT \citep{chang2024adarft} dynamically adjusts problem difficulty based on recent reward signals, while VCRL \citep{jiang2025vcrl} uses reward variance to identify moderately difficult samples for training. However, both AdaRFT and VCRL fundamentally require ground-truth reward signals from external verifiers to compute their curriculum metrics. Our VI-CuRL is verifier-independent: it uses the model's confidence as an intrinsic proxy for variance control, enabling curriculum-based training in domains where verifiers are unavailable. SENT~\citep{cao2025efficient} organizes training data with a semantic-entropy curriculum computed offline from the initial policy. However, the curriculum ranking induced by the initial policy may become mismatched after RL updates shift the policy distribution. In contrast, VI-CuRL computes prompt-level predictive confidence online from the current behavior policy and anneals an importance-weighted selection mask, allowing the curriculum to adaptively trade transient bias for variance reduction during optimization. LILO~\citep{foster2026learningreasonfrontierlearnability} prioritizes learnability through binary-success variance $p(x)(1-p(x))$. That outcome statistic differs from the covariance of vector surrogate gradients analyzed here. A prompt can offer a nonzero learning signal while its rollout gradients remain directionally incoherent. Likewise, high-entropy minority tokens~\citep{wang20258020rulehighentropyminority} locate informative decisions within a trajectory; VI-CuRL ranks prompts using average uncertainty and retains all valid response tokens by default.

\begin{table*}[t]
  \centering
  \small
  \caption{Comparison with \textbf{Oracle (w. Verifier)} reward. Mean and standard deviation \textbf{(pass@8)} with 16 samples and 5 random seeds. We compare VI-CuRL against Curriculum baselines (VCRL, AdaRFT, SENT) and No Curriculum. \colorbox{blue!10}{Blue} rows highlight our method.}
  \label{tab:main_results_oracle_pass@8}
  \resizebox{\linewidth}{!}{
    \begin{tabular}{lrrrrrr|r}
      \toprule
      \rowcolor{gray!15} \textbf{Dataset} & \textbf{AIME} & \textbf{AIME} & \textbf{AMC} & \textbf{Math500} & \textbf{Minerva} & \textbf{Olympiad} & \textbf{Average} \\
      \rowcolor{gray!15} & \textbf{2024} & \textbf{2025} & \textbf{2023} & & \textbf{MATH} & \textbf{Bench} & \\
      \midrule
      \multicolumn{8}{c}{\cellcolor{gray!5}\textbf{Qwen2.5-Math-1.5B}} \\
      \midrule
      Base (No RL) & 32.4 $\pm$ 0.5 & 16.7 $\pm$ 0.7 & 79.3 $\pm$ 1.4 & 62.8 $\pm$ 1.3 & 15.5 $\pm$ 1.1 & 31.3 $\pm$ 0.8 & 39.7 $\pm$ 1.0 \\
      \midrule
      \multicolumn{8}{l}{\textit{Oracle (w. Verifier)}} \\
      No Curriculum & 34.3 $\pm$ 1.2 & 19.9 $\pm$ 1.1 & 80.7 $\pm$ 1.4 & 69.6 $\pm$ 0.5 & 18.0 $\pm$ 0.6 & 32.8 $\pm$ 1.2 & 42.6 $\pm$ 1.0 \\
      VCRL & 32.2 $\pm$ 0.5 & 19.8 $\pm$ 1.0 & 79.0 $\pm$ 0.9 & 70.2 $\pm$ 1.1 & 20.4 $\pm$ 1.0 & 34.0 $\pm$ 1.1 & 42.6 $\pm$ 0.9 \\
      AdaRFT & 32.1 $\pm$ 0.4 & 17.9 $\pm$ 1.5 & 77.5 $\pm$ 1.0 & \textbf{72.4 $\pm$ 0.9} & 21.7 $\pm$ 0.9 & 34.4 $\pm$ 0.5 & 42.7 $\pm$ 0.9 \\
      SENT & 35.0 $\pm$ 0.4 & \textbf{26.7 $\pm$ 0.4} & \textbf{82.5 $\pm$ 0.7} & 71.4 $\pm$ 0.7 & 15.3 $\pm$ 1.4 & 33.2 $\pm$ 0.9 & 44.0 $\pm$ 0.8 \\
      \rowcolor{blue!10} VI-CuRL & \textbf{35.2 $\pm$ 1.4} & 24.6 $\pm$ 0.2 & \textbf{82.5 $\pm$ 1.2} & 71.3 $\pm$ 1.8 & \textbf{26.9 $\pm$ 1.3} & \textbf{37.9 $\pm$ 4.0} & \textbf{46.4 $\pm$ 1.7} \\
      \midrule
      \multicolumn{8}{c}{\cellcolor{gray!5}\textbf{DeepSeek-R1-Distill-Qwen-1.5B}} \\
      \midrule
      Base (No RL) & 44.8 $\pm$ 0.8 & 32.8 $\pm$ 0.7 & 88.2 $\pm$ 0.1 & 85.9 $\pm$ 0.6 & 27.8 $\pm$ 0.3 & 48.6 $\pm$ 0.3 & 54.7 $\pm$ 0.5 \\
      \midrule
      \multicolumn{8}{l}{\textit{Oracle (w. Verifier)}} \\
      No Curriculum & 55.1 $\pm$ 0.6 & 40.1 $\pm$ 0.4 & 96.1 $\pm$ 0.2 & 96.6 $\pm$ 0.1 & 44.6 $\pm$ 0.1 & 64.9 $\pm$ 0.1 & 66.2 $\pm$ 0.3 \\
      VCRL & 55.0 $\pm$ 0.9 & 42.1 $\pm$ 0.9 & 96.8 $\pm$ 1.1 & 97.5 $\pm$ 0.5 & 42.3 $\pm$ 1.7 & 66.9 $\pm$ 3.5 & 66.8 $\pm$ 1.4 \\
      AdaRFT & 57.4 $\pm$ 1.3 & 43.2 $\pm$ 1.9 & 97.9 $\pm$ 0.3 & 98.3 $\pm$ 0.7 & 46.9 $\pm$ 1.0 & 70.0 $\pm$ 1.2 & 69.0 $\pm$ 1.1 \\
      SENT & 57.1 $\pm$ 0.7 & 39.9 $\pm$ 1.5 & 97.1 $\pm$ 0.9 & 98.1 $\pm$ 0.7 & 49.6 $\pm$ 1.3 & 71.4 $\pm$ 0.8 & 68.9 $\pm$ 1.0 \\
      \rowcolor{blue!10} VI-CuRL & \textbf{60.9 $\pm$ 2.3} & \textbf{45.5 $\pm$ 2.6} & \textbf{98.9 $\pm$ 1.3} & \textbf{98.8 $\pm$ 0.4} & \textbf{51.3 $\pm$ 1.3} & \textbf{72.8 $\pm$ 1.5} & \textbf{71.4 $\pm$ 1.6} \\
      \midrule
      \multicolumn{8}{c}{\cellcolor{gray!5}\textbf{Llama3.2-3B-Instruct}} \\
      \midrule
      Base (No RL) & 18.8 $\pm$ 1.1 & 3.3 $\pm$ 0.2 & 54.7 $\pm$ 1.4 & 40.8 $\pm$ 0.8 & 5.2 $\pm$ 1.1 & 16.4 $\pm$ 0.5 & 23.2 $\pm$ 0.9 \\
      \midrule
      \multicolumn{8}{l}{\textit{Oracle (w. Verifier)}} \\
      No Curriculum & 22.4 $\pm$ 1.0 & 6.7 $\pm$ 1.4 & 60.0 $\pm$ 1.3 & 44.8 $\pm$ 0.6 & 13.3 $\pm$ 0.7 & 18.7 $\pm$ 1.5 & 27.7 $\pm$ 1.1 \\
      VCRL & 23.5 $\pm$ 0.4 & 8.2 $\pm$ 1.1 & 61.4 $\pm$ 1.3 & 45.9 $\pm$ 0.9 & 15.3 $\pm$ 0.7 & 17.5 $\pm$ 0.5 & 28.6 $\pm$ 0.8 \\
      AdaRFT & 24.9 $\pm$ 1.8 & 8.0 $\pm$ 0.7 & \textbf{67.9 $\pm$ 1.4} & 45.6 $\pm$ 0.2 & 16.3 $\pm$ 1.9 & 18.1 $\pm$ 1.6 & 30.1 $\pm$ 1.3 \\
      SENT & 25.1 $\pm$ 0.8 & 8.1 $\pm$ 0.7 & 62.2 $\pm$ 1.0 & 48.0 $\pm$ 1.5 & 14.6 $\pm$ 0.4 & 20.1 $\pm$ 1.0 & 29.7 $\pm$ 0.9 \\
      \rowcolor{blue!10} VI-CuRL & \textbf{26.1 $\pm$ 0.8} & \textbf{9.3 $\pm$ 0.7} & 63.1 $\pm$ 1.4 & \textbf{49.6 $\pm$ 1.4} & \textbf{17.1 $\pm$ 0.8} & \textbf{22.9 $\pm$ 1.2} & \textbf{31.4 $\pm$ 1.1} \\
      \midrule
      \multicolumn{8}{c}{\cellcolor{gray!5}\textbf{Qwen2.5-Math-7B}} \\
      \midrule
      Base (No RL) & 36.2 $\pm$ 1.2 & 23.3 $\pm$ 1.0 & 81.9 $\pm$ 1.3 & 66.6 $\pm$ 0.5 & 11.8 $\pm$ 0.6 & 33.9 $\pm$ 0.7 & 42.3 $\pm$ 0.9 \\
      \midrule
      \multicolumn{8}{l}{\textit{Oracle (w. Verifier)}} \\
      No Curriculum & 50.0 $\pm$ 0.1 & 28.7 $\pm$ 0.8 & 83.4 $\pm$ 0.6 & 82.8 $\pm$ 0.5 & 29.0 $\pm$ 1.1 & 45.8 $\pm$ 1.3 & 53.3 $\pm$ 0.7 \\
      VCRL & 48.3 $\pm$ 1.4 & 23.3 $\pm$ 0.5 & 84.7 $\pm$ 0.6 & 79.8 $\pm$ 1.5 & 24.3 $\pm$ 0.4 & 41.0 $\pm$ 0.2 & 50.2 $\pm$ 0.8 \\
      AdaRFT & 50.0 $\pm$ 3.1 & 27.9 $\pm$ 3.0 & 85.0 $\pm$ 1.5 & 78.0 $\pm$ 2.0 & 22.8 $\pm$ 6.6 & 44.5 $\pm$ 1.1 & 51.4 $\pm$ 2.9 \\
      SENT & 55.7 $\pm$ 1.9 & 17.5 $\pm$ 0.7 & 64.7 $\pm$ 0.8 & \textbf{85.8 $\pm$ 0.2} & 35.7 $\pm$ 0.7 & 50.2 $\pm$ 0.8 & 51.6 $\pm$ 0.9 \\
      \rowcolor{blue!10} VI-CuRL & \textbf{56.4 $\pm$ 1.3} & \textbf{30.8 $\pm$ 0.2} & \textbf{88.8 $\pm$ 0.7} & 82.6 $\pm$ 0.4 & \textbf{37.1 $\pm$ 0.9} & \textbf{51.4 $\pm$ 1.2} & \textbf{57.9 $\pm$ 0.8} \\
      \bottomrule
    \end{tabular}
  }
\end{table*}

\begin{table*}[t]
  \centering
  \small
  \caption{Comparison in \textbf{Verifier-Free} settings. Mean and standard deviation \textbf{(pass@8)} with 16 samples and 5 random seeds. We compare VI-CuRL against baselines using \textbf{Majority Vote} and \textbf{Entropy} as intrinsic reward signals. Note that VCRL, AdaRFT, and SENT are excluded as they require ground-truth verifiers or offline initial policy distribution. \colorbox{blue!10}{Blue} rows highlight our method.}
  \label{tab:main_results_independent_pass@8}
  \resizebox{\linewidth}{!}{
    \begin{tabular}{lrrrrrr|r}
      \toprule
      \rowcolor{gray!15} \textbf{Dataset} & \textbf{AIME} & \textbf{AIME} & \textbf{AMC} & \textbf{Math500} & \textbf{Minerva} & \textbf{Olympiad} & \textbf{Average} \\
      \rowcolor{gray!15} & \textbf{2024} & \textbf{2025} & \textbf{2023} & & \textbf{MATH} & \textbf{Bench} & \\
      \midrule
      \multicolumn{8}{c}{\cellcolor{gray!5}\textbf{Qwen2.5-Math-1.5B}} \\
      \midrule
      Base (No RL) & 32.4 $\pm$ 0.5 & 16.7 $\pm$ 0.7 & 79.3 $\pm$ 1.4 & 62.8 $\pm$ 1.3 & 15.5 $\pm$ 1.1 & 31.3 $\pm$ 0.8 & 39.7 $\pm$ 1.0 \\
      \midrule
      \multicolumn{8}{l}{\textit{Majority Vote (w/o. Verifier)}} \\
      TTRL & 0.0 $\pm$ 0.0 & 0.0 $\pm$ 0.0 & 0.0 $\pm$ 0.0 & 0.6 $\pm$ 0.2 & 1.1 $\pm$ 0.4 & 0.0 $\pm$ 0.0 & 0.3 $\pm$ 0.1 \\
      \rowcolor{blue!10} VI-CuRL & \textbf{35.0 $\pm$ 0.2} & \textbf{26.1 $\pm$ 2.0} & \textbf{81.2 $\pm$ 0.6} & \textbf{73.9 $\pm$ 1.3} & \textbf{22.1 $\pm$ 4.6} & \textbf{37.1 $\pm$ 1.0} & \textbf{45.9 $\pm$ 1.6} \\
      \midrule
      \multicolumn{8}{l}{\textit{Entropy (w/o. Verifier)}} \\
      RENT & 0.0 $\pm$ 0.0 & 0.0 $\pm$ 0.0 & 12.9 $\pm$ 0.7 & 11.7 $\pm$ 1.3 & 8.6 $\pm$ 0.9 & 5.7 $\pm$ 0.8 & 6.5 $\pm$ 0.6 \\
      EMPO & 12.5 $\pm$ 1.5 & 8.2 $\pm$ 1.1 & 25.4 $\pm$ 2.5 & 28.6 $\pm$ 2.1 & 10.2 $\pm$ 1.8 & 14.5 $\pm$ 1.5 & 16.6 $\pm$ 1.8 \\
      \rowcolor{blue!10} VI-CuRL & \textbf{33.1 $\pm$ 0.1} & \textbf{21.4 $\pm$ 1.7} & \textbf{81.7 $\pm$ 0.5} & \textbf{71.7 $\pm$ 0.4} & \textbf{26.2 $\pm$ 1.0} & \textbf{35.5 $\pm$ 0.8} & \textbf{44.9 $\pm$ 0.8} \\
      \midrule
      \multicolumn{8}{c}{\cellcolor{gray!5}\textbf{DeepSeek-R1-Distill-Qwen-1.5B}} \\
      \midrule
      Base (No RL) & 44.8 $\pm$ 0.8 & 32.8 $\pm$ 0.7 & 88.2 $\pm$ 0.1 & 85.9 $\pm$ 0.6 & 27.8 $\pm$ 0.3 & 48.6 $\pm$ 0.3 & 54.7 $\pm$ 0.5 \\
      \midrule
      \multicolumn{8}{l}{\textit{Majority Vote (w/o. Verifier)}} \\
      TTRL & 38.7 $\pm$ 0.6 & 30.7 $\pm$ 0.9 & 90.8 $\pm$ 0.8 & 91.9 $\pm$ 0.2 & 36.9 $\pm$ 0.9 & 52.5 $\pm$ 1.7 & 56.9 $\pm$ 0.9 \\
      \rowcolor{blue!10} VI-CuRL & \textbf{62.8 $\pm$ 0.8} & \textbf{50.5 $\pm$ 0.8} & \textbf{99.3 $\pm$ 0.5} & \textbf{98.4 $\pm$ 1.1} & \textbf{47.6 $\pm$ 0.7} & \textbf{71.9 $\pm$ 1.5} & \textbf{71.8 $\pm$ 0.9} \\
      \midrule
      \multicolumn{8}{l}{\textit{Entropy (w/o. Verifier)}} \\
      RENT & 38.0 $\pm$ 1.0 & 25.1 $\pm$ 0.2 & 90.5 $\pm$ 0.8 & 95.0 $\pm$ 0.8 & 43.6 $\pm$ 1.8 & 59.9 $\pm$ 1.1 & 58.7 $\pm$ 1.0 \\
      EMPO & 48.6 $\pm$ 2.5 & 35.1 $\pm$ 2.2 & 91.7 $\pm$ 1.5 & 96.4 $\pm$ 0.4 & 44.0 $\pm$ 1.6 & 57.4 $\pm$ 1.1 & 62.2 $\pm$ 1.6 \\
      \rowcolor{blue!10} VI-CuRL & \textbf{62.1 $\pm$ 1.1} & \textbf{47.3 $\pm$ 1.0} & \textbf{97.6 $\pm$ 0.2} & \textbf{98.8 $\pm$ 0.1} & \textbf{45.8 $\pm$ 0.3} & \textbf{70.4 $\pm$ 1.0} & \textbf{70.3 $\pm$ 0.6} \\
      \midrule
      \multicolumn{8}{c}{\cellcolor{gray!5}\textbf{Llama3.2-3B-Instruct}} \\
      \midrule
      Base (No RL) & 18.8 $\pm$ 1.1 & 3.3 $\pm$ 0.2 & 54.7 $\pm$ 1.4 & 40.8 $\pm$ 0.8 & 5.2 $\pm$ 1.1 & 16.4 $\pm$ 0.5 & 23.2 $\pm$ 0.9 \\
      \midrule
      \multicolumn{8}{l}{\textit{Majority Vote (w/o. Verifier)}} \\
      TTRL & \textbf{23.3 $\pm$ 0.3} & 3.3 $\pm$ 0.7 & 53.8 $\pm$ 0.5 & 50.6 $\pm$ 0.8 & 10.7 $\pm$ 1.0 & \textbf{24.0 $\pm$ 1.3} & 27.6 $\pm$ 0.8 \\
      \rowcolor{blue!10} VI-CuRL & 20.4 $\pm$ 0.8 & \textbf{4.6 $\pm$ 0.6} & \textbf{58.1 $\pm$ 0.9} & \textbf{54.2 $\pm$ 0.7} & \textbf{14.0 $\pm$ 0.8} & 23.3 $\pm$ 1.4 & \textbf{29.1 $\pm$ 0.9} \\
      \midrule
      \multicolumn{8}{l}{\textit{Entropy (w/o. Verifier)}} \\
      RENT & 2.4 $\pm$ 0.9 & 0.0 $\pm$ 0.0 & 15.9 $\pm$ 1.2 & 5.4 $\pm$ 1.6 & 4.2 $\pm$ 0.5 & 1.5 $\pm$ 0.6 & 4.9 $\pm$ 0.8 \\
      EMPO & 8.5 $\pm$ 1.2 & 1.1 $\pm$ 0.3 & 25.4 $\pm$ 2.0 & 18.5 $\pm$ 1.5 & 6.5 $\pm$ 1.1 & 8.4 $\pm$ 1.0 & 11.4 $\pm$ 1.2 \\
      \rowcolor{blue!10} VI-CuRL & \textbf{20.8 $\pm$ 0.9} & \textbf{5.2 $\pm$ 1.0} & \textbf{56.9 $\pm$ 1.3} & \textbf{50.2 $\pm$ 0.1} & \textbf{12.9 $\pm$ 0.7} & \textbf{21.8 $\pm$ 0.7} & \textbf{28.0 $\pm$ 0.8} \\
      \midrule
      \multicolumn{8}{c}{\cellcolor{gray!5}\textbf{Qwen2.5-Math-7B}} \\
      \midrule
      Base (No RL) & 36.2 $\pm$ 1.2 & 23.3 $\pm$ 1.0 & 81.9 $\pm$ 1.3 & 66.6 $\pm$ 0.5 & 11.8 $\pm$ 0.6 & 33.9 $\pm$ 0.7 & 42.3 $\pm$ 0.9 \\
      \midrule
      \multicolumn{8}{l}{\textit{Majority Vote (w/o. Verifier)}} \\
      TTRL & 53.8 $\pm$ 0.7 & 28.7 $\pm$ 0.9 & 82.8 $\pm$ 1.3 & 81.4 $\pm$ 0.1 & 23.2 $\pm$ 0.6 & 41.9 $\pm$ 0.7 & 52.0 $\pm$ 0.7 \\
      \rowcolor{blue!10} VI-CuRL & \textbf{55.5 $\pm$ 2.1} & \textbf{31.7 $\pm$ 0.1} & \textbf{88.2 $\pm$ 1.0} & \textbf{81.6 $\pm$ 0.8} & \textbf{27.9 $\pm$ 0.6} & \textbf{49.1 $\pm$ 3.7} & \textbf{55.7 $\pm$ 1.4} \\
      \midrule
      \multicolumn{8}{l}{\textit{Entropy (w/o. Verifier)}} \\
      RENT & 28.8 $\pm$ 1.4 & 17.9 $\pm$ 2.5 & 71.5 $\pm$ 3.2 & 74.6 $\pm$ 2.3 & \textbf{25.4 $\pm$ 5.0} & 35.2 $\pm$ 4.2 & 42.2 $\pm$ 3.1 \\
      EMPO & 35.2 $\pm$ 2.5 & 22.4 $\pm$ 1.8 & 76.5 $\pm$ 2.2 & 78.2 $\pm$ 1.5 & 22.5 $\pm$ 3.1 & 39.5 $\pm$ 2.0 & 45.7 $\pm$ 2.2 \\
      \rowcolor{blue!10} VI-CuRL & \textbf{48.4 $\pm$ 3.9} & \textbf{30.0 $\pm$ 1.4} & \textbf{85.0 $\pm$ 0.3} & \textbf{80.6 $\pm$ 0.3} & 24.6 $\pm$ 0.7 & \textbf{43.4 $\pm$ 0.8} & \textbf{52.0 $\pm$ 1.2} \\
      \bottomrule
    \end{tabular}
  }
\end{table*}

\section{Algorithm Details and Additional Experimental Results}
\label{app:additional_results}

\subsection{Learning Curve Analysis: with/without Verifiers}
\label{app:learning_curves}
We performed a comprehensive comparison of learning curves across three distinct reward settings: \textbf{Oracle} (Ground Truth), \textbf{Majority Vote}, and \textbf{Entropy}. Figure~\ref{fig:learning_curves_grid} presented a $6 \times 6$ grid of results.

\textbf{Oracle Setting (Rows 1-2):} Even with a perfect verifier, VI-CuRL (blue) often matched or exceeded standard RL (orange) in stability and final performance.

\textbf{Majority Vote (Rows 3-4):} This verifier-independent baseline (purple) was notoriously unstable, as seen in the fluctuating curves. VI-CuRL significantly stabilized training, yielding smooth, and consistent improvements.

\textbf{Entropy (Rows 5-6):} Using entropy as a proxy also benefited from VI-CuRL's curriculum, which filtered out high-entropy (uncertain) samples early on to prevent policy collapse.

VI-CuRL demonstrated effective variance reduction and robust learning across all settings, validating its utility as a general-purpose stabilizer for RLVR.

\begin{figure*}[h!]
    \centering
    \setlength{\tabcolsep}{1pt}
    \renewcommand{\arraystretch}{0.5}
    \begin{tabular}{ccccccc}
         & \scriptsize AIME2024 & \scriptsize MATH500 & \scriptsize AIME2025 & \scriptsize AMC2023 & \scriptsize Minerva & \scriptsize Olympiad \\
        \rotatebox{90}{\scriptsize Oracle P@1} &
        \includegraphics[width=0.155\textwidth]{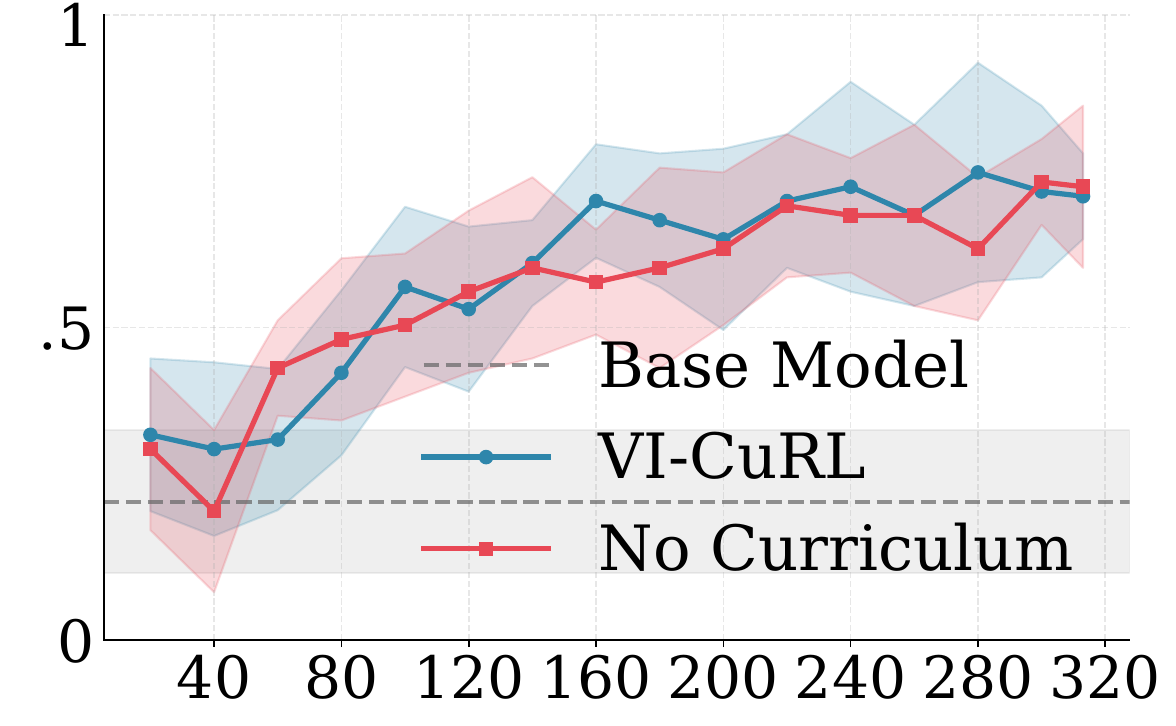} &
        \includegraphics[width=0.155\textwidth]{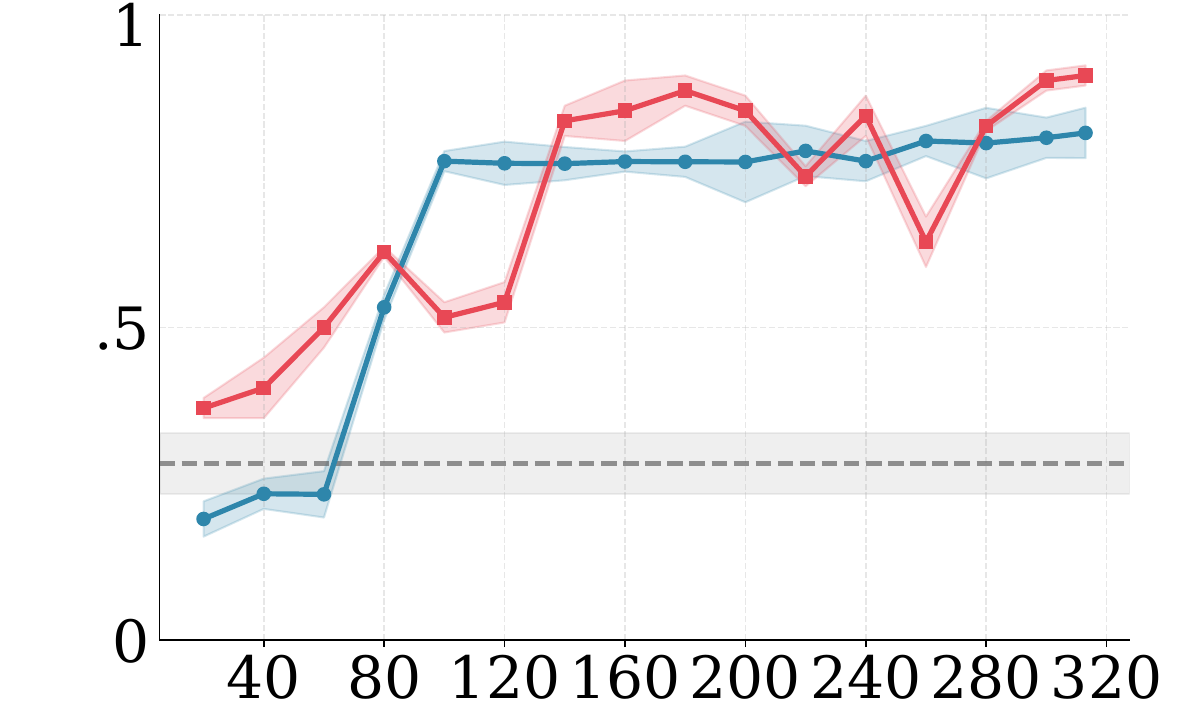} &
        \includegraphics[width=0.155\textwidth]{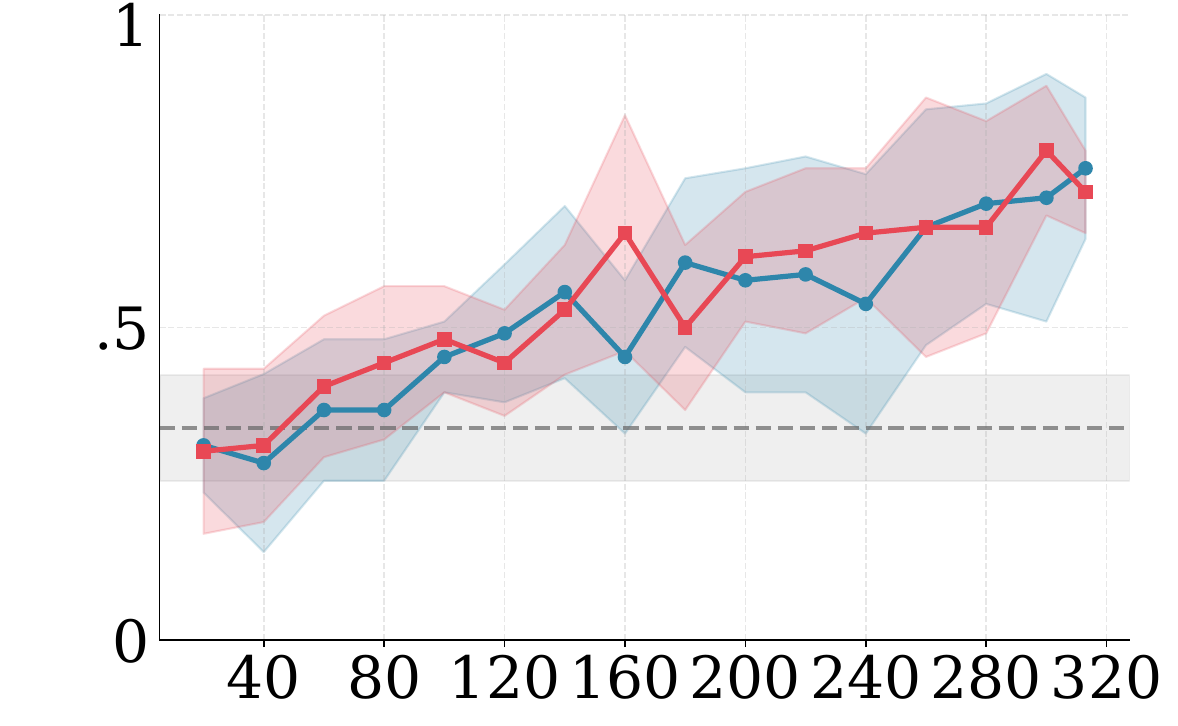} &
        \includegraphics[width=0.155\textwidth]{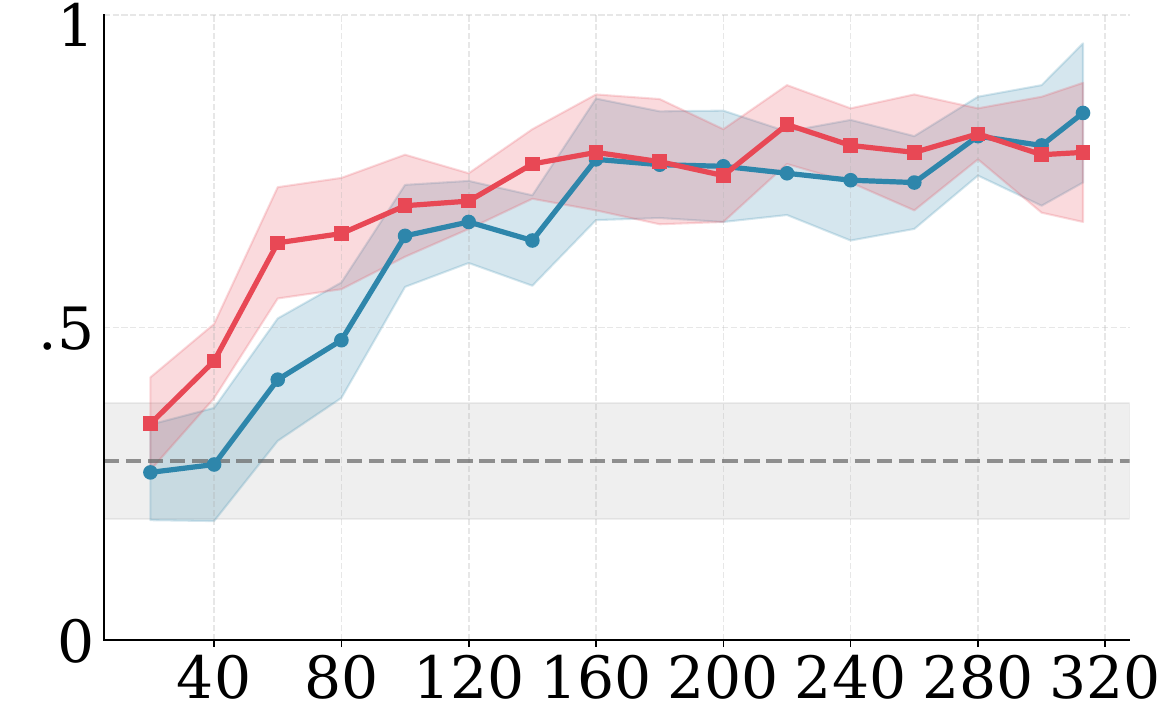} &
        \includegraphics[width=0.155\textwidth]{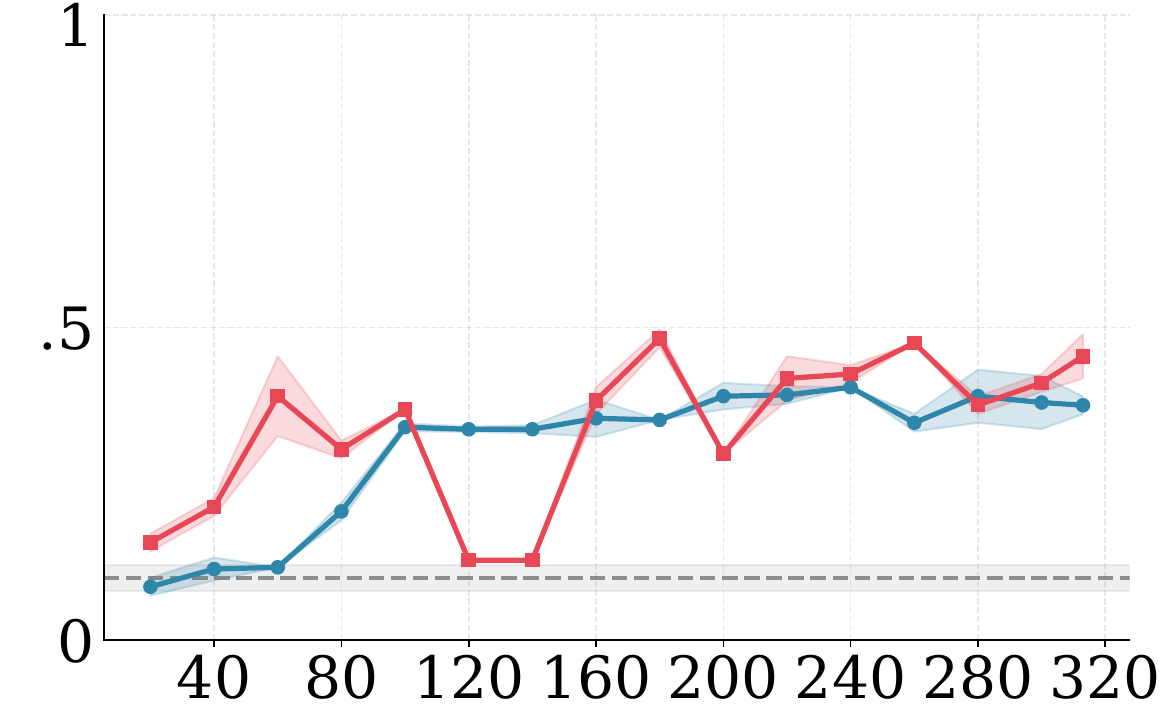} &
        \includegraphics[width=0.155\textwidth]{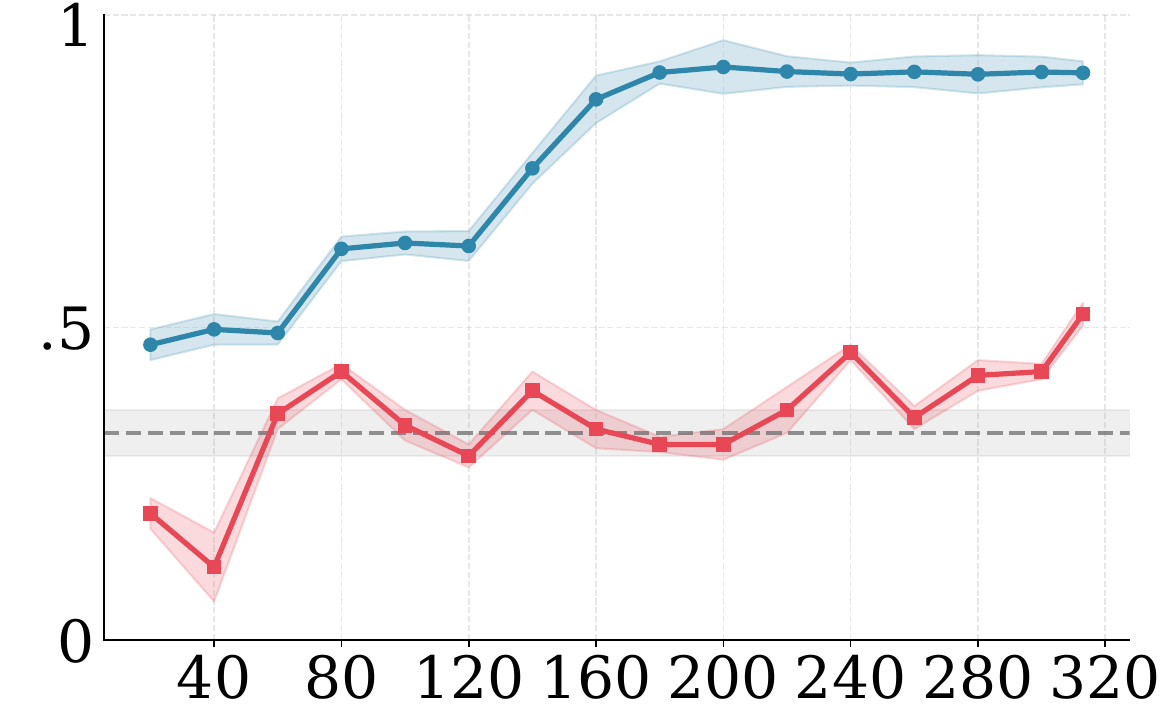} \\
        \rotatebox{90}{\scriptsize Oracle P@8} &
        \includegraphics[width=0.155\textwidth]{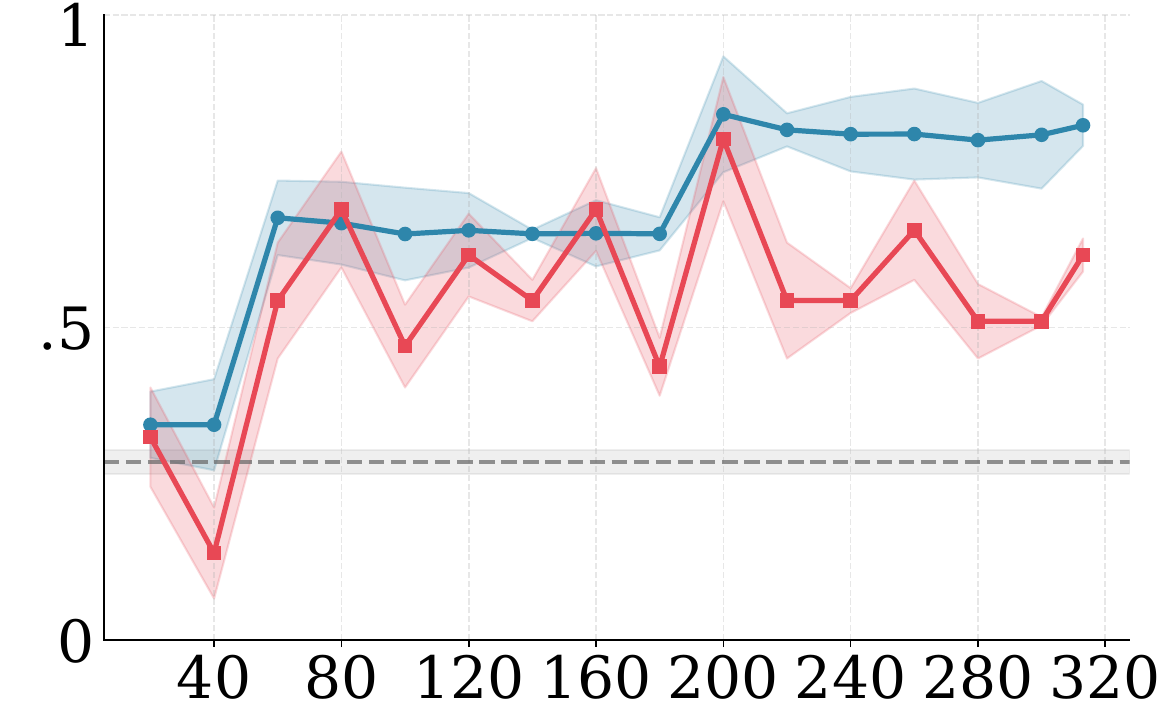} &
        \includegraphics[width=0.155\textwidth]{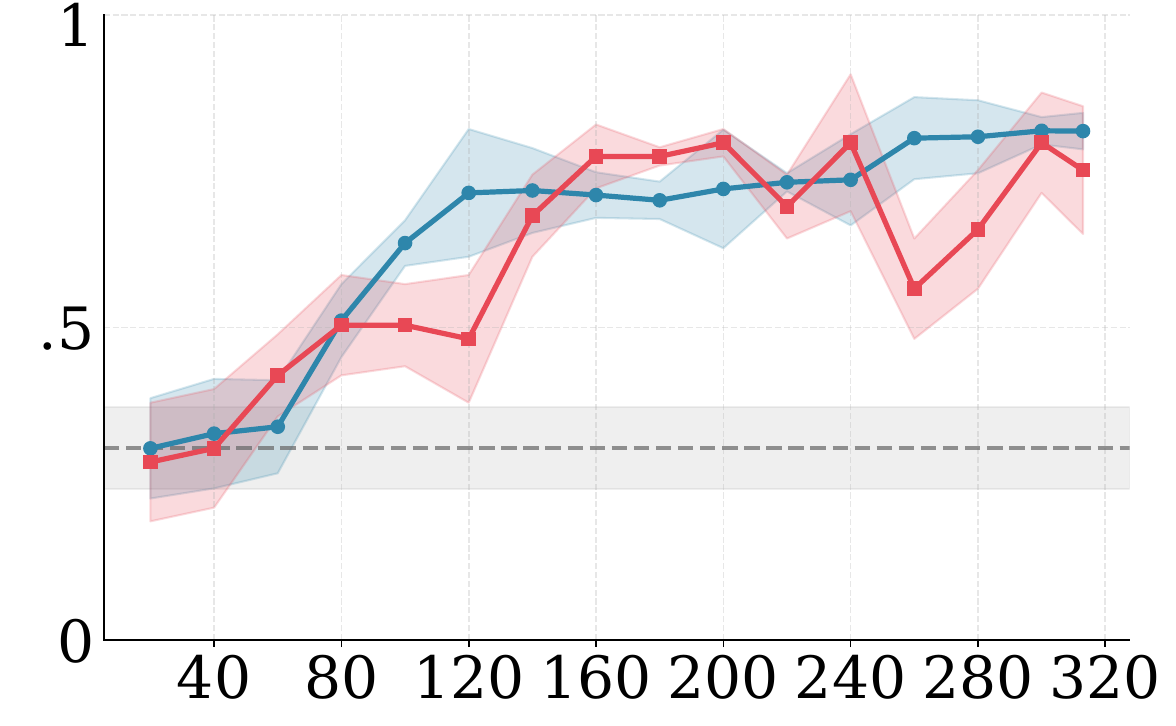} &
        \includegraphics[width=0.155\textwidth]{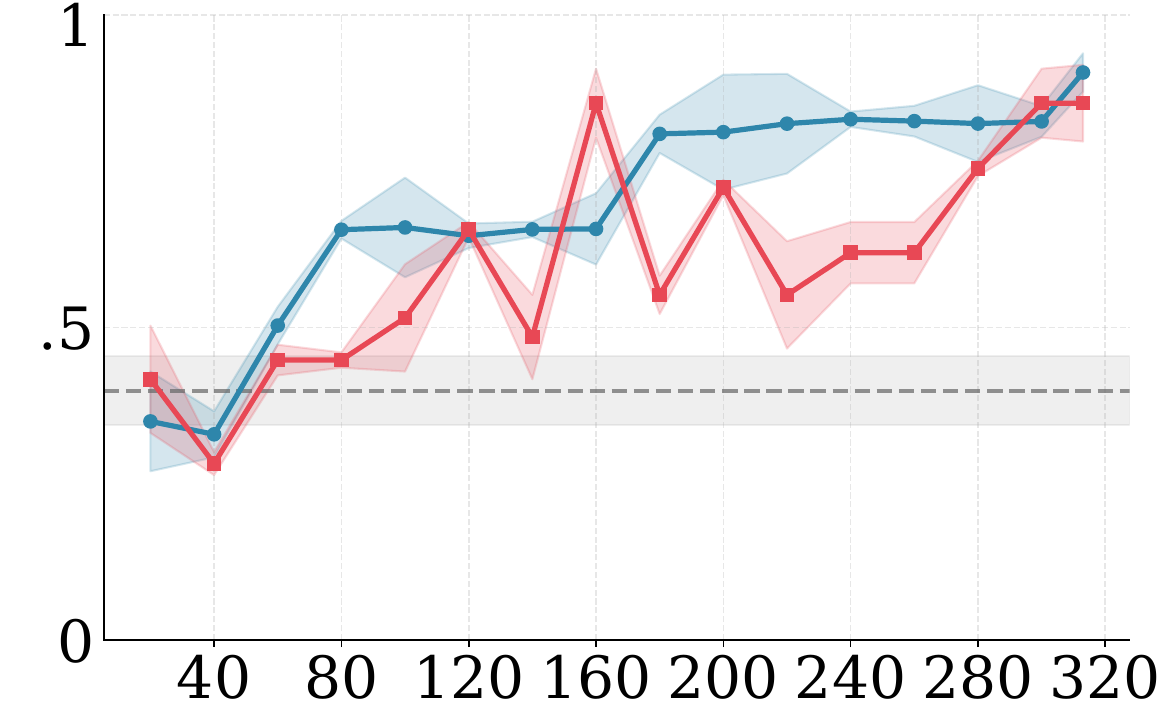} &
        \includegraphics[width=0.155\textwidth]{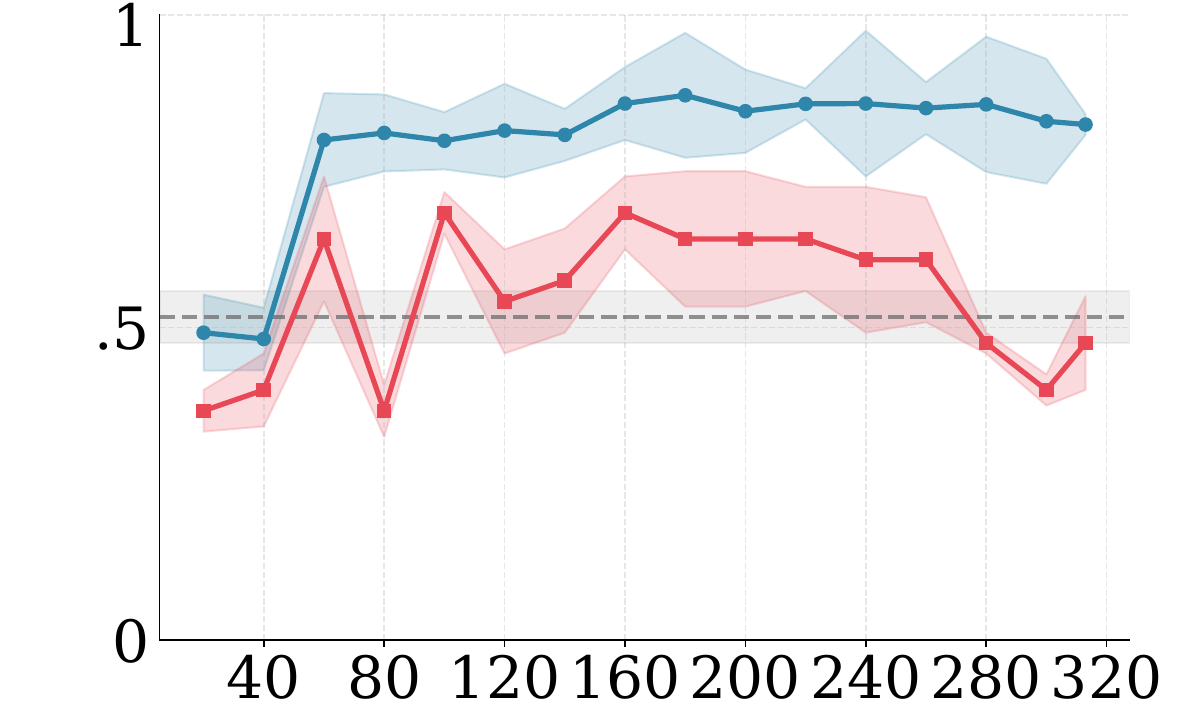} &
        \includegraphics[width=0.155\textwidth]{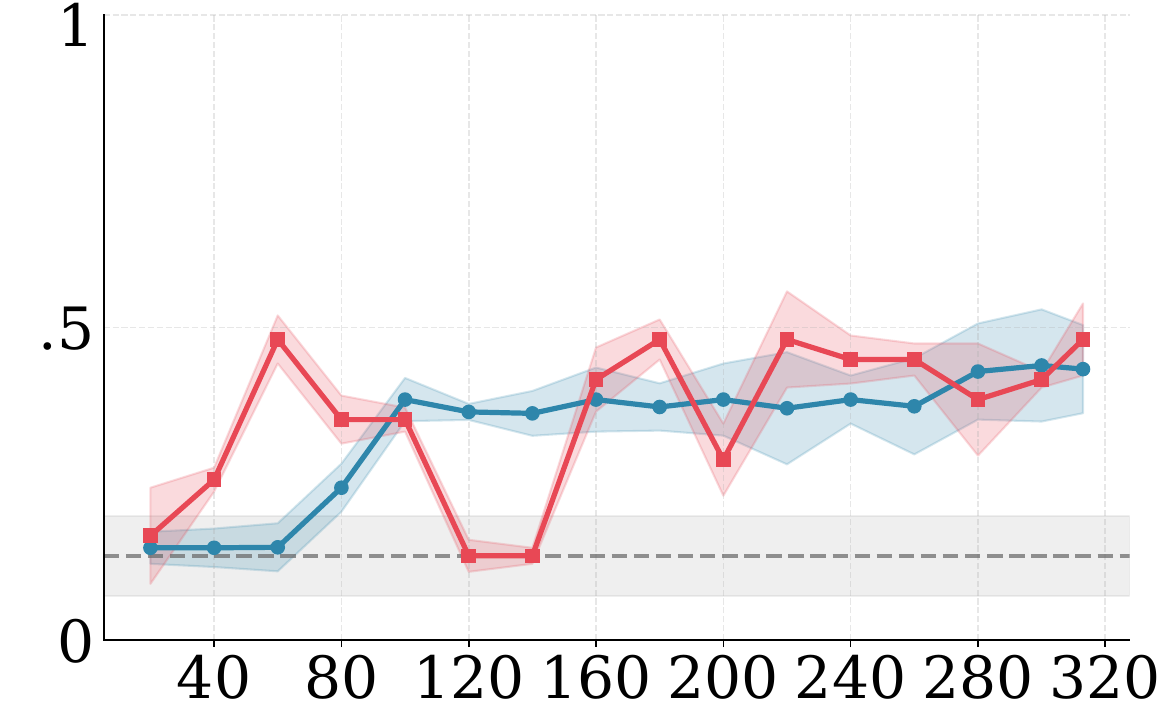} &
        \includegraphics[width=0.155\textwidth]{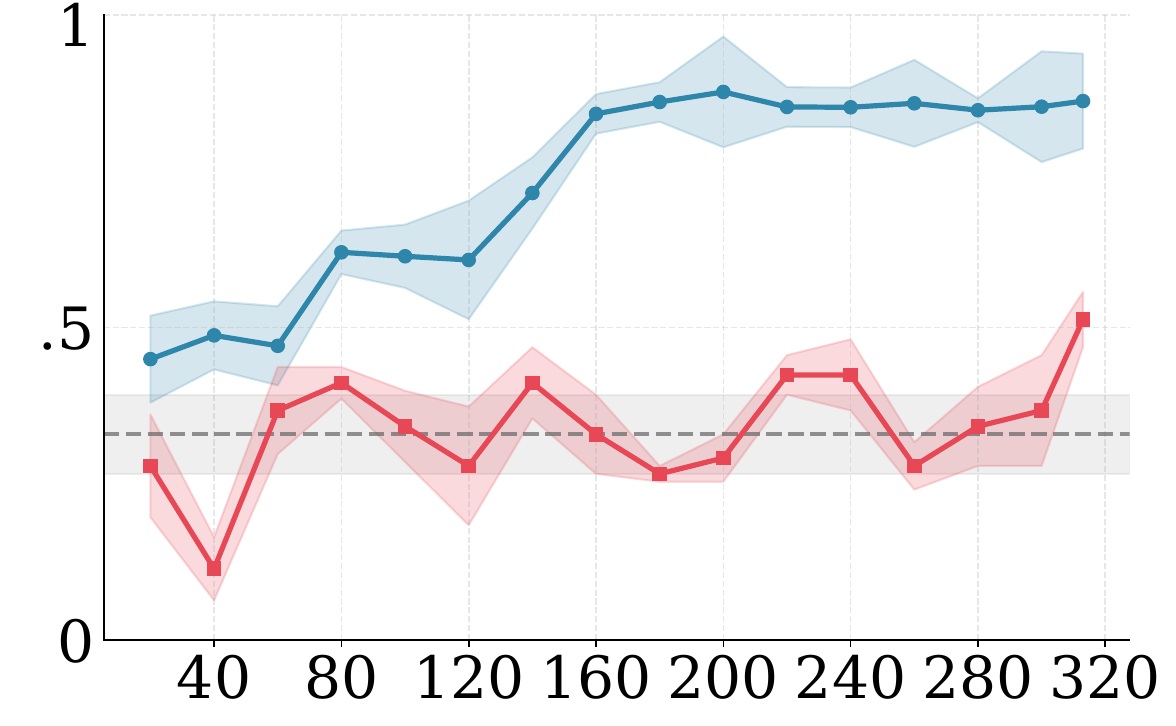} \\
        \rotatebox{90}{\scriptsize MajVote P@1} &
        \includegraphics[width=0.155\textwidth]{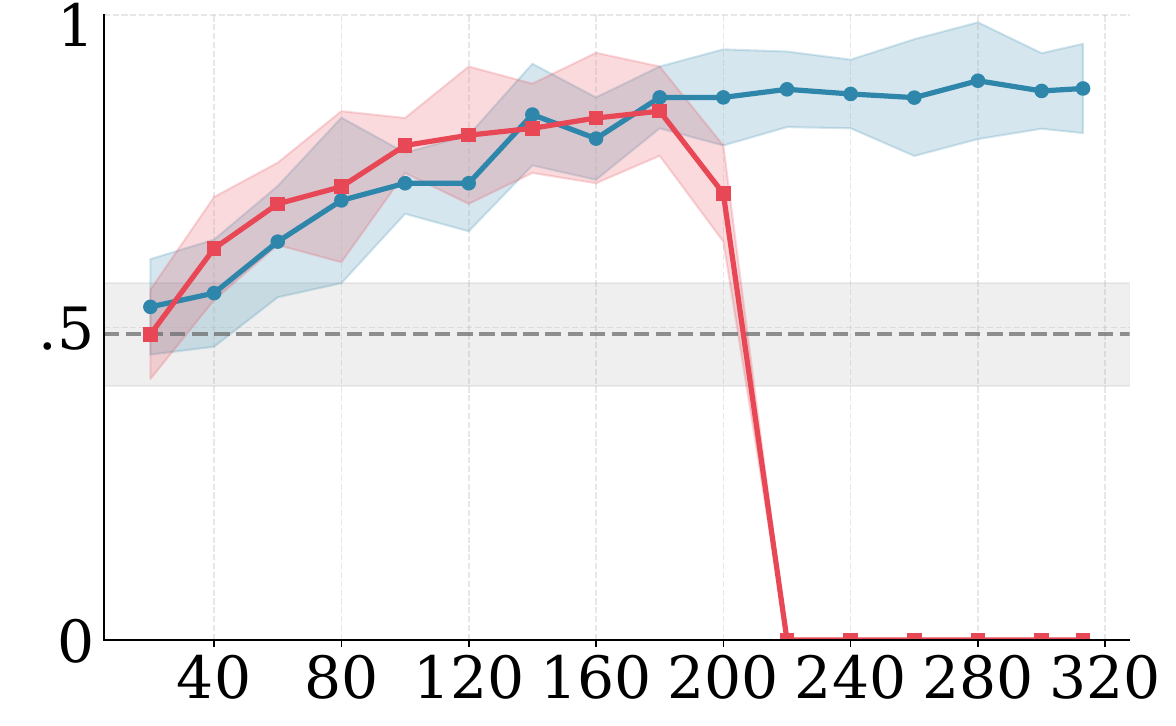} &
        \includegraphics[width=0.155\textwidth]{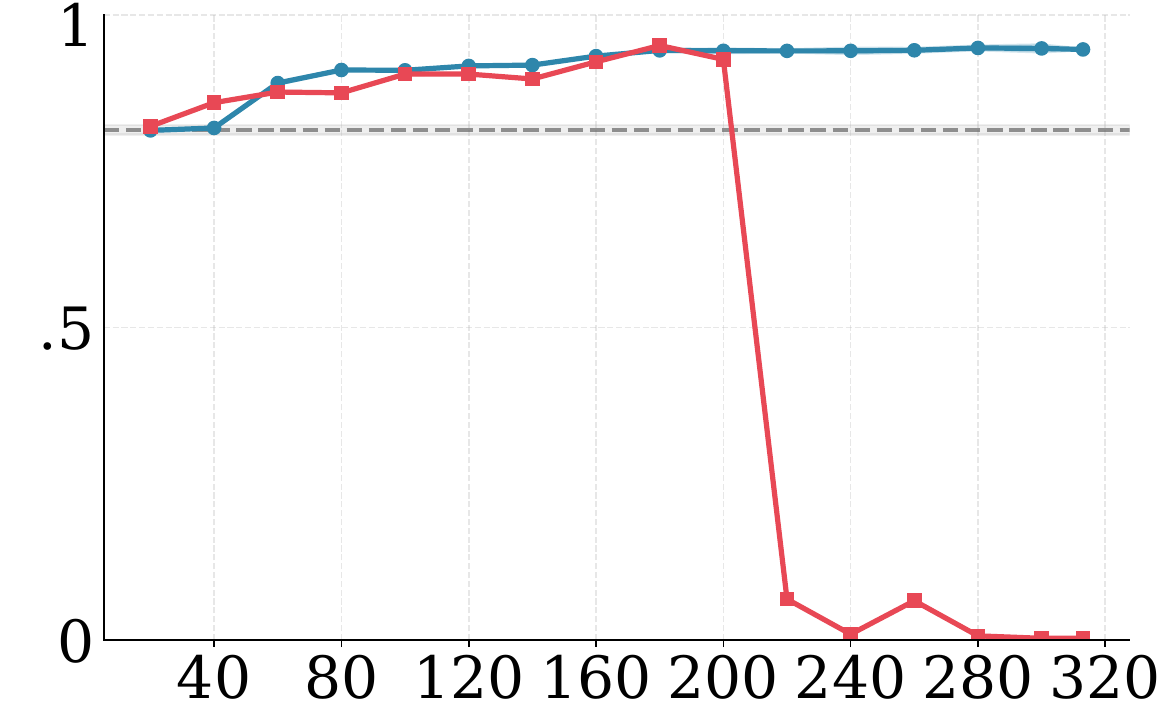} &
        \includegraphics[width=0.155\textwidth]{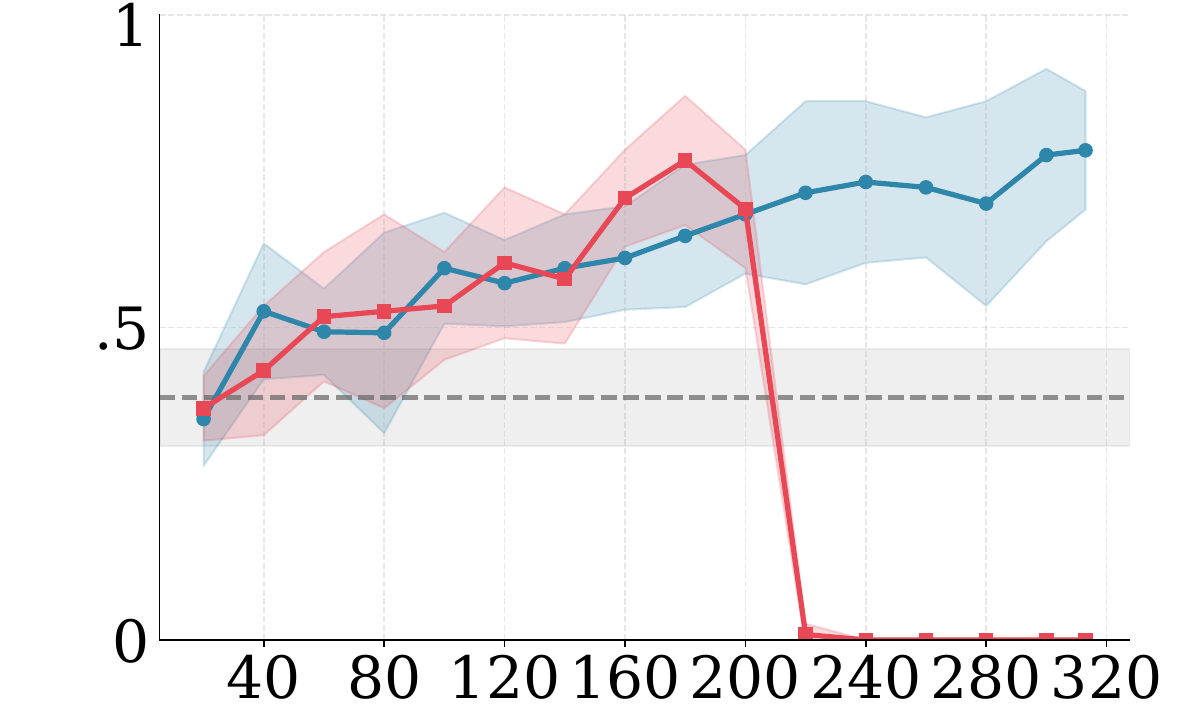} &
        \includegraphics[width=0.155\textwidth]{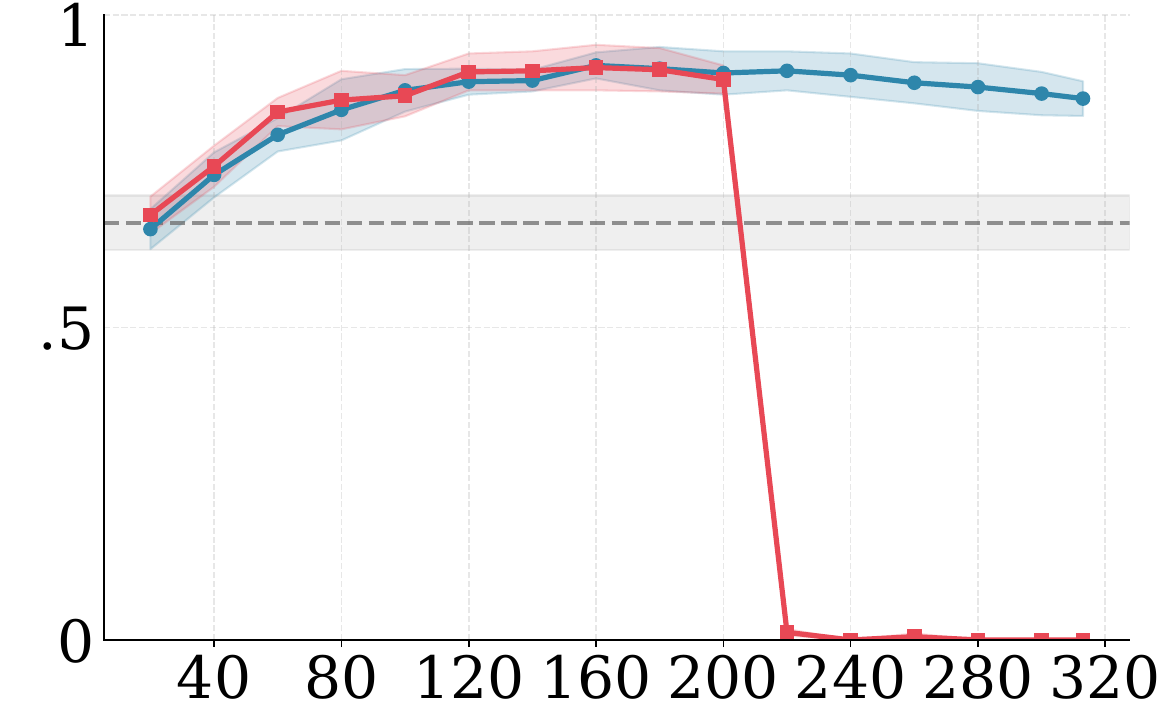} &
        \includegraphics[width=0.155\textwidth]{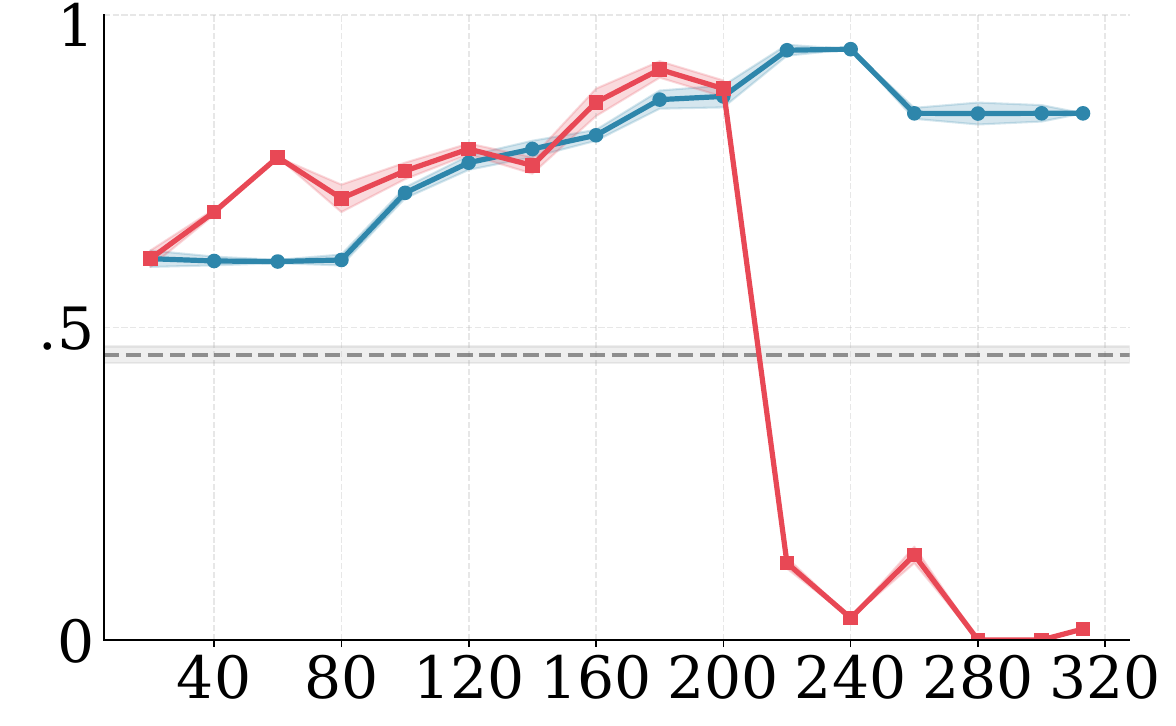} &
        \includegraphics[width=0.155\textwidth]{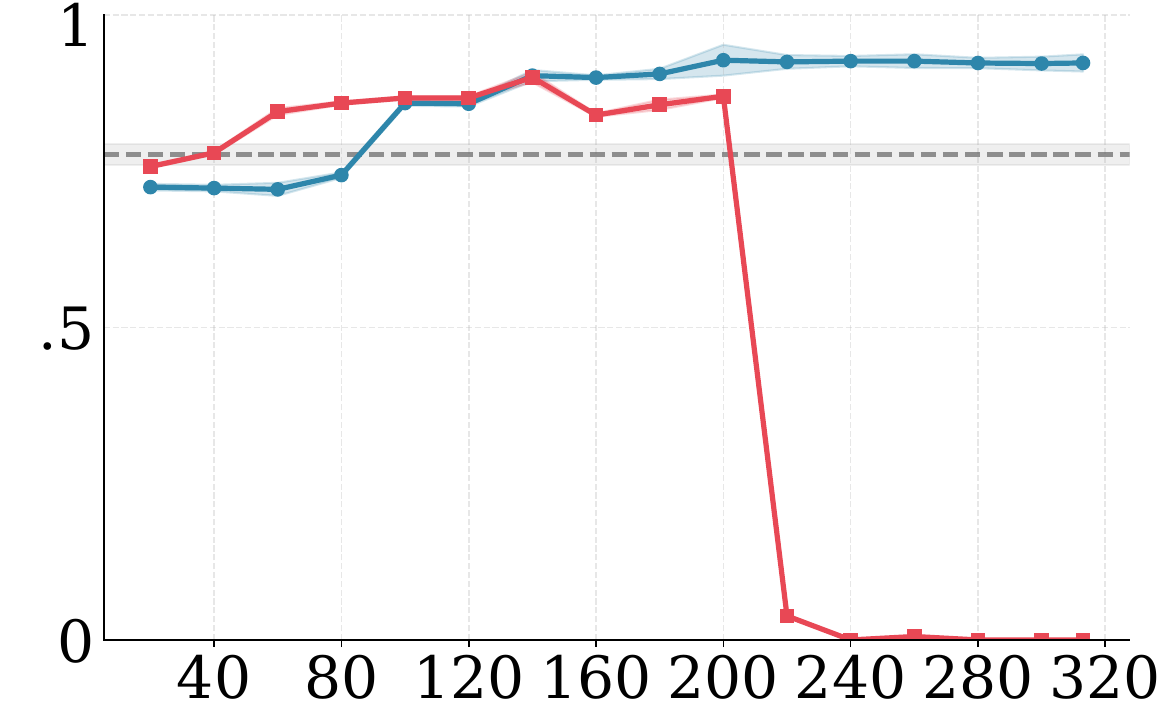} \\
        \rotatebox{90}{\scriptsize MajVote P@8} &
        \includegraphics[width=0.155\textwidth]{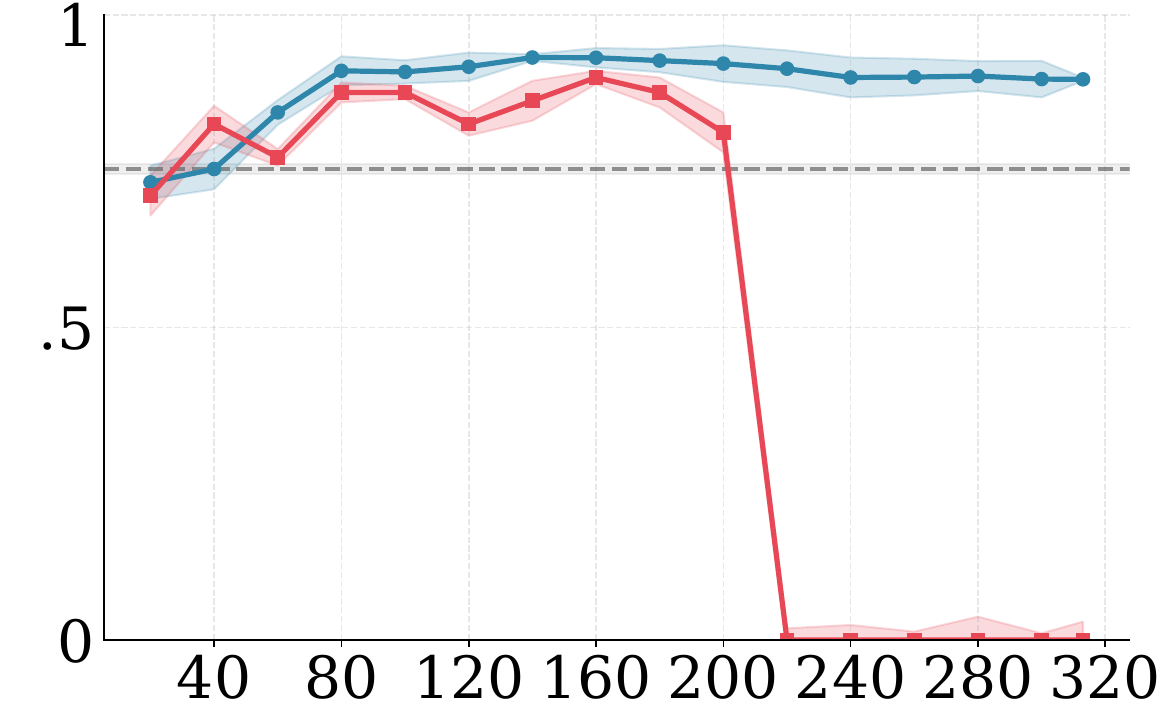} &
        \includegraphics[width=0.155\textwidth]{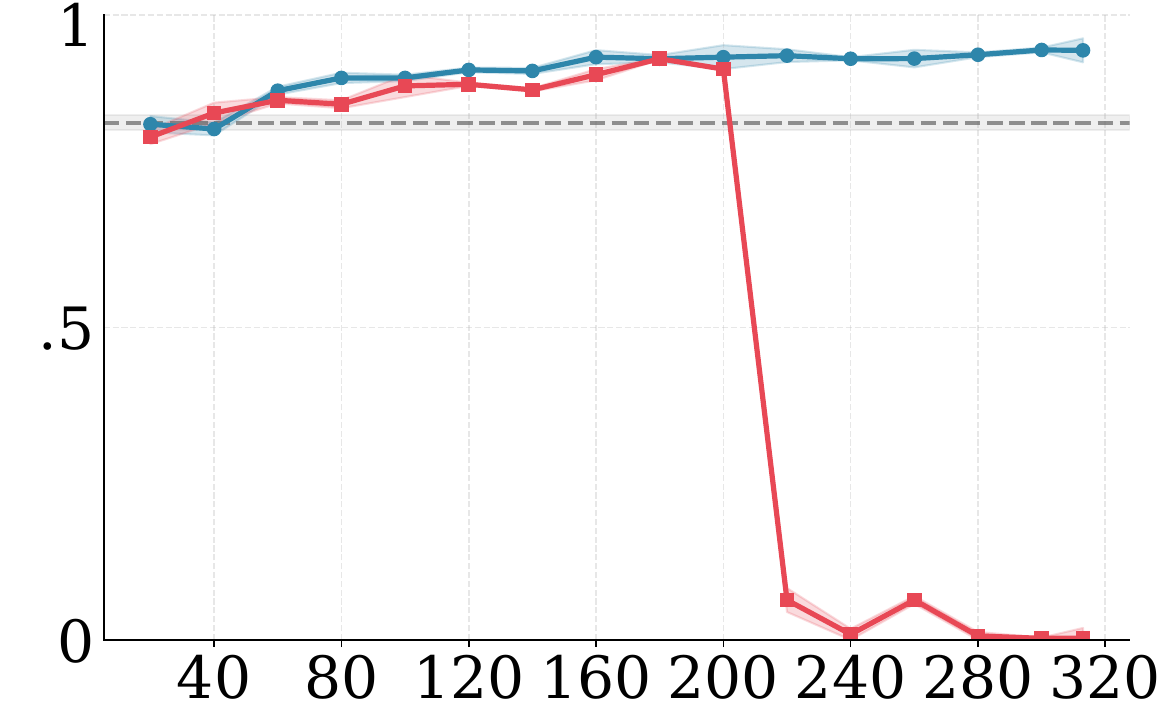} &
        \includegraphics[width=0.155\textwidth]{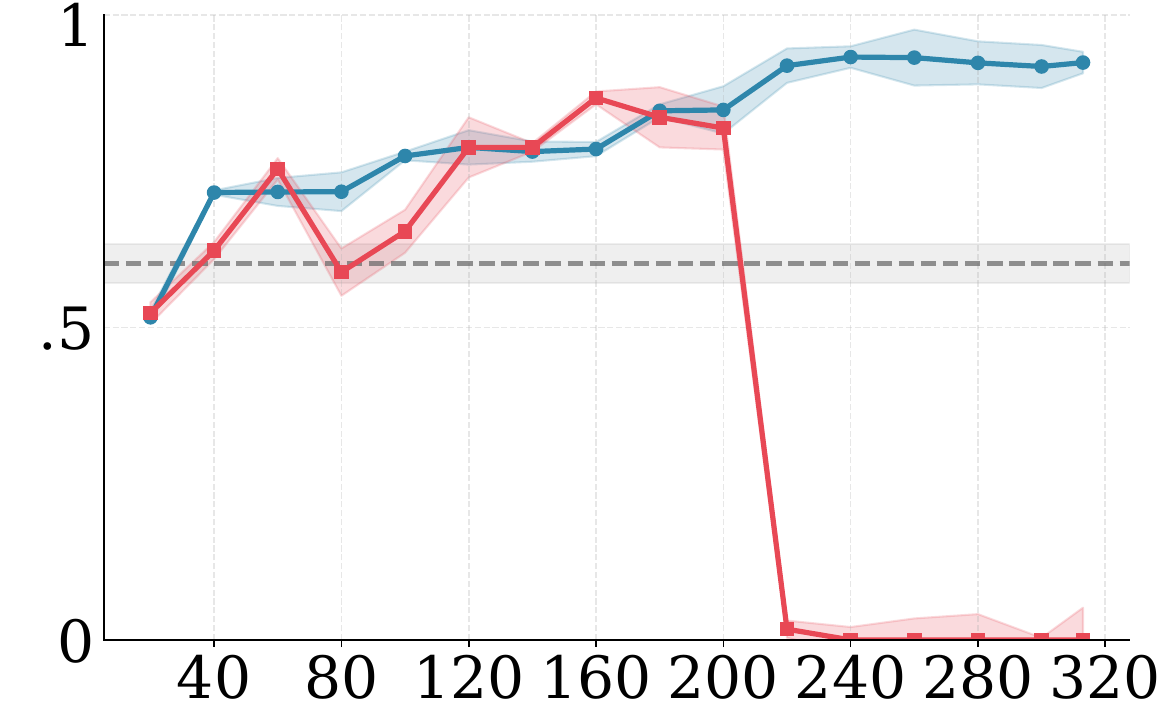} &
        \includegraphics[width=0.155\textwidth]{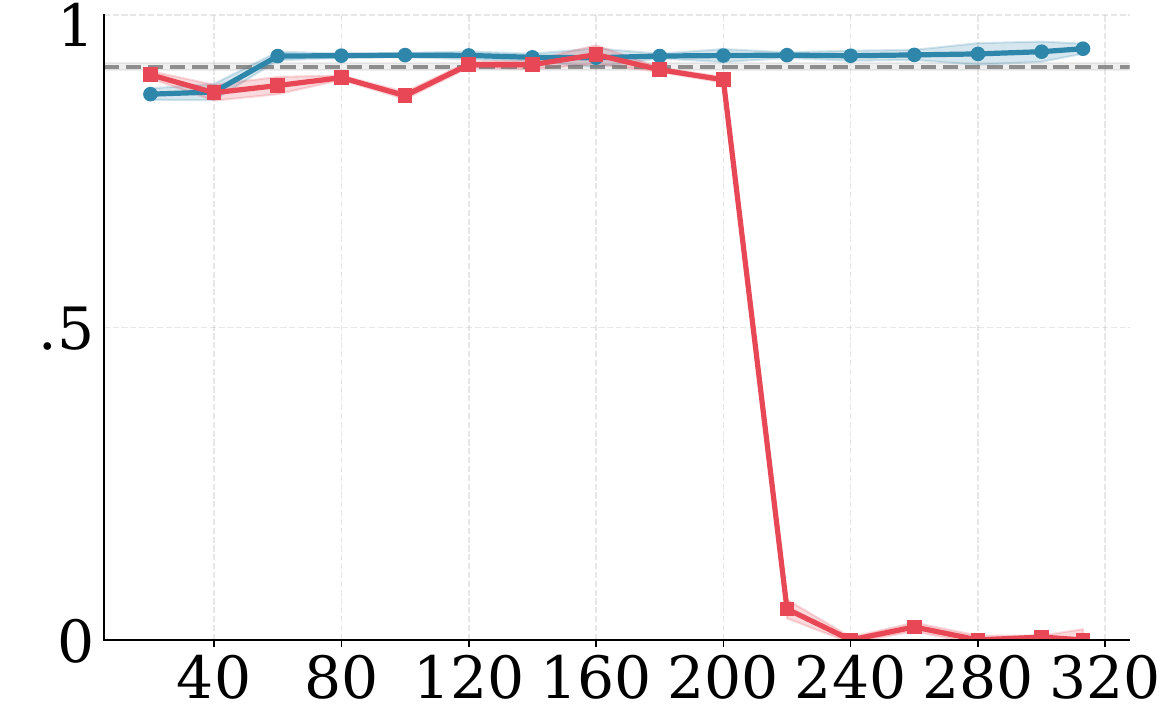} &
        \includegraphics[width=0.155\textwidth]{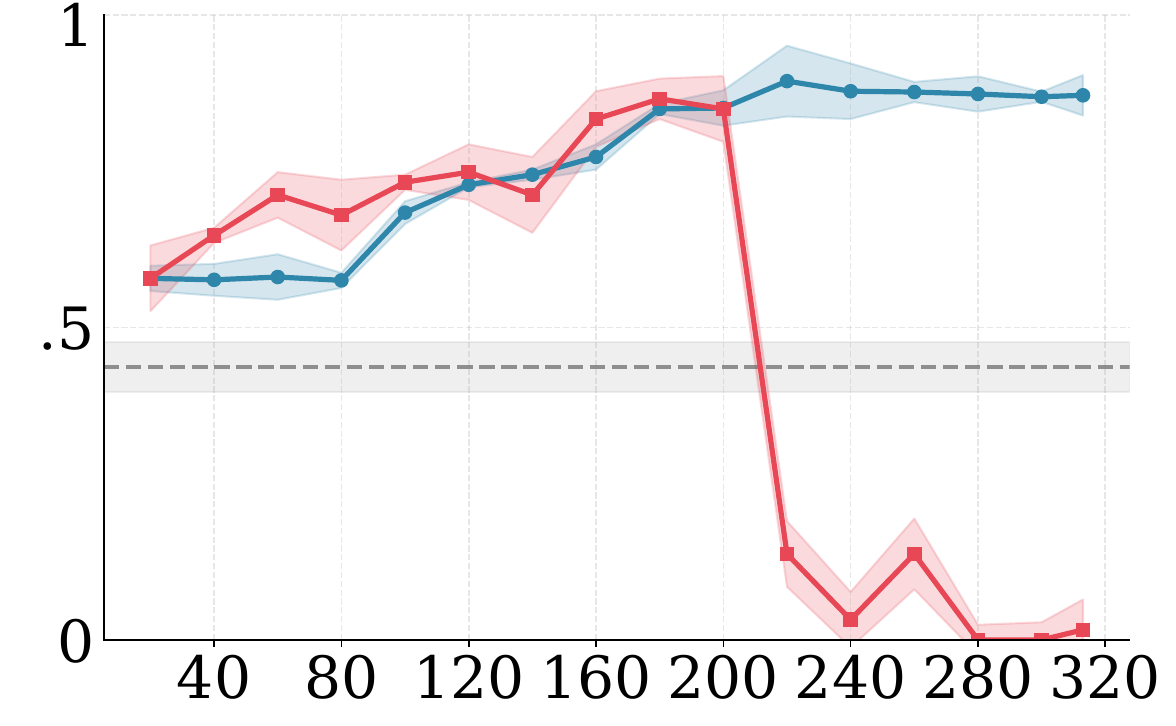} &
        \includegraphics[width=0.155\textwidth]{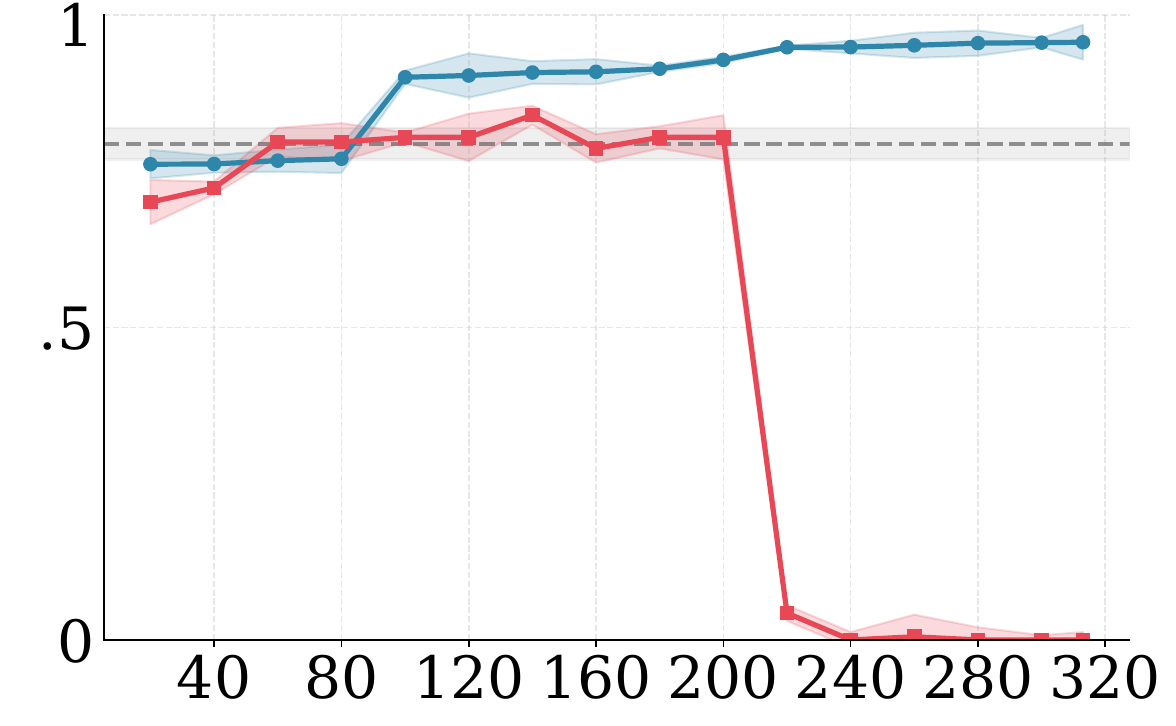} \\
        \rotatebox{90}{\scriptsize Entropy P@1} &
        \includegraphics[width=0.155\textwidth]{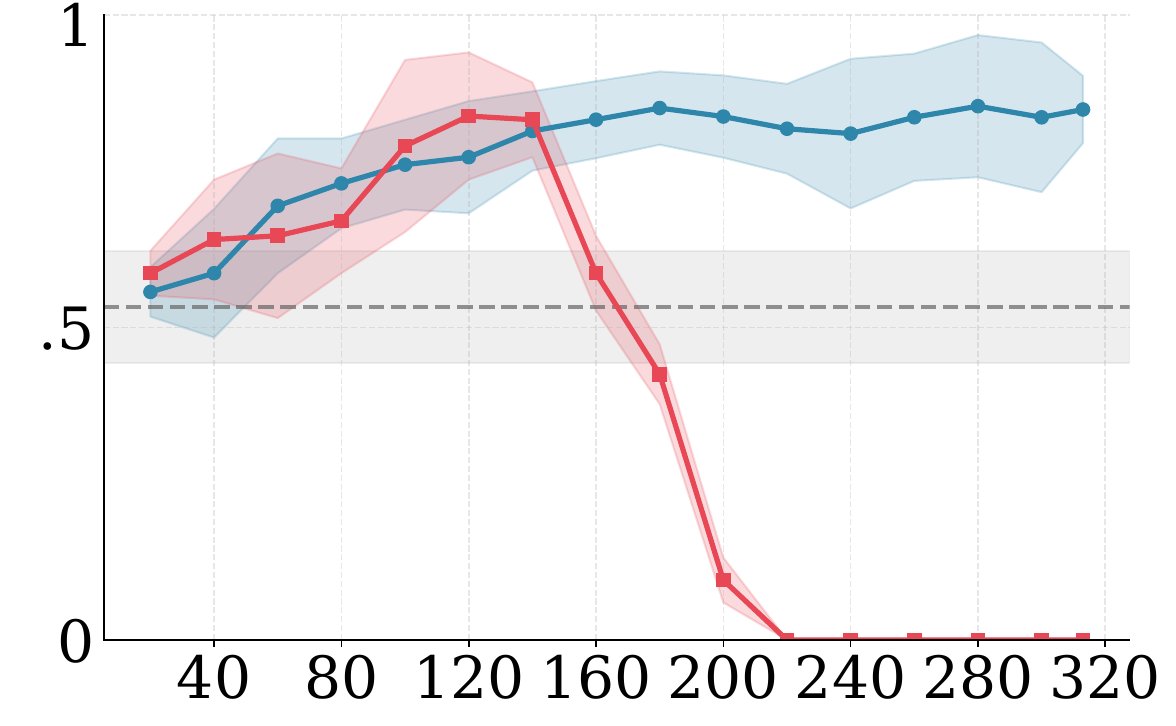} &
        \includegraphics[width=0.155\textwidth]{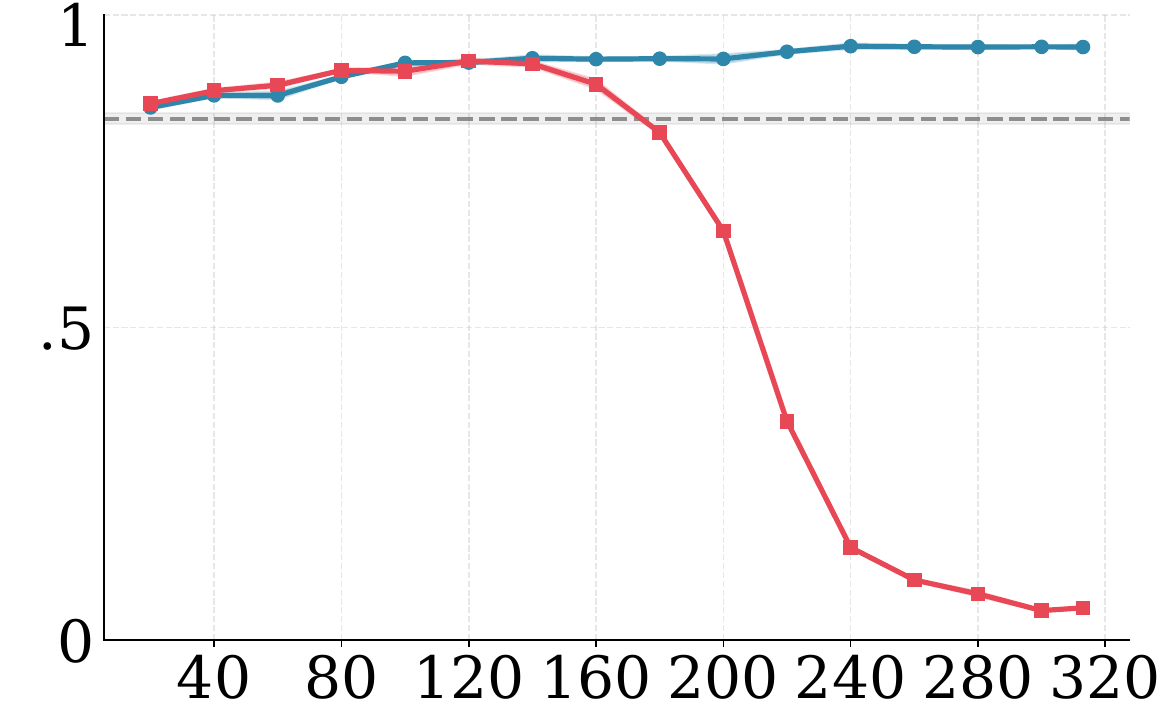} &
        \includegraphics[width=0.155\textwidth]{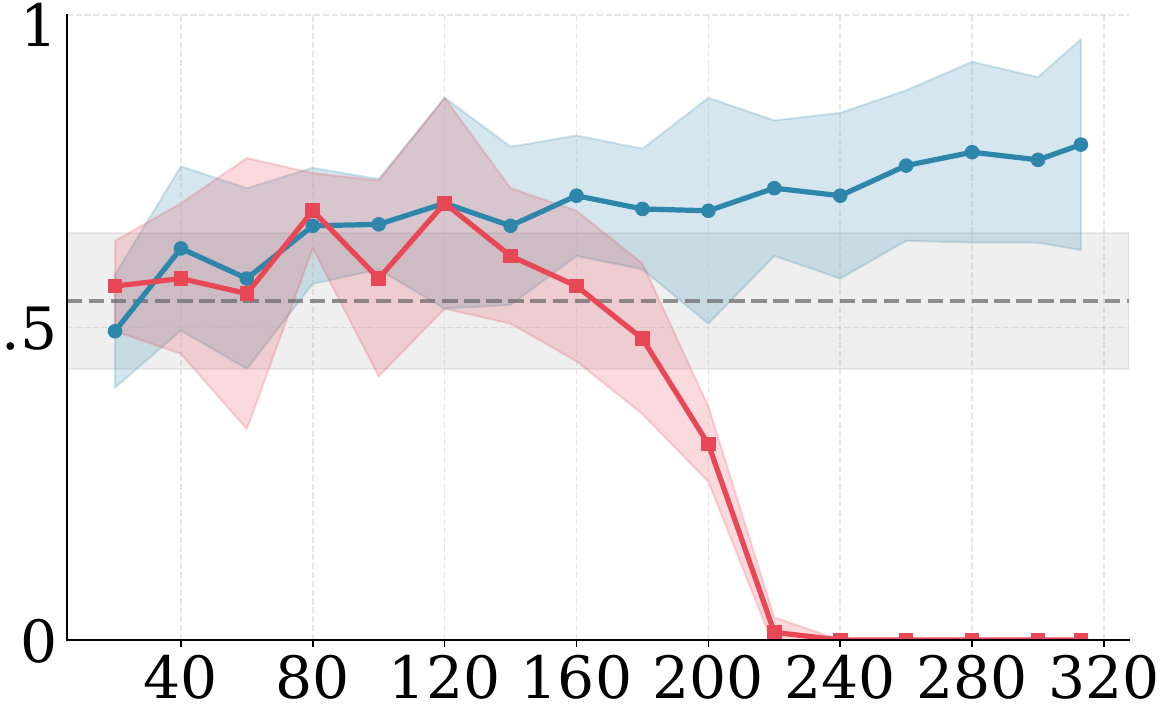} &
        \includegraphics[width=0.155\textwidth]{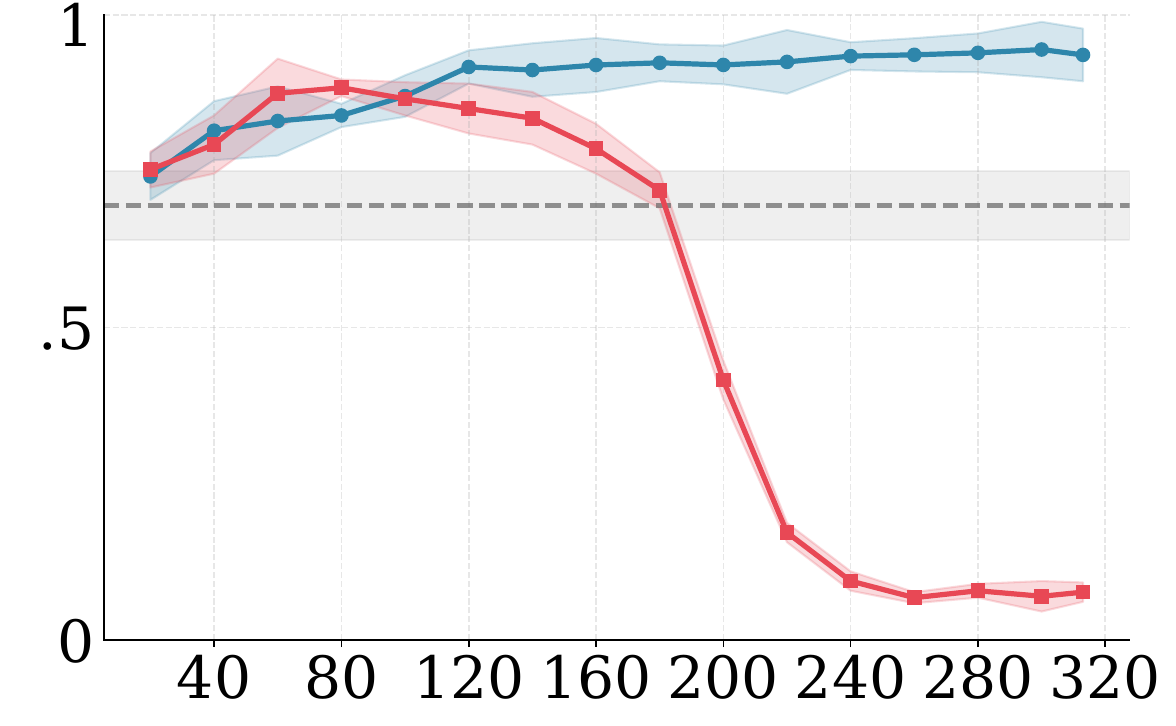} &
        \includegraphics[width=0.155\textwidth]{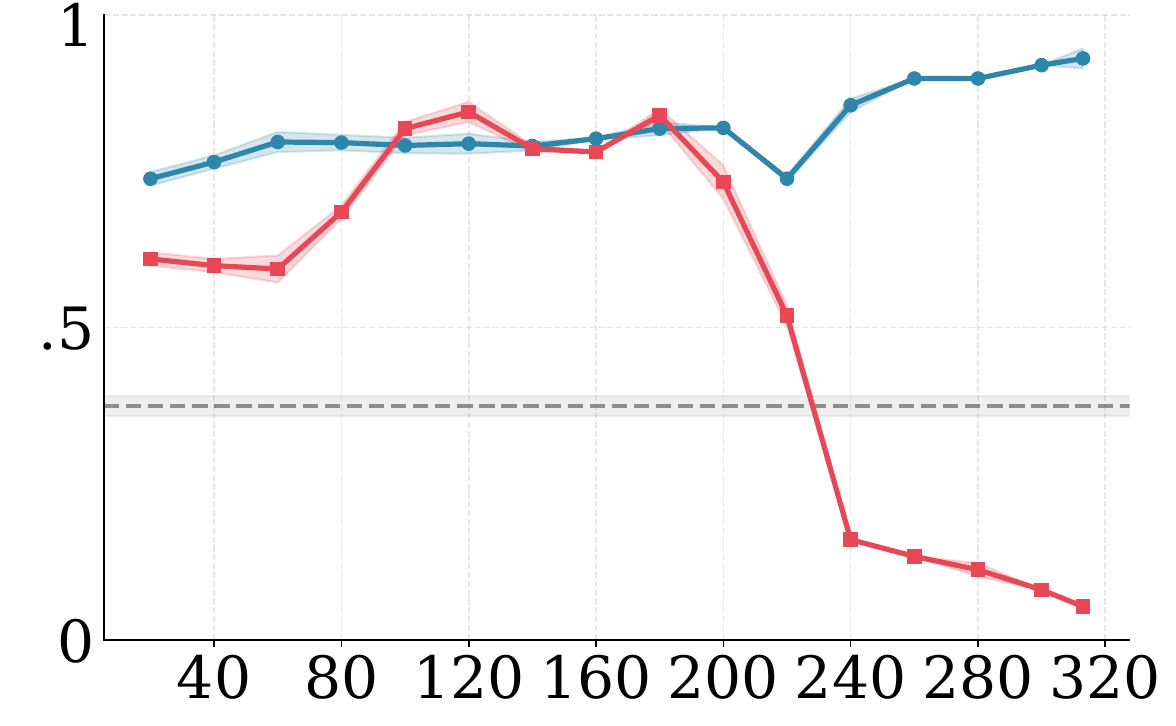} &
        \includegraphics[width=0.155\textwidth]{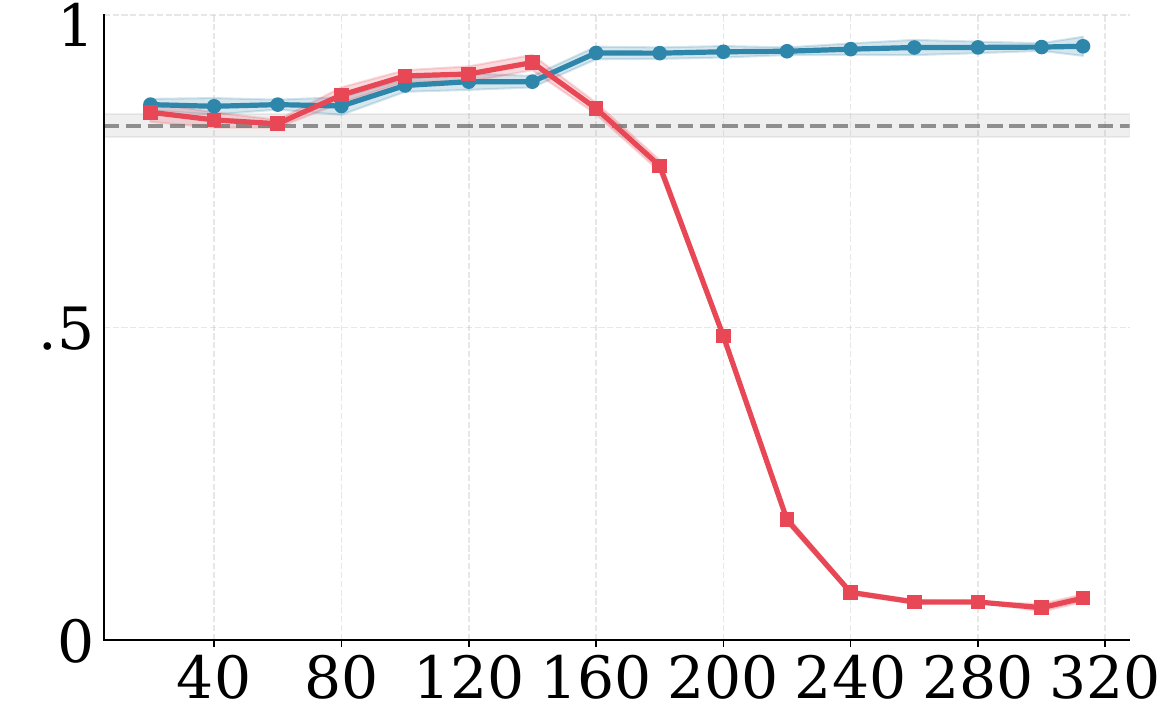} \\
        \rotatebox{90}{\scriptsize Entropy P@8} &
        \includegraphics[width=0.155\textwidth]{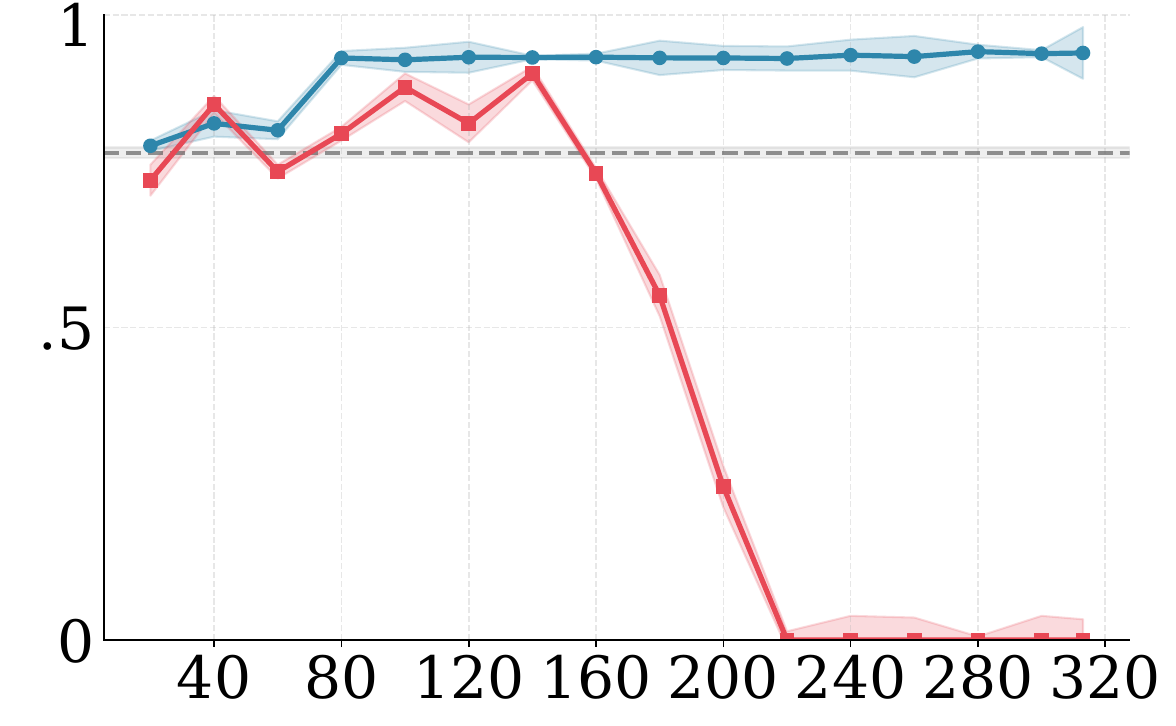} &
        \includegraphics[width=0.155\textwidth]{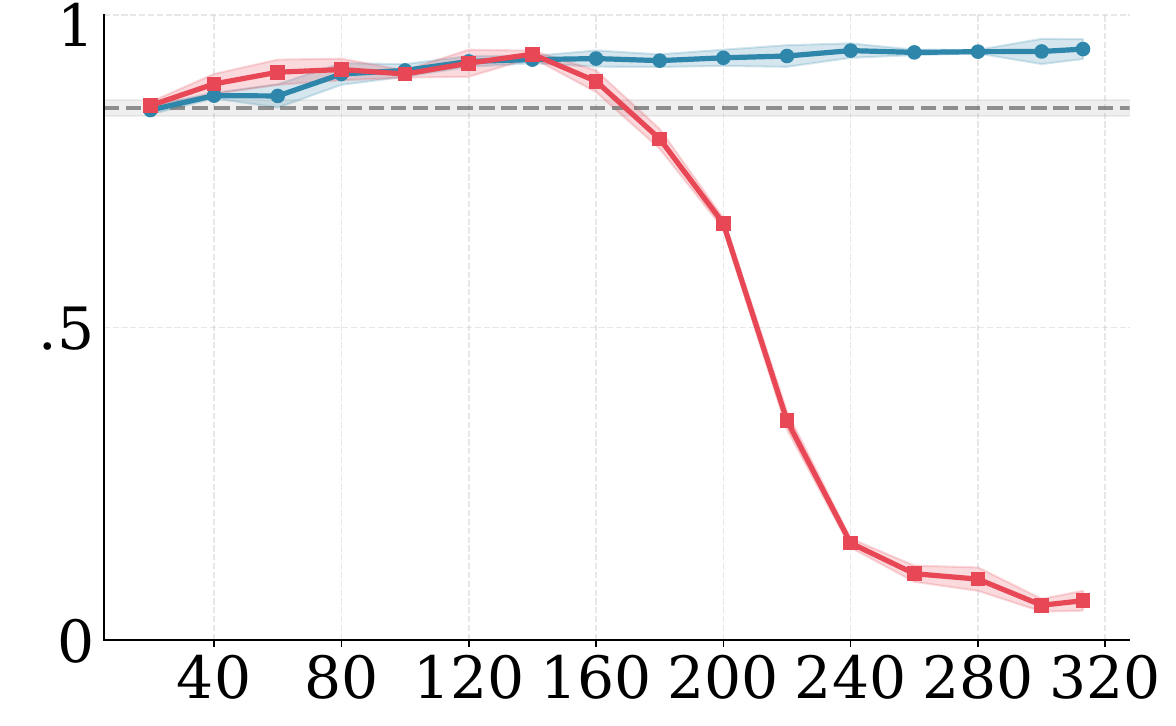} &
        \includegraphics[width=0.155\textwidth]{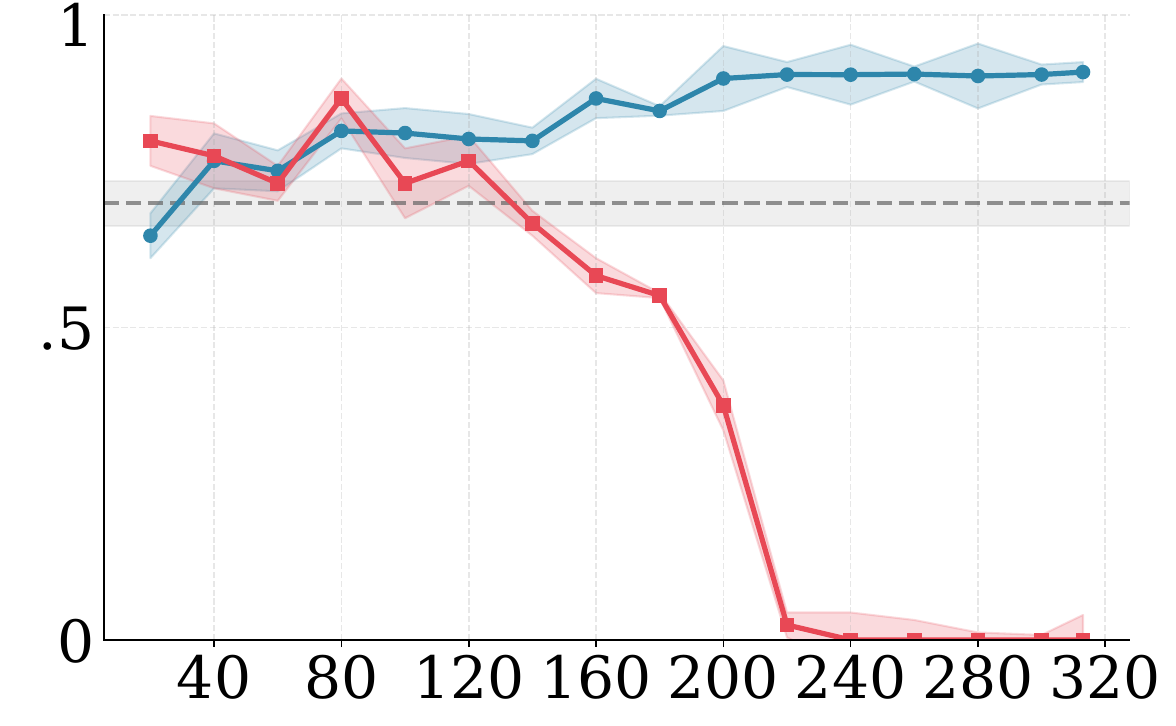} &
        \includegraphics[width=0.155\textwidth]{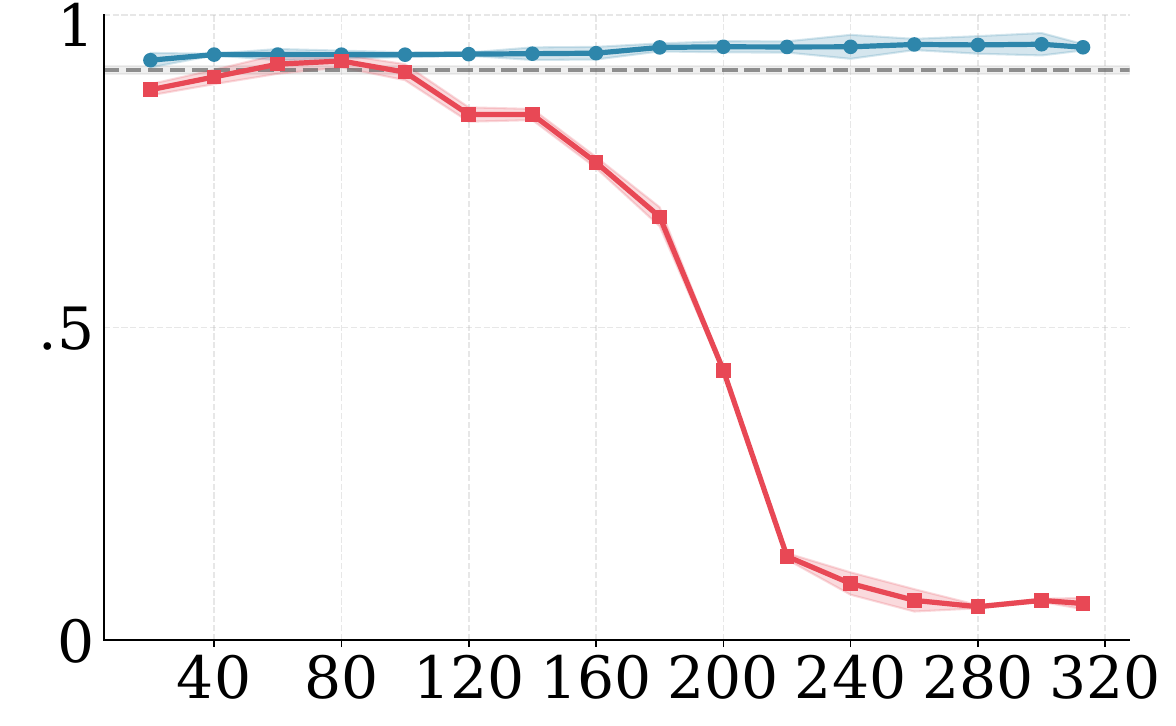} &
        \includegraphics[width=0.155\textwidth]{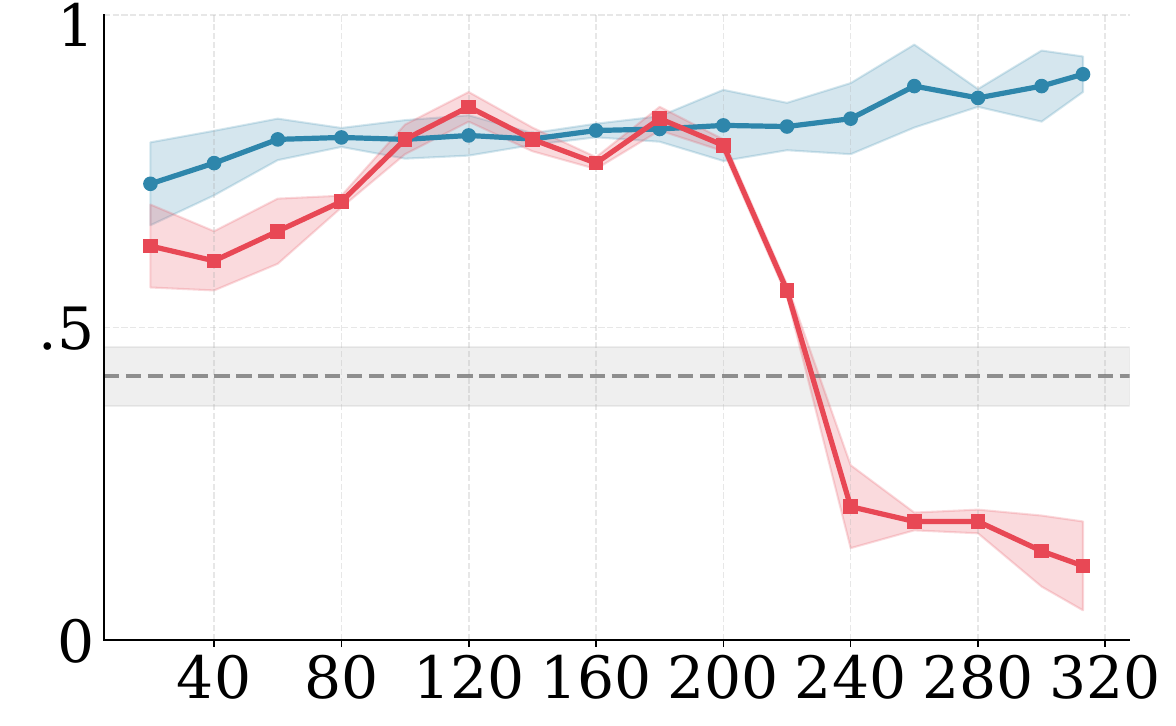} &
        \includegraphics[width=0.155\textwidth]{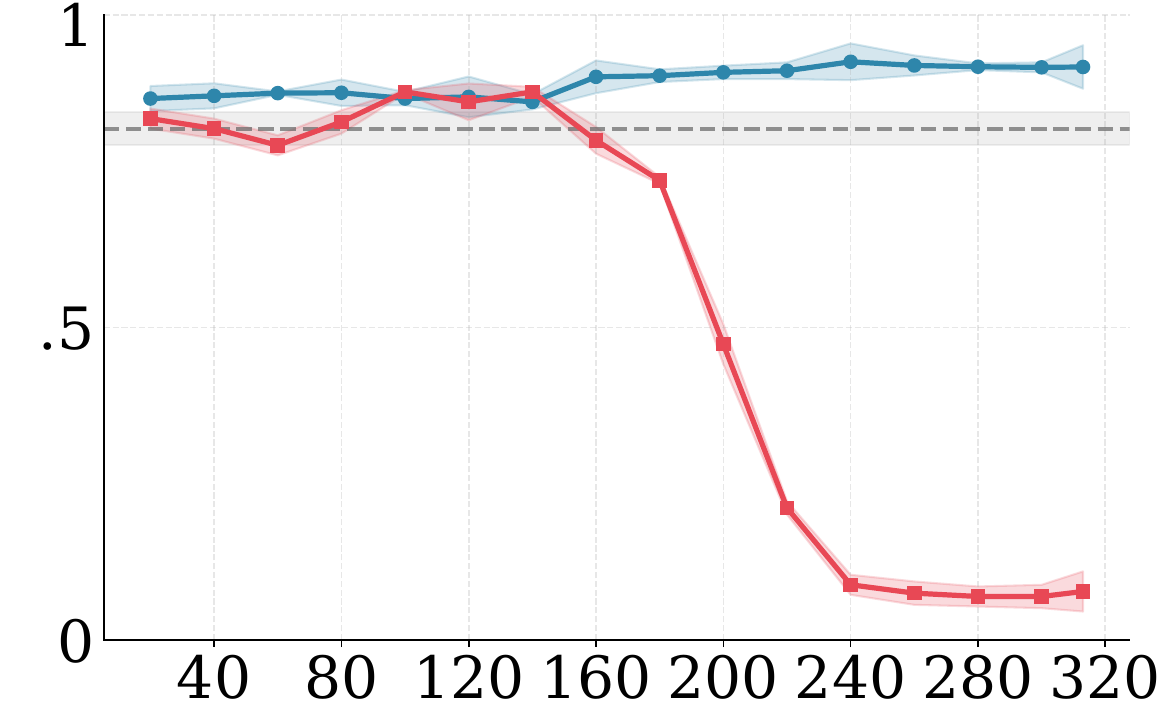} \\
        & \scriptsize Steps & \scriptsize Steps & \scriptsize Steps & \scriptsize Steps & \scriptsize Steps & \scriptsize Steps \\
    \end{tabular}
    \caption{The learning curves comparing \textcolor[HTML]{2E86AB}{VI-CuRL} against the baseline \textcolor[HTML]{E84855}{No Curriculum} across verifier-based (Oracle) and verifier-free (Majority Vote and Entropy) settings. Rows alternate between Pass@1 and Pass@8 for each setting. For each subplot, the y-axis is normalized to $[0,1]$.}
    \label{fig:learning_curves_grid}
\end{figure*}
\subsection{Pass@8 Performance}
While Pass@1 roughly measures the mode of the policy distribution (the ability to generate the best answer), Pass@8 evaluates the breadth and coverage of the model's reasoning capabilities, measuring the probability that at least one solution is correct across 8 independent samples. Table~\ref{tab:main_results_oracle_pass@8} and Table~\ref{tab:main_results_independent_pass@8} present the detailed Pass@8 results for the Oracle and Verifier-Free settings, respectively. The results closely mirror the trends observed in the Pass@1 analysis: VI-CuRL improves coverage across almost all benchmarks. Crucially, in the verifier-independent setting, where baselines often collapse to repetitive, incorrect outputs (leading to low Pass@8), VI-CuRL maintains high diversity and solvability. This consistency between Pass@1 and Pass@8 confirms that our method improves the fundamental quality of the policy without trading off diversity or overfitting to narrow solution paths.

\subsection{Curriculum Control Ablation}
\label{app:curriculum_ablation}
To validate that VI-CuRL's gains stem from its specific annealed confidence curriculum rather than from any form of data filtering, we conduct a systematic ablation study on DeepSeek-R1-Distill-Qwen-1.5B under verifier-free RENT reward. Table~\ref{tab:curriculum_ablation} compares VI-CuRL against several alternative filtering strategies: (1) \textbf{Length-First Filter}, which retains the shortest responses; (2) \textbf{Random Filter}, which randomly retains prompts at the same rate $\beta_t$; and (3) \textbf{Reverse-Confidence Filter}, which retains the \emph{lowest}-confidence prompts (the opposite of VI-CuRL).

\begin{table*}[h!]
  \centering
  \small
  \caption{Curriculum-specific control ablation for DeepSeek-R1-Distill-Qwen-1.5B under verifier-free GRPO with RENT reward. Each cell reports Pass@1 / Pass@8. \colorbox{blue!10}{Blue} rows highlight our method.}
  \label{tab:curriculum_ablation}
  \resizebox{\linewidth}{!}{
    \begin{tabular}{lrrrrrr|r}
      \toprule
      \rowcolor{gray!15} \textbf{Setting} & \textbf{AIME} & \textbf{AIME} & \textbf{AMC} & \textbf{Math500} & \textbf{Minerva} & \textbf{Olympiad} & \textbf{Average} \\
      \rowcolor{gray!15} & \textbf{2024} & \textbf{2025} & \textbf{2023} & & \textbf{MATH} & \textbf{Bench} & \\
      \midrule
      RENT / No Curriculum & 25.3 / 38.0 & 23.2 / 25.1 & 71.1 / 90.5 & 85.3 / 95.0 & 31.5 / 43.6 & 48.0 / 59.9 & 47.4 / 58.7 \\
      RENT / Length-First Filter & 29.6 / 45.4 & 23.6 / 30.5 & 77.0 / 93.6 & 88.7 / 97.1 & 30.1 / 42.8 & 51.1 / 63.1 & 50.0 / 62.1 \\
      RENT / Random Filter & 25.5 / 38.5 & 22.6 / 25.1 & 71.9 / 90.8 & 85.2 / 95.4 & 31.4 / 44.0 & 48.2 / 60.1 & 47.5 / 59.0 \\
      RENT / Reverse-Confidence & 24.9 / 37.5 & 21.6 / 25.6 & 70.1 / 90.0 & 85.1 / 95.0 & 29.9 / 43.3 & 47.5 / 59.4 & 46.5 / 58.5 \\
      \rowcolor{blue!10} VI-CuRL / RENT & \textbf{38.0 / 62.1} & \textbf{31.0 / 47.3} & \textbf{83.8 / 97.6} & \textbf{92.8 / 98.8} & \textbf{32.8 / 45.8} & \textbf{57.8 / 70.4} & \textbf{56.0 / 70.3} \\
      \bottomrule
    \end{tabular}
  }
\end{table*}

\subsection{Long-Output Ablation}
\label{app:long_output}
A potential concern is that truncation artifacts (when responses hit the decoding cap) may affect the confidence signal. Table~\ref{tab:long_output} compares \texttt{max\_new\_tokens}$=3072$ (our default) with $8192$. 

\begin{table*}[h!]
  \centering
  \small
  \caption{Long-output ablation for DeepSeek-R1-Distill-Qwen-1.5B (RENT reward).}
  \label{tab:long_output}
  \resizebox{\linewidth}{!}{\begin{tabular}{llrrrrrrr}
\toprule
Setting & max\_tokens & AIME24 & AIME25 & AMC23 & MATH500 & Minerva & Olympiad & Avg \\
\midrule
RENT & 3072 & 25.3 & 23.2 & 71.1 & 85.3 & 31.5 & 48.0 & 47.4 \\
RENT & 8192 & 25.8 & 23.2 & 71.9 & 86.1 & 31.5 & 48.2 & 47.8 \\
\rowcolor{blue!10} VI-CuRL / RENT & 3072 & 38.0 & 31.0 & 83.8 & 92.8 & 32.8 & 57.8 & 56.0 \\
\rowcolor{blue!10} VI-CuRL / RENT & 8192 & \textbf{38.9} & \textbf{31.6} & \textbf{84.2} & \textbf{93.4} & \textbf{33.5} & \textbf{58.6} & \textbf{56.7} \\
\bottomrule\end{tabular}}
\par\smallskip
\begin{tabular}{llrrrr}
\toprule
Setting & max\_tokens & Cap-hit & No-boxed & Mean len & P90 len \\
\midrule
RENT & 3072 & 3.9\% & 10.0\% & 1895 & 3004 \\
RENT & 8192 & 1.7\% & 7.9\% & 2136 & 4251 \\
\rowcolor{blue!10} VI-CuRL / RENT & 3072 & 2.9\% & 8.1\% & 1056 & 2805 \\
\rowcolor{blue!10} VI-CuRL / RENT & 8192 & 1.2\% & 3.9\% & 2311 & 4511 \\
\bottomrule\end{tabular}
\end{table*}

\subsection{Mechanism Analysis: Why Annealing Matters}
\label{app:mechanism}
Table~\ref{tab:mechanism} examines why the annealed curriculum outperforms static alternatives. We compare four strategies with 16 rollout samples: (1) RENT / No Curriculum, (2) Static Confidence Filter (fixed $\beta=0.2$), (3) Soft Confidence Weighting (using confidence as sample weights without hard filtering), and (4) VI-CuRL's annealed schedule. The key metric is \emph{Upgration} (change from reference step $T_\mathrm{ref}$ to long-horizon $T_\mathrm{long}$): Static filtering achieves high initial performance ($30.2$) but collapses long-term ($-22.4$), while Soft Weighting decays less severely ($-10.2$) but still degrades. Only VI-CuRL maintains stability ($+0.8$) with the highest AULC ($13.0$ vs.~$1.0$ baseline). Table~\ref{tab:mechanism_breakdown} provides the per-dataset breakdown, confirming that VI-CuRL's advantage is consistent across all benchmarks. Besides, some other reasoning works studies highlight the importance of grounding predictions in visual evidence~\citep{xiao2026rethinking,xiao2026staying}. Such signals could complement confidence-based selection when extending VI-CuRL beyond text.

\begin{table*}[h!]
  \centering
  \small
  \caption{Mechanism analysis with 16 rollout samples. $T_\mathrm{ref}$: step 313; $T_\mathrm{long}$: step 626. AULC is normalized with RENT/No Curriculum $= 1.0$. \colorbox{blue!10}{Blue} rows highlight our method.}
  \label{tab:mechanism}
  \resizebox{\linewidth}{!}{
    \begin{tabular}{lcccc|c}
      \toprule
      \rowcolor{gray!15} \textbf{Setting} & \textbf{Avg P@1 @ $T_\mathrm{ref}$} & \textbf{Avg P@1 @ $T_\mathrm{long}$} & \textbf{Best Avg P@1} & \textbf{Upgration} & \textbf{AULC} \\
      \midrule
      RENT / No Curriculum & 4.8 & 2.6 & 32.8 & $-$2.2 & 1.0 \\
      RENT / Static Conf.~Filter & 30.2 & 7.8 & 30.2 & $-$22.4 & 3.3 \\
      RENT / Soft Conf.~Weighting & 13.8 & 3.6 & 30.5 & $-$10.2 & 1.9 \\
      \rowcolor{blue!10} VI-CuRL / RENT & \textbf{31.7} & \textbf{32.5} & \textbf{36.7} & \textbf{$+$0.8} & \textbf{13.0} \\
      \bottomrule
    \end{tabular}
  }
\end{table*}

\begin{table*}[h!]
  \centering
  \small
  \caption{Per-dataset Pass@1 breakdown at $T_\mathrm{long}$ (step 626). \colorbox{blue!10}{Blue} rows highlight our method.}
  \label{tab:mechanism_breakdown}
  \resizebox{\linewidth}{!}{
    \begin{tabular}{lrrrrrr}
      \toprule
      \rowcolor{gray!15} \textbf{Setting} & \textbf{AIME2024} & \textbf{AIME2025} & \textbf{AMC2023} & \textbf{MATH500} & \textbf{Minerva} & \textbf{OlympiadBench} \\
      \midrule
      RENT / No Curriculum & 0.0 & 0.0 & 5.1 & 4.0 & 4.2 & 2.5 \\
      RENT / Static Conf.~Filter & 0.3 & 1.2 & 22.6 & 8.6 & 9.5 & 4.3 \\
      RENT / Soft Conf.~Weighting & 0.0 & 0.0 & 7.8 & 7.2 & 4.1 & 2.6 \\
      \rowcolor{blue!10} VI-CuRL / RENT & \textbf{14.3} & \textbf{5.4} & \textbf{51.3} & \textbf{70.4} & \textbf{20.5} & \textbf{33.1} \\
      \bottomrule
    \end{tabular}
  }
\end{table*}

\section{Proofs and Additional Analysis}\label{sec:proofs}

\subsection{Proof of Theorem~\ref{thm:asymptotic_unbiasedness}}\label{app:proof_asymptotic_unbiasedness}
\begin{proof}
The goal is to bound the difference between the true surrogate objective $\mathcal{L}(\theta)$ and the VI-CuRL surrogate $\mathcal{L}_t(\theta)$.
Recall the definitions:
\[ \mathcal{L}(\theta) = \E_{x \sim p(x)} \E_{y_{1:G} \sim \pi_{\theta_\mathrm{old}}(\cdot|x)} [\ell(\theta; x, y_{1:G})] \]
\[ \mathcal{L}_t(\theta) = \E_{x \sim p(x)} \E_{y_{1:G} \sim \pi_{\theta_\mathrm{old}}(\cdot|x)} \left[ \frac{w_t(x)}{\beta_t} \ell(\theta; x, y_{1:G}) \right] \]
We partitions the objective using the law of total expectation conditioned on the mask $w_t(x)$. Let $\mathcal{S}_t = \{x \mid w_t(x)=1\}$ be the selected subspace.
\begin{align*}
\mathcal{L}(\theta) &= \int \ell(\theta; x, y_{1:G}) p(x) \pi_{\theta_\mathrm{old}}(y_{1:G}|x) dy_{1:G} dx \\
&= \beta_t \mathcal{L}_t(\theta) + (1-\beta_t) \mathcal{L}_{\text{discard}}(\theta).
\end{align*}
Thus, we have the exact decomposition:
\[ \mathcal{L}(\theta) = \beta_t \mathcal{L}_t(\theta) + (1-\beta_t) \mathcal{L}_{\text{discard}}(\theta). \]
Now, we examine the difference:
\begin{align*}
\mathcal{L}(\theta) - \mathcal{L}_t(\theta) &= (1-\beta_t) (\mathcal{L}_{\text{discard}}(\theta) - \mathcal{L}_t(\theta)).
\end{align*}
Assuming the surrogate loss is bounded $|\ell(\theta; x, y_{1:G})| \le L_{\max}$, then:
\begin{align*}
|\mathcal{L}(\theta) - \mathcal{L}_t(\theta)| &\le (1-\beta_t) (L_{\max} + L_{\max}) \\
&= 2L_{\max}(1-\beta_t).
\end{align*}
This completes the proof.
\end{proof}

\subsection{Proof of Theorem~\ref{thm:variance_decomposition}}\label{app:proof_variance_decomposition}
\begin{proof}
Let $\hat{g}_t$ be the VI-CuRL gradient estimator:
\[ \hat{g}_t = \frac{w_t(x)}{\beta_t} g(x, y_{1:G}; \theta), \]
where $g(x, y_{1:G}; \theta) = \nabla_\theta \ell(\theta; x, y_{1:G})$ is the gradient of the surrogate loss.
We use the Law of Total Variance, conditioning on the prompt $x$:
\[ \Var(\hat{g}_t) = \E_{x}[\Var_{\mathbf{y}}(\hat{g}_t | x)] + \Var_{x}(\E_{\mathbf{y}}[\hat{g}_t | x]). \]
We analyze these two terms separately.

\textbf{1. The Expected Action Variance Term:}
Conditioned on a specific prompt $x$, the mask $w_t(x)$ and scalar $\beta_t$ are constants. Thus:
\begin{align*}
\Var_{y_{1:G}}(\hat{g}_t | x) &= \Var_{y_{1:G}}\left( \frac{w_t(x)}{\beta_t} g(x, y_{1:G}; \theta) \bigg| x \right) \\
&= \left( \frac{w_t(x)}{\beta_t} \right)^2 \Var_{y_{1:G}}( g(x, y_{1:G}; \theta) | x ), \\ &\qquad \rev{\text{(quadratic scaling of covariance trace)}}
\end{align*}
Taking the expectation over $x$:
\begin{align*}
\E_{x}[\Var_{y_{1:G}}(\hat{g}_t | x)] &= \E_{x}\left[ \frac{w_t(x)^2}{\beta_t^2} \Var_{y_{1:G}}( g(x, y_{1:G}; \theta) | x ) \right].
\end{align*}
Since $w_t(x)$ is an indicator, $w_t(x)^2 = w_t(x)$.
\begin{align*}
\E_{x}[\Var_{y_{1:G}}(\hat{g}_t | x)] &= \frac{1}{\beta_t^2} \E_{x}\left[ w_t(x) \Var_{y_{1:G}}( g(x, y_{1:G}; \theta) | x ) \right] \\
&= \frac{1}{\beta_t} \cdot \frac{\E_{x}[ w_t(x) \Var_{y_{1:G}}( g(x, y_{1:G}; \theta) | x ) ]}{\beta_t} \\
&= \frac{1}{\beta_t} \E_{x | w_t(x)=1} [ \Var_{y_{1:G}}( g(x, y_{1:G}; \theta) | x ) ].
\end{align*}
We define $\sigma_{g,t}^2 \triangleq \E_{x | w_t(x)=1} [ \Var_{y_{1:G}}( g | x ) ]$ as the average \textbf{Action Variance}. Thus, the first term is $\frac{\sigma_{g,t}^2}{\beta_t}$.

\textbf{2. The Problem Variance Term:}
First, consider the inner expectation $\E_{y_{1:G}}[\hat{g}_t | x]$. Let $\bar{g}(x) = \E_{y_{1:G}|\pi_{\theta_\mathrm{old}}}[g | x]$ be the expected gradient for prompt $x$. Then:
\[ \E_{y_{1:G}}[\hat{g}_t | x] = \frac{w_t(x)}{\beta_t} \bar{g}(x). \]
The variance of this quantity over $x$ is:
\begin{align*}
\Var_{x}\left( \frac{w_t(x)}{\beta_t} \bar{g}(x) \right) &= \E_{x} \left[ \left\| \frac{w_t(x)}{\beta_t} \bar{g}(x) \right\|^2 \right] - \left\| \E_{x} \left[ \frac{w_t(x)}{\beta_t} \bar{g}(x) \right] \right\|^2.
\end{align*}
The second part is simply the squared norm of the true gradient of the surrogate objective $\nabla_\theta \mathcal{L}_t(\theta)$.
Expanding the first part and using the definition of $V_{\mathrm{prob}, t} = \Var_{x | w_t(x)=1}(\bar{g}(x))$:
\begin{align*}
\Var_{x}(\E_{y_{1:G}}[\hat{g}_t | x]) &= \frac{V_{\mathrm{prob}, t}}{\beta_t} + \frac{1-\beta_t}{\beta_t} \|\nabla_\theta \mathcal{L}_t(\theta)\|^2.
\end{align*}
Combining the Action Variance term and the Problem Variance term yields the final decomposition:
\[ \Var(\hat{g}_t) = \frac{\sigma_{g,t}^2}{\beta_t} + \frac{V_{\mathrm{prob}, t}}{\beta_t} + \frac{1-\beta_t}{\beta_t} \|\nabla_\theta \mathcal{L}_t(\theta)\|^2. \]
\end{proof}

\subsection{Proof of Lemma~\ref{lem:conf_aware_bound}}\label{app:proof_conf_aware_bound}
\begin{lemma}[Conditional Entropy-Envelope Bound]\label{lem:conf_aware_bound}
\rev{Let $h(x)=1-c(x)$ and $a_t=1-\tau_t$ be normalized entropy and its selection cutoff. Suppose there exists a non-decreasing function $G_{\conf}:[0,1]\to[0,\infty)$ such that for all selected prompts ($h(x)\le a_t$), the gradient norm obeys the following entropy envelope:}
\begin{equation}
\|\nabla_\theta \ell(\theta; x, y_{1:G})\| \le G_{\conf}(h(x)).
\end{equation}
Then
\begin{equation}
\Var(\hat g_t) \le \frac{G_{\conf}(a_t)^2}{\beta_t} - \|\nabla_\theta \mathcal{L}_t(\theta)\|^2.
\end{equation}
\end{lemma}
This bound is conditional on the envelope assumption. It does not imply that the numerator compensates for inverse retention.
\begin{proof}
We derive a bound for the second moment $\E[\|\hat{g}_t\|^2]$ under the entropy-envelope assumption.
\[ \hat{g}_t = \frac{w_t(x)}{\beta_t} \nabla_\theta \ell(\theta; x, y_{1:G}). \]
Taking the squared norm and expectation:
\begin{align*}
\E[\|\hat{g}_t\|^2] &= \E_{x, y_{1:G}} \left[ \left\| \frac{w_t(x)}{\beta_t} \nabla_\theta \ell \right\|^2 \right] \\
&= \frac{1}{\beta_t^2} \E_{x} \left[ w_t(x) \E_{y_{1:G}|x} [ \|\nabla_\theta \ell(\theta; x, y_{1:G})\|^2 ] \right].
\end{align*}
Using the assumption $\|\nabla_\theta \ell\| \le G_{\conf}(h(x))$ applied to the prompt $x$:
\begin{align*}
\E_{y_{1:G}|x} [ \|\nabla_\theta \ell\|^2 ] \le G_{\conf}(h(x))^2.
\end{align*}
Since $w_t(x)=1$ implies $h(x)\le a_t$ and $G_{\conf}$ is non-decreasing, $G_{\conf}(h(x))\le G_{\conf}(a_t)$.
\begin{align*}
\E[\|\hat{g}_t\|^2] &\le \frac{1}{\beta_t^2} \E_{x} [ w_t(x) G_{\conf}(a_t)^2 ] \\
&= \frac{G_{\conf}(a_t)^2}{\beta_t^2} \E_{x}[w_t(x)] \\
&= \frac{G_{\conf}(a_t)^2}{\beta_t}.
\end{align*}
Finally, using $\Var(\hat{g}_t) = \E[\|\hat{g}_t\|^2] - \|\E[\hat{g}_t]\|^2$, we obtain:
\begin{align*}
\Var(\hat{g}_t) \le \frac{G_{\conf}(a_t)^2}{\beta_t} - \|\nabla_\theta \mathcal{L}_t(\theta)\|^2.
\end{align*}
\end{proof}

\subsection{Proof of Theorem~\ref{thm:curriculum_sensitive}}\label{app:proof_curriculum_sensitive}
\begin{theorem}[Conditional Curriculum-Sensitive Variance Bound]\label{thm:curriculum_sensitive}
Assume that for each phase $t$ there exist constants $a_1,a_2\ge0$ and $\alpha>0$ such that the variance numerators from Theorem~\ref{thm:variance_decomposition} decay faster than the retention rate:
\begin{equation} \label{eq:assumption_A}
\sigma_{g,t}^2 \le a_1\,\beta_t^{1+\alpha}, \quad \text{and} \quad V_{\mathrm{prob},t} \le a_2\,\beta_t^{1+\alpha}.
\end{equation}
Then the total variance is bounded by
\begin{equation}
\Var(\hat g_t) \le (a_1+a_2)\,\beta_t^{\alpha} + \frac{1-\beta_t}{\beta_t}\, \bigl\|\nabla_\theta \mathcal{L}_t(\theta)\bigr\|^2.
\end{equation}
\end{theorem}
\rev{Assumption~\eqref{eq:assumption_A} is an additional sufficient condition, not a consequence of confidence ranking. The measured component ratios alone do not verify its proposed retention scaling. The masking contribution still depends on the selected mean gradient and may grow under aggressive retention; the bound does not establish universal variance reduction.}
\begin{proof}
This proof follows directly from substituting the assumption bounds into the exact decomposition derived in Theorem~\ref{thm:variance_decomposition}.
Recall:
\[ \Var(\hat{g}_t) = \frac{\sigma_{g,t}^2}{\beta_t} + \frac{V_{\mathrm{prob}, t}}{\beta_t} + \frac{1-\beta_t}{\beta_t} \|\nabla_\theta \mathcal{L}_t(\theta)\|^2. \]
Assumption~\eqref{eq:assumption_A} posits that the curriculum effectively selects easier samples such that:
\[ \sigma_{g,t}^2 \le a_1 \beta_t^{1+\alpha} \quad \text{and} \quad V_{\mathrm{prob}, t} \le a_2 \beta_t^{1+\alpha}, \]
for some $\alpha > 0$.
Substituting these inequalities:
\begin{align*}
\frac{\sigma_{g,t}^2}{\beta_t} &\le \frac{a_1 \beta_t^{1+\alpha}}{\beta_t} = a_1 \beta_t^\alpha \\
\frac{V_{\mathrm{prob}, t}}{\beta_t} &\le \frac{a_2 \beta_t^{1+\alpha}}{\beta_t} = a_2 \beta_t^\alpha.
\end{align*}
Summing these terms:
\[ \Var(\hat{g}_t) \le (a_1 + a_2) \beta_t^\alpha + \frac{1-\beta_t}{\beta_t} \|\nabla_\theta \mathcal{L}_t(\theta)\|^2. \]
\rev{As $\beta_t\to0$, only the first contribution is guaranteed to vanish under the assumption. Bounded total variance additionally requires control of $\|\nabla_\theta\mathcal L_t\|^2/\beta_t$; the practical schedule starts at positive retention.}
\end{proof}

\section{Implementation Details}\label{app:implementation}

This section details the prompt configurations, hyperparameters, and datasets used to ensure the reproducibility of our experiments. We adopt a structured presentation to distinguish experimental components clearly.

\subsection{Prompt Templates}
\label{app:prompt-templates}

\paragraph{Training \& Generation.}
We employ a consistent prompt structure for both training rollouts and evaluation, designed to elicit step-by-step reasoning (Chain-of-Thought) and strictly enforce the output format required for rule-based verification.

\begin{tcolorbox}[colback=gray!5!white,colframe=gray!50!black,title=Standard Prompt Template]
\textbf{User:} \\
\texttt{\{QUESTION\}}

Let's think step by step and enclose the reasoning process within \texttt{<think>} and \texttt{</think>} tags.
The final result in the answer MUST BE within \texttt{\textbackslash boxed\{\}}.
\end{tcolorbox}

\noindent During preprocessing, this is formatted as a single-turn conversation. The ground-truth answer is attached separately for reward computation.

\paragraph{Verification Logic.}
The rule-based verifier parses the model's output to extract the content inside the last \texttt{\textbackslash boxed\{\dots\}} tag. It then performs robust equality checking (handling fractions, decimals, and algebraic equivalents) against the ground truth to assign a binary reward $\tilde{R} \in \{0, 1\}$.

\subsection{Configuration and Hyperparameters}
\label{app:train-config}
We utilize Group Relative Policy Optimization (GRPO) as the underlying reinforcement learning algorithm. In the VI-CuRL setting, we introduce curriculum-specific parameters.

\begin{table}[ht]
    \centering
    \caption{Detailed hyperparameters for GRPO training and VI-CuRL curriculum.}
    \label{tab:hparams}
    \resizebox{\linewidth}{!}{
    \begin{tabular}{l c}
        \toprule
        \rowcolor{gray!10} \textbf{Hyperparameter} & \textbf{Value} \\
        \midrule
        \multicolumn{2}{c}{\textit{Optimization \& Model}} \\
        Algorithm & GRPO \\
        Training Strategy & Fully Sharded Data Parallel 2 (FSDP2) + Zero Redundancy Optimizer Stage 3 (ZeRO-3) \\
        Optimizer & AdamW \\
        Learning Rate & $3 \times 10^{-6}$ \\
        Total Epochs & 1 \\
        Global Batch Size & 128 \\
        Gradient Accumulation & 1 \\
        \midrule
        \multicolumn{2}{c}{\textit{Rollout \& GRPO}} \\
        Rollout Engine & vLLM \\
        Context Window & 512 (Input) + 3072 (Output) \\
        Group Size ($G$) & 8 \\
        Sampling Temperature & 1.0 \\
        Advantage & Group-Normalized ($A_i = \frac{R_i - \bar{R}}{\sigma + \epsilon}$) \\
        KL Coefficient ($\beta_{KL}$) & 0.001 \\
        Entropy Coefficient & 0.0 \\
        Reference Model & Frozen Initial Policy \\
        \midrule
        \multicolumn{2}{c}{\textit{VI-CuRL Curriculum}} \\
        Initial Retention ($\beta_{\text{start}}$) & 0.2 \\
        Curriculum Duration ($T_{\text{curr}}$) & 100\% of Training Steps \\
        Curriculum Signal & Confidence ($c(x)$ from Eq.~\eqref{eq:confidence}) \\
        Retention Schedule & Linear Annealing \\
        \bottomrule
    \end{tabular}
    }
\end{table}

\subsection{Dataset Composition}
\label{app:datasets}
Our experiments leverage a combination of established benchmarks for training and comprehensive evaluation.

\paragraph{Training Source.}
\begin{description}
    \item[DeepScaleR] \citep{DeepScaleR}: A large-scale math reasoning corpus used for all outcome-supervised training runs.
    \item[WebInstruct-Verified] \citep{ma2025generalreasoner}: A verified instruction-following dataset used for general knowledge RL training. For our general-domain experiments, we filter out mathematics prompts to focus on non-mathematical reasoning and knowledge-intensive tasks.
\end{description}

\paragraph{Mathematical Reasoning Evaluation Suite.}
We evaluate mathematical reasoning on six distinct datasets to assess robustness and generalization.

\begin{description}
    \item[MATH500] \citep{lightman2024letsverify}: A subset of the MATH benchmark containing 500 representative problems.
    \item[AIME (2024 \& 2025)] \citep{aime2024, aime2025}: Recent sets from the American Invitational Mathematics Examination, testing high-difficulty competition math.
    \item[AMC 2023] \citep{amc23}: Problems from the American Mathematics Competitions, varying in difficulty.
    \item[Minerva Math] \citep{minerva}: A comprehensive suite evaluating quantitative reasoning capabilities.
    \item[OlympiadBench] \citep{olympiadbench}: An advanced benchmark specifically designing to test olympiad-level problem solving.
\end{description}

\paragraph{General Knowledge Evaluation Suite.}
To evaluate whether VI-CuRL generalizes beyond formal mathematical reasoning, we additionally test on four general knowledge and reasoning benchmarks.
\begin{description}
    \item[MMLU-Pro] \citep{wang2024mmlu}: A more challenging version of MMLU designed to evaluate broad multitask language understanding with harder questions and expanded answer choices.
    \item[GPQA-Diamond] \citep{rein2024gpqa}: A high-difficulty graduate-level question answering benchmark emphasizing scientific knowledge and expert reasoning.
    \item[TheoremQA] \citep{chen-etal-2023-theoremqa}: A theorem-driven question answering benchmark that evaluates whether models can apply formal scientific and mathematical principles to solve problems.
    \item[WebInstruct-Validation] \citep{ma2025generalreasoner}: The validation split associated with WebInstruct-Verified, used to measure general instruction-following and reasoning performance under the same data family as the general-domain training source.
\end{description}

\clearpage
\section*{NeurIPS Paper Checklist}

\begin{enumerate}

\item {\bf Claims}
    \item[] Question: Do the main claims made in the abstract and introduction accurately reflect the paper's contributions and scope?
    \item[] Answer: \answerYes{} %
    \item[] Justification: We accurately state our theoretical claims and experimental findings in the abstract and introduction.

\item {\bf Limitations}
    \item[] Question: Does the paper discuss the limitations of the work performed by the authors?
    \item[] Answer: \answerYes{} %
    \item[] Justification: We discuss the current scope of evaluated domains and future extensions in Section~\ref{sec:conclusion}.

\item {\bf Theory assumptions and proofs}
    \item[] Question: For each theoretical result, does the paper provide the full set of assumptions and a complete (and correct) proof?
    \item[] Answer: \answerYes{} %
    \item[] Justification: All assumptions are clearly stated in Section~\ref{sec:theory}, and complete proofs are provided in Appendix~\ref{sec:proofs}.

    \item {\bf Experimental result reproducibility}
    \item[] Question: Does the paper fully disclose all the information needed to reproduce the main experimental results of the paper to the extent that it affects the main claims and/or conclusions of the paper (regardless of whether the code and data are provided or not)?
    \item[] Answer: \answerYes{} %
    \item[] Justification: We fully disclose training configurations, prompt templates, and evaluation datasets in Section~\ref{sec:experiments} and Appendix~\ref{app:implementation}.

\item {\bf Open access to data and code}
    \item[] Question: Does the paper provide open access to the data and code, with sufficient instructions to faithfully reproduce the main experimental results, as described in supplemental material?
    \item[] Answer: \answerYes{} %
    \item[] Justification: We provide our full source code in the supplemental material and discuss dataset access in Appendix~\ref{app:implementation}.

\item {\bf Experimental setting/details}
    \item[] Question: Does the paper specify all the training and test details (e.g., data splits, hyperparameters, how they were chosen, type of optimizer) necessary to understand the results?
    \item[] Answer: \answerYes{} %
    \item[] Justification: Detailed experimental settings and hyperparameters are provided in Section~\ref{sec:experiments} and Appendix~\ref{app:implementation}.

\item {\bf Experiment statistical significance}
    \item[] Question: Does the paper report error bars suitably and correctly defined or other appropriate information about the statistical significance of the experiments?
    \item[] Answer: \answerYes{} %
    \item[] Justification: We report mean and standard deviation across random seeds for our main results in Section~\ref{sec:experiments}.

\item {\bf Experiments compute resources}
    \item[] Question: For each experiment, does the paper provide sufficient information on the computer resources (type of compute workers, memory, time of execution) needed to reproduce the experiments?
    \item[] Answer: \answerYes{} %
    \item[] Justification: The computer resources required (8x A100 GPUs) are specified in Section~\ref{sec:experiments}.
    
\item {\bf Code of ethics}
    \item[] Question: Does the research conducted in the paper conform, in every respect, with the NeurIPS Code of Ethics \url{https://neurips.cc/public/EthicsGuidelines}?
    \item[] Answer: \answerYes{} %
    \item[] Justification: Our research conforms with the NeurIPS Code of Ethics.

\item {\bf Broader impacts}
    \item[] Question: Does the paper discuss both potential positive societal impacts and negative societal impacts of the work performed?
    \item[] Answer: \answerNA{} %
    \item[] Justification: This work focuses on foundational RL optimization algorithms and does not introduce immediate direct negative societal impacts.
    
\item {\bf Safeguards}
    \item[] Question: Does the paper describe safeguards that have been put in place for responsible release of data or models that have a high risk for misuse (e.g., pre-trained language models, image generators, or scraped datasets)?
    \item[] Answer: \answerNA{} %
    \item[] Justification: This paper does not release models or scraped datasets that have a high risk for misuse.

\item {\bf Licenses for existing assets}
    \item[] Question: Are the creators or original owners of assets (e.g., code, data, models), used in the paper, properly credited and are the license and terms of use explicitly mentioned and properly respected?
    \item[] Answer: \answerYes{} %
    \item[] Justification: We appropriately cite all the existing datasets and models used in our experiments in Section~\ref{sec:experiments} and Appendix~\ref{app:implementation}.

\item {\bf New assets}
    \item[] Question: Are new assets introduced in the paper well documented and is the documentation provided alongside the assets?
    \item[] Answer: \answerNA{} %
    \item[] Justification: This paper does not introduce new datasets or pre-trained models.

\item {\bf Crowdsourcing and research with human subjects}
    \item[] Question: For crowdsourcing experiments and research with human subjects, does the paper include the full text of instructions given to participants and screenshots, if applicable, as well as details about compensation (if any)? 
    \item[] Answer: \answerNA{} %
    \item[] Justification: This research does not involve crowdsourcing or human subjects.

\item {\bf Institutional review board (IRB) approvals or equivalent for research with human subjects}
    \item[] Question: Does the paper describe potential risks incurred by study participants, whether such risks were disclosed to the subjects, and whether Institutional Review Board (IRB) approvals (or an equivalent approval/review based on the requirements of your country or institution) were obtained?
    \item[] Answer: \answerNA{} %
    \item[] Justification: This research does not involve human subjects.

\item {\bf Declaration of LLM usage}
    \item[] Question: Does the paper describe the usage of LLMs if it is an important, original, or non-standard component of the core methods in this research? Note that if the LLM is used only for writing, editing, or formatting purposes and does \emph{not} impact the core methodology, scientific rigor, or originality of the research, declaration is not required.
    \item[] Answer: \answerNA{} %
    \item[] Justification: The core method development in this research does not involve LLMs as any original or non-standard components; they are solely the subject of optimization in Section~\ref{sec:experiments}.

\end{enumerate}

\end{document}